\documentclass{article}
\PassOptionsToPackage{table}{xcolor}
\let\OriginalAddContentsLine\addcontentsline
\usepackage{iclr2027_conference,times}
\let\addcontentsline\OriginalAddContentsLine
\usepackage[utf8]{inputenc}
\usepackage[T1]{fontenc}
\usepackage{amsmath,amssymb,amsthm,mathtools}
\usepackage{booktabs,tabularx,array,enumitem}
\usepackage{flafter}
\usepackage{tikz}
\usetikzlibrary{arrows.meta,positioning,calc}
\usepackage{microtype}
\usepackage{url,hyperref}
\usepackage{fontawesome5}  %
\usepackage[capitalize,noabbrev]{cleveref}
\definecolor{ink}{HTML}{183247}
\definecolor{blueaccent}{HTML}{2878A1}
\definecolor{bluewash}{HTML}{EEF5FA}
\definecolor{greenaccent}{HTML}{278572}
\definecolor{greenwash}{HTML}{EDF7F2}
\definecolor{amberaccent}{HTML}{B27A24}
\definecolor{amberwash}{HTML}{FFF6E8}
\definecolor{purpleaccent}{HTML}{7960A0}
\definecolor{purplewash}{HTML}{F2EEF8}
\hypersetup{colorlinks=true,linkcolor=blueaccent,citecolor=blueaccent,
  urlcolor=blueaccent,pdftitle={Learning Masked Diffusion Samplers from Energy}}
\newif\ifappendixcontents
\appendixcontentsfalse
\let\AppendixAddContentsLine\addcontentsline
\renewcommand{\addcontentsline}[3]{%
  \ifappendixcontents\AppendixAddContentsLine{#1}{#2}{#3}\fi}
\newtheorem{theorem}{Theorem}[section]
\newtheorem{proposition}[theorem]{Proposition}
\newtheorem{lemma}[theorem]{Lemma}

\theoremstyle{definition}

\newcommand{\X}{\mathcal X}
\newcommand{\E}{\mathbb E}

\newcommand{\one}{\mathbf 1}
\newcommand{\dd}{\mathrm d}

\crefname{lemma}{lemma}{lemmas}
\Crefname{lemma}{Lemma}{Lemmas}

\DeclareMathOperator{\ESS}{ESS}

\DeclareMathOperator{\osc}{osc}
\newcolumntype{Y}{>{\raggedright\arraybackslash}X}
\setlist[itemize]{leftmargin=1.5em,itemsep=2pt,topsep=3pt}
\setlist[enumerate]{leftmargin=1.7em,itemsep=2pt,topsep=3pt}
\title{Fixed-point neural samplers on discrete spaces}
\author{Jiajun He\thanks{Equal contribution.  Work done during JH, DB and MC's internship at Microsoft Research New England, mentored by YD and CDE.  $^\dagger$Equal supervision.}\\University of Cambridge \\\texttt{jh2383@cam.ac.uk}\And Denis Blessing$^*$ \\Karlsruhe Institute of Technology\\\texttt{jl8142@kit.edu}\And  Mouyang Cheng \\MIT \\\texttt{vipandyc@mit.edu} \And   Yuanqi Du$^\dagger$ \\Microsoft Research New England\\\texttt{yuanqidu@microsoft.com} \And  Carles Domingo-Enrich$^\dagger$ \\Microsoft Research New England\\\texttt{carlesd@microsoft.com}}

\usepackage{placeins}  %
\usepackage{pgfplots}
\pgfplotsset{compat=1.18}
\usepgfplotslibrary{groupplots}
\usepgfplotslibrary{external}
\tikzifexternalizing{\ClearShipoutPicture}{}

\renewcommand{\arraystretch}{0.94}
\usepackage{comment}

\newcommand{\IsingOrdDSIZ}{\best{$0.75_{\pm0.45}$}}

\newcommand{\IsingOrdMDNSZ}{$693_{\pm0}$}
\newcommand{\IsingOrdPDNSZ}{$693_{\pm0}$}
\newcommand{\IsingOrdASBSZ}{--}
\newcommand{\IsingOrdASBSMZ}{$825_{\pm413}$}
\newcommand{\IsingOrdLEAPSZ}{$10680_{\pm2968}$}
\newcommand{\IsingOrdDSIESS}{$.979_{\pm.011}$}

\newcommand{\IsingOrdMDNSESS}{$.983_{\pm.007}$}
\newcommand{\IsingOrdPDNSESS}{\best{$.988_{\pm.003}$}}
\newcommand{\IsingOrdASBSESS}{--}
\newcommand{\IsingOrdASBSMESS}{$.003_{\pm.003}$}
\newcommand{\IsingOrdLEAPSESS}{$.000_{\pm.000}$}
\newcommand{\IsingOrdDSIMTV}{\best{$.014_{\pm.007}$}}

\newcommand{\IsingOrdMDNSMTV}{$.500_{\pm.000}$}
\newcommand{\IsingOrdPDNSMTV}{$.500_{\pm.000}$}
\newcommand{\IsingOrdASBSMTV}{$.447_{\pm.075}$}
\newcommand{\IsingOrdASBSMMTV}{$.500_{\pm.000}$}
\newcommand{\IsingOrdLEAPSMTV}{$.023_{\pm.007}$}
\newcommand{\IsingCritDSIZ}{\best{$1.63_{\pm0.27}$}}

\newcommand{\IsingCritMDNSZ}{$1.93_{\pm1.00}$}
\newcommand{\IsingCritPDNSZ}{\best{$1.47_{\pm0.61}$}}
\newcommand{\IsingCritASBSZ}{--}
\newcommand{\IsingCritASBSMZ}{$821_{\pm116}$}
\newcommand{\IsingCritLEAPSZ}{$2536_{\pm146}$}
\newcommand{\IsingCritDSIESS}{\best{$.980_{\pm.001}$}}

\newcommand{\IsingCritMDNSESS}{$.925_{\pm.006}$}
\newcommand{\IsingCritPDNSESS}{$.917_{\pm.008}$}
\newcommand{\IsingCritASBSESS}{--}
\newcommand{\IsingCritASBSMESS}{$.002_{\pm.001}$}
\newcommand{\IsingCritLEAPSESS}{$.006_{\pm.003}$}
\newcommand{\IsingCritDSIMTV}{\best{$.015_{\pm.004}$}}

\newcommand{\IsingCritMDNSMTV}{\best{$.016_{\pm.006}$}}
\newcommand{\IsingCritPDNSMTV}{\best{$.016_{\pm.008}$}}
\newcommand{\IsingCritASBSMTV}{$.042_{\pm.028}$}
\newcommand{\IsingCritASBSMMTV}{$.360_{\pm.198}$}
\newcommand{\IsingCritLEAPSMTV}{$.027_{\pm.014}$}
\newcommand{\IsingDisDSIZ}{\best{$1.07_{\pm0.35}$}}

\newcommand{\IsingDisMDNSZ}{$1.40_{\pm0.79}$}
\newcommand{\IsingDisPDNSZ}{$2.86_{\pm1.86}$}
\newcommand{\IsingDisASBSZ}{--}
\newcommand{\IsingDisASBSMZ}{$601_{\pm159}$}
\newcommand{\IsingDisLEAPSZ}{$660_{\pm1}$}
\newcommand{\IsingDisDSIESS}{\best{$.995_{\pm.000}$}}

\newcommand{\IsingDisMDNSESS}{$.961_{\pm.002}$}
\newcommand{\IsingDisPDNSESS}{$.956_{\pm.001}$}
\newcommand{\IsingDisASBSESS}{--}
\newcommand{\IsingDisASBSMESS}{$.002_{\pm.002}$}
\newcommand{\IsingDisLEAPSESS}{$.985_{\pm.001}$}
\newcommand{\IsingDisDSIMTV}{\best{$.006_{\pm.004}$}}

\newcommand{\IsingDisMDNSMTV}{$.012_{\pm.004}$}
\newcommand{\IsingDisPDNSMTV}{$.014_{\pm.006}$}
\newcommand{\IsingDisASBSMTV}{\best{$.007_{\pm.004}$}}
\newcommand{\IsingDisASBSMMTV}{\best{$.010_{\pm.004}$}}
\newcommand{\IsingDisLEAPSMTV}{\best{$.007_{\pm.004}$}}
\newcommand{\PottsOrdDSIZ}{\best{$9.44_{\pm1.34}$}}

\newcommand{\PottsOrdMDNSZ}{$1089_{\pm2}$}
\newcommand{\PottsOrdPDNSZ}{$628_{\pm325}$}
\newcommand{\PottsOrdASBSZ}{--}
\newcommand{\PottsOrdASBSMZ}{$814_{\pm334}$}
\newcommand{\PottsOrdLEAPSZ}{$7693_{\pm832}$}
\newcommand{\PottsOrdDSIESS}{\best{$.974_{\pm.007}$}}

\newcommand{\PottsOrdMDNSESS}{$.944_{\pm.041}$}
\newcommand{\PottsOrdPDNSESS}{$.964_{\pm.008}$}
\newcommand{\PottsOrdASBSESS}{--}
\newcommand{\PottsOrdASBSMESS}{$.004_{\pm.002}$}
\newcommand{\PottsOrdLEAPSESS}{$.000_{\pm.000}$}
\newcommand{\PottsOrdDSIMTV}{\best{$.014_{\pm.011}$}}

\newcommand{\PottsOrdMDNSMTV}{$.667_{\pm.000}$}
\newcommand{\PottsOrdPDNSMTV}{$.444_{\pm.157}$}
\newcommand{\PottsOrdASBSMTV}{$.441_{\pm.223}$}
\newcommand{\PottsOrdASBSMMTV}{$.584_{\pm.118}$}
\newcommand{\PottsOrdLEAPSMTV}{$.032_{\pm.011}$}
\newcommand{\PottsCritDSIZ}{$13_{\pm2}$}

\newcommand{\PottsCritMDNSZ}{\best{$11_{\pm2}$}}
\newcommand{\PottsCritPDNSZ}{\best{$8.31_{\pm4.32}$}}
\newcommand{\PottsCritASBSZ}{--}
\newcommand{\PottsCritASBSMZ}{$187_{\pm87}$}
\newcommand{\PottsCritLEAPSZ}{$1411_{\pm366}$}
\newcommand{\PottsCritDSIESS}{\best{$.971_{\pm.000}$}}

\newcommand{\PottsCritMDNSESS}{$.874_{\pm.007}$}
\newcommand{\PottsCritPDNSESS}{$.879_{\pm.003}$}
\newcommand{\PottsCritASBSESS}{--}
\newcommand{\PottsCritASBSMESS}{$.006_{\pm.002}$}
\newcommand{\PottsCritLEAPSESS}{$.004_{\pm.003}$}
\newcommand{\PottsCritDSIMTV}{\best{$.006_{\pm.003}$}}

\newcommand{\PottsCritMDNSMTV}{$.017_{\pm.009}$}
\newcommand{\PottsCritPDNSMTV}{$.013_{\pm.009}$}
\newcommand{\PottsCritASBSMTV}{$.129_{\pm.026}$}
\newcommand{\PottsCritASBSMMTV}{$.290_{\pm.163}$}
\newcommand{\PottsCritLEAPSMTV}{$.019_{\pm.011}$}
\newcommand{\PottsDisDSIZ}{\best{$1.03_{\pm0.28}$}}

\newcommand{\PottsDisMDNSZ}{$1.40_{\pm0.92}$}
\newcommand{\PottsDisPDNSZ}{$1.33_{\pm0.86}$}
\newcommand{\PottsDisASBSZ}{--}
\newcommand{\PottsDisASBSMZ}{$91_{\pm40}$}
\newcommand{\PottsDisLEAPSZ}{$163_{\pm0}$}
\newcommand{\PottsDisDSIESS}{$.997_{\pm.000}$}

\newcommand{\PottsDisMDNSESS}{$.957_{\pm.004}$}
\newcommand{\PottsDisPDNSESS}{$.957_{\pm.005}$}
\newcommand{\PottsDisASBSESS}{--}
\newcommand{\PottsDisASBSMESS}{$.016_{\pm.002}$}
\newcommand{\PottsDisLEAPSESS}{\best{$.999_{\pm.000}$}}
\newcommand{\PottsDisDSIMTV}{\best{$.005_{\pm.002}$}}

\newcommand{\PottsDisMDNSMTV}{$.008_{\pm.003}$}
\newcommand{\PottsDisPDNSMTV}{$.008_{\pm.004}$}
\newcommand{\PottsDisASBSMTV}{$.005_{\pm.003}$}
\newcommand{\PottsDisASBSMMTV}{\best{$.004_{\pm.004}$}}
\newcommand{\PottsDisLEAPSMTV}{$.007_{\pm.003}$}

\newcommand{\NSeeds}{3}
\newcommand{\NDraws}{8192}
\newcommand{\CmpIsingDSITrain}{$6.33_{\pm0.11}$}

\newcommand{\CmpIsingDSITotal}{$1.5{\cdot}10^{6}$}

\newcommand{\CmpIsingMDNSTrain}{$6.05_{\pm0.09}$}

\newcommand{\CmpIsingMDNSTotal}{$7.7{\cdot}10^{5}$}
\newcommand{\CmpIsingPDNSTrain}{$6.42_{\pm0.35}$}

\newcommand{\CmpIsingPDNSTotal}{$7.7{\cdot}10^{5}$}
\newcommand{\CmpIsingASBSTrain}{$11.57_{\pm0.09}$}

\newcommand{\CmpIsingASBSTotal}{$1.1{\cdot}10^{8}$}
\newcommand{\CmpIsingASBSMTrain}{$10.72_{\pm0.13}$}

\newcommand{\CmpIsingASBSMTotal}{$1.6{\cdot}10^{8}$}
\newcommand{\CmpIsingLEAPSTrain}{$5.85_{\pm0.12}$}

\newcommand{\CmpIsingLEAPSTotal}{$5.2{\cdot}10^{8}$}
\newcommand{\CmpPottsDSITrain}{$1.89_{\pm0.07}$}

\newcommand{\CmpPottsDSITotal}{$1.5{\cdot}10^{6}$}

\newcommand{\CmpPottsMDNSTrain}{$1.79_{\pm0.10}$}

\newcommand{\CmpPottsMDNSTotal}{$7.7{\cdot}10^{5}$}
\newcommand{\CmpPottsPDNSTrain}{$2.22_{\pm0.30}$}

\newcommand{\CmpPottsPDNSTotal}{$7.7{\cdot}10^{5}$}
\newcommand{\CmpPottsASBSTrain}{$5.60_{\pm0.28}$}

\newcommand{\CmpPottsASBSTotal}{$1.1{\cdot}10^{8}$}
\newcommand{\CmpPottsASBSMTrain}{$5.19_{\pm0.19}$}

\newcommand{\CmpPottsASBSMTotal}{$1.6{\cdot}10^{8}$}
\newcommand{\CmpPottsLEAPSTrain}{$2.86_{\pm0.17}$}

\newcommand{\CmpPottsLEAPSTotal}{$4.2{\cdot}10^{8}$}
\newcommand{\CmpIsingDSIBB}{$4.8{\cdot}10^{8}$}
\newcommand{\CmpIsingMDNSBB}{$7.7{\cdot}10^{5}$}
\newcommand{\CmpIsingPDNSBB}{$7.7{\cdot}10^{5}$}
\newcommand{\CmpIsingASBSBB}{$6.4{\cdot}10^{10}$}
\newcommand{\CmpIsingASBSMBB}{$9.6{\cdot}10^{10}$}
\newcommand{\CmpIsingLEAPSBB}{$2.5{\cdot}10^{11}$}
\newcommand{\CmpPottsDSIBB}{$3.9{\cdot}10^{8}$}
\newcommand{\CmpPottsMDNSBB}{$7.7{\cdot}10^{5}$}
\newcommand{\CmpPottsPDNSBB}{$7.7{\cdot}10^{5}$}
\newcommand{\CmpPottsASBSBB}{$5.3{\cdot}10^{10}$}
\newcommand{\CmpPottsASBSMBB}{$7.9{\cdot}10^{10}$}
\newcommand{\CmpPottsLEAPSBB}{$1.6{\cdot}10^{11}$}

\newcommand{\FPIsingESS}{$.729_{\pm.007}$}
\newcommand{\FPIsingESSmin}{0.20}
\newcommand{\FPIsingRounds}{4120--4200}
\newcommand{\FPIsingNSeeds}{3}
\newcommand{\FPIsingBetaLo}{0.39}
\newcommand{\FPIsingBetaHi}{0.65}
\newcommand{\FPIsingHmax}{0.1}
\newcommand{\FPIsingSteps}{512}

\newcommand{\FPIsingMTVRW}{$.0012_{\pm.0002}$}
\newcommand{\FPIsingZ}{$6.8_{\pm0.3}$}

\newcommand{\FPWorstYmax}{0.30}

\newcommand{\FPPottsESS}{$.789_{\pm.002}$}
\newcommand{\FPPottsESSmin}{0.57}
\newcommand{\FPPottsRounds}{15200--19063}
\newcommand{\FPPottsNSeeds}{3}
\newcommand{\FPPottsBetaLo}{0.8}
\newcommand{\FPPottsBetaHi}{1.3}
\newcommand{\FPPottsHmax}{0.01}
\newcommand{\FPPottsSteps}{288}

\newcommand{\FPPottsMTVRW}{$.0082_{\pm.0005}$}
\newcommand{\FPPottsZ}{$9.3_{\pm0.2}$}

\newcommand{\FPPottsCfgBetaDis}{0.8}
\newcommand{\FPPottsCfgBetaCrit}{0.9875}
\newcommand{\FPPottsCfgBetaOrd}{1.3}
\newcommand{\FPNDraws}{5000}

\newcommand{\FPMTVRWScale}{0.05}

\newcommand{\AblBaseOrdZ}{$0.7_{\pm0.4}$}
\newcommand{\AblBaseOrdESS}{$.979_{\pm.011}$}
\newcommand{\AblBaseOrdMTV}{$.014_{\pm.007}$}

\newcommand{\AblBaseCritZ}{$1.6_{\pm0.3}$}
\newcommand{\AblBaseCritESS}{$.980_{\pm.001}$}
\newcommand{\AblBaseCritMTV}{$.015_{\pm.004}$}

\newcommand{\AblBaseDisZ}{$1.1_{\pm0.4}$}
\newcommand{\AblBaseDisESS}{$.995_{\pm.000}$}
\newcommand{\AblBaseDisMTV}{$.006_{\pm.004}$}

\newcommand{\AblNoDampOrdZ}{$462.0_{\pm326.4}$}
\newcommand{\AblNoDampOrdESS}{$.994_{\pm.005}$}
\newcommand{\AblNoDampOrdMTV}{$.336_{\pm.232}$}

\newcommand{\AblNoDampCritZ}{$1.1_{\pm0.4}$}
\newcommand{\AblNoDampCritESS}{$.966_{\pm.007}$}
\newcommand{\AblNoDampCritMTV}{$.017_{\pm.012}$}

\newcommand{\AblNoDampDisZ}{$1.2_{\pm0.5}$}
\newcommand{\AblNoDampDisESS}{$.992_{\pm.001}$}
\newcommand{\AblNoDampDisMTV}{$.010_{\pm.007}$}

\newcommand{\AblNoGibbsTOrdZ}{$26479.6_{\pm456.9}$}
\newcommand{\AblNoGibbsTOrdESS}{$.001_{\pm.000}$}
\newcommand{\AblNoGibbsTOrdMTV}{$.064_{\pm.022}$}

\newcommand{\AblNoGibbsTCritZ}{$1.5_{\pm0.9}$}
\newcommand{\AblNoGibbsTCritESS}{$.895_{\pm.009}$}
\newcommand{\AblNoGibbsTCritMTV}{$.011_{\pm.008}$}

\newcommand{\AblNoGibbsTDisZ}{$2.1_{\pm1.9}$}
\newcommand{\AblNoGibbsTDisESS}{$.965_{\pm.001}$}
\newcommand{\AblNoGibbsTDisMTV}{$.008_{\pm.004}$}

\newcommand{\AblNoGibbsPOrdZ}{$693.1_{\pm0.8}$}
\newcommand{\AblNoGibbsPOrdESS}{$.993_{\pm.000}$}
\newcommand{\AblNoGibbsPOrdMTV}{$.500_{\pm.000}$}

\newcommand{\AblNoGibbsPCritZ}{$1.0_{\pm0.2}$}
\newcommand{\AblNoGibbsPCritESS}{$.943_{\pm.003}$}
\newcommand{\AblNoGibbsPCritMTV}{$.019_{\pm.004}$}

\newcommand{\AblNoGibbsPDisZ}{$1.2_{\pm0.8}$}
\newcommand{\AblNoGibbsPDisESS}{$.974_{\pm.002}$}
\newcommand{\AblNoGibbsPDisMTV}{$.009_{\pm.005}$}

\newcommand{\AblNoIWOrdZ}{$15.8_{\pm9.9}$}
\newcommand{\AblNoIWOrdESS}{$.475_{\pm.219}$}
\newcommand{\AblNoIWOrdMTV}{$.206_{\pm.127}$}

\newcommand{\AblNoIWCritZ}{$4.0_{\pm1.7}$}
\newcommand{\AblNoIWCritESS}{$.715_{\pm.012}$}
\newcommand{\AblNoIWCritMTV}{$.098_{\pm.050}$}

\newcommand{\AblNoIWDisZ}{$1.1_{\pm0.8}$}
\newcommand{\AblNoIWDisESS}{$.993_{\pm.000}$}
\newcommand{\AblNoIWDisMTV}{$.009_{\pm.004}$}

\newcommand{\AblPDNSOrdZ}{$693.1_{\pm0.4}$}
\newcommand{\AblPDNSOrdESS}{$.988_{\pm.003}$}
\newcommand{\AblPDNSOrdMTV}{$.500_{\pm.000}$}

\newcommand{\AblPDNSCritZ}{$1.5_{\pm0.6}$}
\newcommand{\AblPDNSCritESS}{$.917_{\pm.008}$}
\newcommand{\AblPDNSCritMTV}{$.016_{\pm.008}$}

\newcommand{\AblPDNSDisZ}{$2.9_{\pm1.9}$}
\newcommand{\AblPDNSDisESS}{$.956_{\pm.001}$}
\newcommand{\AblPDNSDisMTV}{$.014_{\pm.006}$}

\newcommand{\AblMDNSOrdZ}{$692.7_{\pm0.4}$}
\newcommand{\AblMDNSOrdESS}{$.983_{\pm.007}$}
\newcommand{\AblMDNSOrdMTV}{$.500_{\pm.000}$}

\newcommand{\AblMDNSCritZ}{$1.9_{\pm1.0}$}
\newcommand{\AblMDNSCritESS}{$.925_{\pm.006}$}
\newcommand{\AblMDNSCritMTV}{$.016_{\pm.006}$}

\newcommand{\AblMDNSDisZ}{$1.4_{\pm0.8}$}
\newcommand{\AblMDNSDisESS}{$.961_{\pm.002}$}
\newcommand{\AblMDNSDisMTV}{$.012_{\pm.004}$}

\newcommand{\AblNoDampSingleMode}{2}

\newcommand{\AblNSeeds}{3}

\newcommand{\BrCrossNTimes}{129}

\newcommand{\BrCrossDlnZExact}{297.9}
\newcommand{\BrCrossErr}{$0.014_{\pm0.019}$}
\newcommand{\BrCrossErrRel}{$0.5_{\pm0.6}$}

\newcommand{\BrCrossESS}{$.105_{\pm.023}$}

\newcommand{\BrCrossMTVRW}{$.007_{\pm.003}$}

\newcommand{\BrCrossBetaZero}{0.20}
\newcommand{\BrCrossBetaOne}{0.60}

\newcommand{\BrOrdDlnZExact}{115.2}
\newcommand{\BrOrdErr}{$0.115_{\pm0.073}$}
\newcommand{\BrOrdErrRel}{$10.0_{\pm6.3}$}

\newcommand{\BrOrdESS}{$.036_{\pm.019}$}

\newcommand{\BrOrdMTVRW}{$.030_{\pm.024}$}

\newcommand{\BrOrdBetaZero}{0.50}
\newcommand{\BrOrdBetaOne}{0.60}

\newcommand{\BrNSeeds}{3}

\newcommand{\BrL}{25}

\newcommand{\AgPdNT}{29}
\newcommand{\AgPdNdmu}{81}
\newcommand{\AgPdNpoints}{2349}

\newcommand{\AgPdDmuMin}{-0.40}
\newcommand{\AgPdDmuMax}{0.40}
\newcommand{\AgPdDmuLo}{-0.405}
\newcommand{\AgPdDmuHi}{0.405}
\newcommand{\AgPdTLo}{187.5}
\newcommand{\AgPdTHi}{912.5}
\newcommand{\AgPdNties}{15}
\newcommand{\AgPdTClose}{550}
\newcommand{\AgPdCErrUnit}{10^{-3}}
\newcommand{\AgPdDSICErr}{\best{$0.48 \pm 0.00$}}
\newcommand{\AgPdFrOGSCErr}{$5.25 \pm 5.80$}
\newcommand{\AgPdZErrUnit}{10^{-4}}
\newcommand{\AgPdDSIZErr}{\best{$0.62 \pm 0.01$}}
\newcommand{\AgPdFrOGSZErr}{$27.0 \pm 26.5$}
\newcommand{\AgPdDSIESS}{\best{$0.973 \pm 0.001$}}
\newcommand{\AgPdFrOGSESS}{$0.210 \pm 0.067$}
\newcommand{\FmixTA}{200}
\newcommand{\FmixTB}{700}

\newcommand{\FmixTiesRefA}{4}
\newcommand{\FmixTiesModA}{4}
\newcommand{\FmixTiesRefB}{2}
\newcommand{\FmixTiesModB}{2}
\newcommand{\FmixLo}{-0.0403}
\newcommand{\FmixHi}{0.0023}
\newcommand{\CuAuNT}{29}
\newcommand{\CuAuNdmu}{36}
\newcommand{\CuAuNpoints}{1044}

\newcommand{\CuAuDmuMin}{-0.20}
\newcommand{\CuAuDmuMax}{0.15}
\newcommand{\CuAuDmuLo}{-0.205}
\newcommand{\CuAuDmuHi}{0.155}
\newcommand{\CuAuTLo}{187.5}
\newcommand{\CuAuTHi}{912.5}
\newcommand{\CuAuNties}{74}
\newcommand{\CuAuTClose}{750}
\newcommand{\CuAuCErrUnit}{10^{-2}}
\newcommand{\CuAuDSICErr}{\best{$1.10 \pm 0.01$}}
\newcommand{\CuAuFrOGSCErr}{$2.35 \pm 0.61$}
\newcommand{\CuAuZErrUnit}{10^{-4}}
\newcommand{\CuAuDSIZErr}{\best{$0.82 \pm 0.04$}}
\newcommand{\CuAuFrOGSZErr}{$112 \pm 14$}
\newcommand{\CuAuDSIESS}{\best{$0.874 \pm 0.006$}}
\newcommand{\CuAuFrOGSESS}{$0.098 \pm 0.055$}
\newif\ifAgPdDSIPlaceholder
\AgPdDSIPlaceholderfalse
\newif\ifAgPdFrOGSPlaceholder
\AgPdFrOGSPlaceholderfalse
\newif\ifCuAuDSIPlaceholder
\CuAuDSIPlaceholderfalse
\newif\ifCuAuFrOGSPlaceholder
\CuAuFrOGSPlaceholderfalse
\newcommand{\placeholdernote}{}

\newcommand{\AlloyNDraws}{5000}
\newcommand{\AlloyNSeeds}{3}

\newcommand{\RLNTasks}{20}
\newcommand{\RLNSeeds}{10}

\newcommand{\RLREPPOCell}{$0.888$\,{\scriptsize$[0.874,0.901]$}}

\newcommand{\RLPPOCell}{$0.310$\,{\scriptsize$[0.270,0.351]$}}

\newcommand{\RLSPOCell}{$0.739$\,{\scriptsize$[0.713,0.766]$}}

\newcommand{\RLTRUDICell}{$0.906$\,{\scriptsize$[0.899,0.913]$}}

\newcommand{\RLAMDPCell}{$0.895$\,{\scriptsize$[0.880,0.909]$}}

\newcommand{\RLDPPOCell}{$0.242$\,{\scriptsize$[0.222,0.261]$}}

\newcommand{\RLFPOCell}{$0.744$\,{\scriptsize$[0.716,0.772]$}}

\newcommand{\RLFFMCell}{$0.879$\,{\scriptsize$[0.858,0.900]$}}

\newcommand{\RLDIGINCell}{\best{$0.917$\,{\scriptsize$[0.908,0.923]$}}}
\newcommand{\RLDIGINNBest}{15}

\newcommand{\PepDiNFrames}{2{,}159{,}296}
\newcommand{\PepDiNMDTrain}{215{,}936}

\newcommand{\PepDiIntMLELogZ}{$-195.323_{\pm.001}$}

\newcommand{\PepDiIntLogZ}{$-195.324_{\pm.001}$}

\newcommand{\PepDiIntCmpESS}{\best{$.503_{\pm.002}$}}
\newcommand{\PepDiIntMLECmpESS}{$.263_{\pm.020}$}
\newcommand{\PepDiIntCmpWOneRaw}{\best{$.97_{\pm.02}$}}
\newcommand{\PepDiIntMLECmpWOneRaw}{$2.61_{\pm.13}$}
\newcommand{\PepDiIntCmpWOne}{\best{$.02_{\pm.01}$}}
\newcommand{\PepDiIntMLECmpWOne}{\best{$.02_{\pm.00}$}}
\newcommand{\PepDiIntCmpJSRaw}{\best{$.25_{\pm.02}$}}
\newcommand{\PepDiIntMLECmpJSRaw}{$.58_{\pm.05}$}
\newcommand{\PepDiIntCmpJS}{\best{$.18_{\pm.01}$}}
\newcommand{\PepDiIntMLECmpJS}{$.34_{\pm.03}$}

\newcommand{\PepDiCartMLELogZ}{$-194.607_{\pm.251}$}
\newcommand{\PepDiCartMLEWOneRaw}{$20.26_{\pm.97}$}

\newcommand{\PepDiCartESS}{$.103_{\pm.023}$}
\newcommand{\PepDiCartLogZ}{$-193.249_{\pm.000}$}
\newcommand{\PepDiCartWOneRaw}{$3.89_{\pm.16}$}

\newcommand{\PepDiCartCmpESS}{\best{$.103_{\pm.023}$}}
\newcommand{\PepDiCartMLECmpESS}{$<\!10^{-3}$}
\newcommand{\PepDiCartCmpWOneRaw}{\best{$3.89_{\pm.16}$}}
\newcommand{\PepDiCartMLECmpWOneRaw}{$20.26_{\pm.97}$}
\newcommand{\PepDiCartCmpWOne}{\best{$.06_{\pm.02}$}}
\newcommand{\PepDiCartMLECmpWOne}{$1.58_{\pm.63}$}
\newcommand{\PepDiCartCmpJSRaw}{\best{$1.32_{\pm.04}$}}
\newcommand{\PepDiCartMLECmpJSRaw}{$1.85_{\pm.21}$}
\newcommand{\PepDiCartCmpJS}{\best{$.69_{\pm.06}$}}
\newcommand{\PepDiCartMLECmpJS}{$89.39_{\pm30.32}$}
\newcommand{\PepDiELo}{-109}
\newcommand{\PepDiEHi}{25}
\newcommand{\PepDiDLogZ}{$2.075_{\pm.000}$}
\newcommand{\PepDiDLogZExpected}{2.079}
\newcommand{\PepDiMethyls}{3}
\newcommand{\PepDiNSamples}{2M}
\newcommand{\PepTetraNFrames}{2{,}275{,}584}
\newcommand{\PepTetraNMDTrain}{227{,}584}

\newcommand{\PepTetraIntMLELogZ}{$-446.143_{\pm.003}$}

\newcommand{\PepTetraIntLogZ}{$-446.141_{\pm.001}$}

\newcommand{\PepTetraIntCmpESS}{\best{$.062_{\pm.004}$}}
\newcommand{\PepTetraIntMLECmpESS}{$.041_{\pm.009}$}
\newcommand{\PepTetraIntCmpWOneRaw}{\best{$4.84_{\pm.19}$}}
\newcommand{\PepTetraIntMLECmpWOneRaw}{$9.38_{\pm.62}$}
\newcommand{\PepTetraIntCmpWOne}{\best{$.06_{\pm.02}$}}
\newcommand{\PepTetraIntMLECmpWOne}{$.09_{\pm.01}$}
\newcommand{\PepTetraIntCmpJSRaw}{$2.01_{\pm.35}$}
\newcommand{\PepTetraIntMLECmpJSRaw}{\best{$.76_{\pm.03}$}}
\newcommand{\PepTetraIntCmpJS}{\best{$.83_{\pm.05}$}}
\newcommand{\PepTetraIntMLECmpJS}{$1.23_{\pm.24}$}

\newcommand{\PepTetraCartMLELogZ}{$-444.613_{\pm1.396}$}

\newcommand{\PepTetraCartLogZ}{$-442.693_{\pm.150}$}

\newcommand{\PepTetraCartCmpESS}{$<\!10^{-3}$}
\newcommand{\PepTetraCartMLECmpESS}{$<\!10^{-3}$}
\newcommand{\PepTetraCartCmpWOneRaw}{\best{$28.96_{\pm.85}$}}
\newcommand{\PepTetraCartMLECmpWOneRaw}{$95.76_{\pm4.52}$}
\newcommand{\PepTetraCartCmpWOne}{\best{$3.09_{\pm1.39}$}}
\newcommand{\PepTetraCartMLECmpWOne}{$9.27_{\pm4.88}$}
\newcommand{\PepTetraCartCmpJSRaw}{$24.85_{\pm1.73}$}
\newcommand{\PepTetraCartMLECmpJSRaw}{\best{$7.48_{\pm.36}$}}
\newcommand{\PepTetraCartCmpJS}{\best{$90.85_{\pm47.31}$}}
\newcommand{\PepTetraCartMLECmpJS}{$304.32_{\pm175.73}$}

\newcommand{\PepTetraELo}{-49}
\newcommand{\PepTetraEHi}{154}
\newcommand{\PepTetraDLogZ}{$3.448_{\pm.150}$}
\newcommand{\PepTetraDLogZExpected}{3.466}
\newcommand{\PepTetraMethyls}{5}

\newcommand{\PepDiIntECalls}{3.0}

\newcommand{\PepDiCartECalls}{24.7}

\newcommand{\PepTetraIntECalls}{6.0}

\newcommand{\PepTetraCartECalls}{49.6}
\newcommand{\PepFECap}{8}
\newcommand{\PepNSeeds}{3}

\newcommand{\FESEst}{$29.436_{\pm.004}$}

\newcommand{\FESESS}{$.103_{\pm.014}$}

\newcommand{\FESNPaths}{249{,}943}
\newcommand{\FESNStart}{24{,}994}
\newcommand{\FESNRef}{210{,}334}

\newcommand{\FESRef}{$29.428_{\pm.003}$}
\newcommand{\FESRefPlain}{29.428}
\newcommand{\FESRefErr}{0.003}
\newcommand{\FESFEAT}{$29.38_{\pm.04}$}
\newcommand{\FESFEATPlain}{29.38}
\newcommand{\FESFEATErr}{0.04}

\newcommand{\FETEst}{$-4.458_{\pm.017}$}

\newcommand{\FETESS}{$.047_{\pm.036}$}

\newcommand{\FETNPaths}{246{,}191}
\newcommand{\FETNStart}{24{,}619}
\newcommand{\FETNRef}{3{,}752}

\newcommand{\FETRef}{$-4.44_{\pm.02}$}
\newcommand{\FETRefPlain}{-4.44}
\newcommand{\FETRefErr}{0.02}
\newcommand{\FETFEAT}{$-4.56_{\pm.08}$}
\newcommand{\FETFEATPlain}{-4.56}
\newcommand{\FETFEATErr}{0.08}
\newcommand{\FENSeeds}{3}

\newcommand{\FESConvYmin}{29.15}\newcommand{\FESConvYmax}{29.65}
\newcommand{\FETConvYmin}{-4.8}\newcommand{\FETConvYmax}{-4.15}
\newcommand{\FESTrainYmin}{29.25}\newcommand{\FESTrainYmax}{29.7}
\newcommand{\FETTrainYmin}{-4.75}\newcommand{\FETTrainYmax}{-4.0}

\makeatletter
\newcommand{\figcaption}{\def\@captype{figure}\caption}
\newcommand{\tabcaption}{\def\@captype{table}\caption}
\makeatother

\newcommand{\method}{\texorpdfstring{\mbox{DIG-IN}}{DIG-IN}}
\newcommand{\DMu}{\Delta\mu}
\newcommand{\kB}{k_{\mathrm{B}}}
\DeclareMathOperator{\TVD}{TVD}
\newcommand{\mTVD}{\TVD_{\mathrm{mode}}}
\DeclareMathOperator{\logsumexp}{logsumexp}

\definecolor{cBest}{HTML}{D6E4F2}
\newcommand{\best}[1]{\cellcolor{cBest}#1}
\definecolor{cDSI}{HTML}{4477AA}
\definecolor{cMDNS}{HTML}{EE9933}
\definecolor{cFrOGS}{HTML}{EE9933}
\definecolor{cRef}{gray}{0.55}

\pgfplotsset{
  colormap={ising}{[1pt] rgb255(0pt)=(68,119,170) rgb255(1pt)=(204,102,119)},
  colormap={potts}{[1pt] rgb255(0pt)=(68,119,170) rgb255(1pt)=(204,102,119)
                         rgb255(2pt)=(34,136,85)},
  cfgpanel/.style={
    scale only axis, width=2.5cm, height=2.5cm,
    enlargelimits=false, axis on top, axis line style={black!45},
    xtick=\empty, ytick=\empty,
    title style={font=\small, yshift=-0.8ex}},
  pmpanel/.style={
    scale only axis, width=5.2cm, height=3.0cm,
    tick label style={font=\scriptsize}, label style={font=\small},
    legend style={font=\scriptsize, draw=none, fill=none, at={(0.5,1.02)},
                  anchor=south, legend columns=3, column sep=0.5em},
    tick align=outside, axis line style={black!60},
    xmin=-1, xmax=1, ymin=0},
}

\pgfplotsset{
  fpmap/.style={
    scale only axis, width=2.45cm, height=2.45cm,
    enlargelimits=false, axis on top, axis line style={black!60},
    tick label style={font=\scriptsize}, label style={font=\small},
    title style={font=\small, yshift=-0.8ex}, tick align=outside,
    scaled x ticks=false, colormap/viridis},
  fpdata/.style={matrix plot*, point meta=explicit, mesh/cols=9},
  fpising/.style={xtick={-0.1,0,0.1}, xticklabels={$-0.1$,0,$0.1$},
    ytick={0.4,0.5,0.6}},
  fppotts/.style={xtick={-0.01,0,0.01}, xticklabels={$-0.01$,0,$0.01$},
    ytick={0.8,1.0,1.2},
    yticklabel style={/pgf/number format/.cd, fixed, fixed zerofill, precision=1}},
}
\pgfplotsset{
  worstpanel/.style={
    scale only axis, width=4.6cm, height=1.45cm,
    tick label style={font=\scriptsize}, label style={font=\small},
    title style={font=\small, yshift=-0.6ex},
    tick align=outside, axis line style={black!60},
    ymin=0, ymax=\FPWorstYmax, ytick={0,0.1,0.2},
    scaled x ticks=false, enlarge x limits=false},
  worstneg/.style={xmin=-1.0004, xmax=-0.9524, xtick={-1,-0.98,-0.96}},
  worstpos/.style={xmin=0.9524, xmax=1.0004, xtick={0.96,0.98,1}},
  refhist/.style={const plot, fill=black!12, draw=black!35, area legend},
  rawhist/.style={const plot, cDSI, line width=0.9pt},
  rwhist/.style={const plot, black, densely dashed, line width=0.8pt},
  mlehist/.style={const plot, cMDNS, line width=0.8pt},
}

\definecolor{cNoDamp}{HTML}{EE9933}
\definecolor{cNoGibbsT}{HTML}{CC6677}
\definecolor{cNoGibbsP}{HTML}{228833}
\definecolor{cNoIW}{HTML}{AA4499}
\pgfplotsset{
  ablpanel/.style={
    scale only axis, width=3.0cm, height=1.95cm,
    tick label style={font=\scriptsize}, label style={font=\small},
    title style={font=\small, yshift=-0.8ex},
    tick align=outside, axis line style={black!60},
    xmin=0, xmax=150000, xtick={0,50000,100000,150000},
    xticklabels={0,50k,100k,150k}, scaled x ticks=false,
    unbounded coords=jump},
  ablline/.style={line width=0.9pt, mark=none},
}
\tikzset{betac/.style={white, densely dashed, line width=0.6pt}}

\definecolor{viridisLow}{rgb}{0.267004,0.004874,0.329415}
\pgfplotsset{
  margbase/.style={
    scale only axis, width=1.05cm, height=2.3cm,
    ymin=-1.0125, ymax=1.0125, xmin=0, enlarge x limits={upper, value=0.06},
    xtick=\empty, tick label style={font=\scriptsize}, label style={font=\small},
    title style={font=\small, yshift=-0.7ex},
    tick align=outside, axis line style={black!60}},
  margL/.style={margbase, x dir=reverse, ytick={-1,0,1}},
  margR/.style={margbase, ytick=\empty},
  refstair/.style={fill=black!14, draw=black!40, line width=0.4pt},
  modstair/.style={cDSI, line width=0.8pt},
  rwstair/.style={black, densely dashed, line width=0.7pt},
  flowT/.style={
    scale only axis, width=2.6cm, height=2.3cm,
    view={0}{90}, xmin=0, xmax=1, ymin=-1.0125, ymax=1.0125,
    xtick={0,0.5,1}, ytick=\empty, colormap/viridis,
    point meta min=0, point meta max=1, axis on top,
    axis background/.style={fill=viridisLow},
    tick label style={font=\scriptsize}, label style={font=\small},
    title style={font=\small, yshift=-0.7ex}, tick align=outside,
    axis line style={black!60}},
  zoomR/.style={
    scale only axis, width=0.8cm, height=1.0cm,
    xmin=0, enlarge x limits={upper, value=0.06}, xtick=\empty,
    yticklabel pos=right, tick label style={font=\tiny},
    yticklabel style={/pgf/number format/.cd, fixed, precision=2},
    title style={font=\small, yshift=-0.7ex},
    tick align=outside, axis line style={black!60}},
  zoomPos/.style={zoomR, ymin=0.925, ymax=1.0125, ytick={0.96,1}},
  zoomNeg/.style={zoomR, ymin=-1.0125, ymax=-0.925, ytick={-1,-0.96}},
  flowdata/.style={surf, shader=interp, point meta=explicit,
                   mesh/cols=\BrCrossNTimes},
  brcfgpanel/.style={
    scale only axis, width=1.9cm, height=1.9cm,
    enlargelimits=false, axis on top, axis line style={black!45},
    xtick=\empty, ytick=\empty,
    title style={font=\small, yshift=-0.8ex},
    colormap={ising}{[1pt] rgb255(0pt)=(68,119,170) rgb255(1pt)=(204,102,119)},
    point meta min=0, point meta max=1},
  brcfgdata/.style={matrix plot*, point meta=explicit, mesh/cols=\BrL},
}

\pgfplotsset{
  alloypanel/.style={
    scale only axis, width=3.2cm, height=2.75cm,
    tick label style={font=\scriptsize}, label style={font=\small},
    title style={font=\small, yshift=-0.8ex},
    tick align=outside, axis line style={black!60},
    ytick={200,400,600,800}},
  pd/.style={alloypanel, height=2.1cm, xmin=0, xmax=1, ymin=175, ymax=925,
    xtick={0,0.25,0.5,0.75,1}, xticklabels={0,,0.5,,1},
    unbounded coords=jump},
  refties/.style={black!13, line width=2.4pt, line cap=butt},
  refbounds/.style={black!55, line width=0.8pt, densely dashed, line join=round},
  ties/.style={#1!20, line width=0.9pt},
  bounds/.style={#1, line width=1.0pt, line join=round},
  fmixref/.style={black!65, line width=0.8pt},
  fmixback/.style={black!18, line width=0.8pt},
  fmixmodel/.style={cDSI, only marks, mark=*, mark size=0.5pt},
  tangentref/.style={black, line width=0.7pt, densely dashed, mark=*,
                     mark size=1.0pt, mark options={solid}},
  tangentmodel/.style={cDSI, line width=0.7pt, densely dashed, mark=*,
                       mark size=1.0pt, mark options={solid}},
  essmap/.style={alloypanel, width=3.05cm, enlargelimits=false, axis on top,
    colormap/viridis, point meta min=0, point meta max=1},
  essdata/.style={matrix plot*, point meta=explicit, mesh/cols=#1},
}

\definecolor{cAg}{RGB}{192,192,192}
\definecolor{cPd}{RGB}{0,105,133}
\definecolor{cCu}{RGB}{200,128,51}
\definecolor{cAu}{RGB}{255,209,35}
\tikzset{fccedge/.style={black!45, line width=0.35pt}}
\newcommand{\fccr}{0.155}
\newcommand{\fccatom}[4]{\shade[ball color=#4] (#1,#2,#3) circle (\fccr);}
\newcommand{\fccCu}[3]{\fccatom{#1}{#2}{#3}{cCu}}
\newcommand{\fccAu}[3]{\fccatom{#1}{#2}{#3}{cAu}}
\newcommand{\fccmixed}[3]{%
  \coordinate (fccc) at (#1,#2,#3);
  \shade[ball color=cAg] (fccc) circle (\fccr);
  \begin{scope}
    \clip ([shift={(0,-\fccr)}]fccc) rectangle ([shift={(\fccr,\fccr)}]fccc);
    \shade[ball color=cPd] (fccc) circle (\fccr);
  \end{scope}}
\newcommand{\fcccell}[3]{%
  \draw[fccedge] (0,0,1) -- (1,0,1) -- (1,1,1) -- (0,1,1) -- cycle;
  \draw[fccedge] (0,0,0) -- (0,0,1) (1,0,0) -- (1,0,1)
                 (1,1,0) -- (1,1,1) (0,1,0) -- (0,1,1);
  \draw[fccedge] (0,0,0) -- (1,0,0) -- (1,1,0) -- (0,1,0) -- cycle;
  #1{0}{0}{1} #1{1}{0}{1} #1{0}{1}{1} #1{1}{1}{1} #2{0.5}{0.5}{1}
  #3{0.5}{0}{0.5} #3{0}{0.5}{0.5} #3{1}{0.5}{0.5} #3{0.5}{1}{0.5}
  #1{0}{0}{0} #1{1}{0}{0} #1{0}{1}{0} #1{1}{1}{0} #2{0.5}{0.5}{0}}

\usepackage{etoolbox}
\newcommand{\panelmark}[2]{%
  \ifbool{#1#2Placeholder}{%
    \node[anchor=north east, font=\tiny\bfseries, text=red!75!black,
          fill=white, fill opacity=0.85, text opacity=1, inner sep=1.5pt]
      at (rel axis cs:0.98,0.98) {PLACEHOLDER};
  }{}%
}

\definecolor{cREPPO}{HTML}{EE6677}
\definecolor{cPPO}{HTML}{CCBB44}
\definecolor{cSPO}{HTML}{66CCEE}
\definecolor{cTRUDI}{HTML}{228833}
\definecolor{cAMDP}{HTML}{AA3377}
\definecolor{cDPPO}{HTML}{BBBBBB}
\definecolor{cFPO}{HTML}{EE9933}
\definecolor{cFFM}{HTML}{555555}
\pgfplotsset{
  rlpanel/.style={
    scale only axis, width=3.8cm, height=2.1cm,
    tick label style={font=\scriptsize}, label style={font=\scriptsize},
    tick align=outside, axis line style={black!60},
    grid=major, grid style={black!12, line width=0.3pt},
    xmin=0, xmax=52.43, xtick={0,10,20,30,40,50},
    ymin=0, ymax=1, ytick={0,0.25,0.5,0.75,1}, yticklabels={0,,0.5,,1},
    unbounded coords=jump},
  rlband/.style={fill=#1, draw=none, fill opacity=0.15, forget plot},
  rlline/.style={#1, line width=0.8pt, mark=none},
}
\newcommand{\rlkey}[1]{\textcolor{#1}{\rule[0.5ex]{1em}{1.2pt}}}
\newcommand{\dashkey}[1]{\textcolor{#1}{%
  \rule[0.5ex]{0.24em}{1.2pt}\hspace{0.14em}\rule[0.5ex]{0.24em}{1.2pt}\hspace{0.14em}%
  \rule[0.5ex]{0.24em}{1.2pt}}}

\pgfplotsset{
  ramapanel/.style={
    scale only axis, width=1.55cm, height=1.55cm,
    xmin=-180, xmax=180, ymin=-180, ymax=180, enlargelimits=false, axis on top,
    xtick={-180,0,180}, ytick={-180,0,180}, xticklabels={$-\pi$,0,$\pi$},
    yticklabels={$-\pi$,0,$\pi$}, minor xtick={-90,90}, minor ytick={-90,90},
    tick label style={font=\tiny}, label style={font=\scriptsize},
    title style={font=\scriptsize, yshift=-0.9ex, align=center},
    tick align=outside, axis line style={black!60},
    xlabel shift=-3pt, ylabel shift=-4pt},
  enpanel/.style={
    scale only axis, width=2.6cm, height=1.45cm,
    tick label style={font=\tiny}, label style={font=\scriptsize},
    title style={font=\scriptsize, yshift=-0.8ex},
    tick align=outside, axis line style={black!60},
    ymin=0, ytick=\empty, enlarge x limits=false, scaled x ticks=false},
  convbase/.style={
    scale only axis, width=2.75cm, height=1.95cm,
    tick label style={font=\scriptsize}, label style={font=\scriptsize},
    title style={font=\scriptsize, yshift=-0.8ex},
    tick align=outside, axis line style={black!60},
    grid=major, grid style={black!10, line width=0.3pt},
    unbounded coords=jump, restrict y to domain=-1000:1000},
  convpanel/.style={convbase, xmode=log, log basis x=10},
  trainpanel/.style={convbase, width=2.6cm, height=1.45cm, xmin=0, xmax=150000,
    xtick={0,50000,100000,150000}, xticklabels={0,50k,100k,150k}, scaled x ticks=false},
  convseed/.style={cDSI, line width=0.7pt, mark=none},
  convfep/.style={cMDNS, line width=0.8pt, mark=none},
}
\newcommand{\refband}[4]{%
  \fill[#3, opacity=0.16] ({rel axis cs:0,0} |- {axis cs:1,#1-#2}) rectangle
                          ({rel axis cs:1,0} |- {axis cs:1,#1+#2});
  \draw[#3, line width=0.7pt, #4] ({rel axis cs:0,0} |- {axis cs:1,#1}) --
                                  ({rel axis cs:1,0} |- {axis cs:1,#1});}
\newcommand{\ramacolorbar}[3][$\ln p(\phi,\psi)$]{%
  \begin{axis}[at={(#2)}, anchor=north west, xshift=0.18cm,
      hide axis, scale only axis, width=1pt, height=#3,
      xmin=0, xmax=1, ymin=0, ymax=1,
      colormap/viridis, point meta min=-\PepFECap, point meta max=0,
      colorbar, colorbar style={width=0.16cm, ytick={-8,-4,0},
        tick label style={font=\tiny}, ylabel={#1},
        ylabel style={font=\scriptsize}}]
    \addplot[draw=none] coordinates {(0,0)};
  \end{axis}}
\newcommand{\rama}[1]{\addplot graphics[xmin=-180, xmax=180, ymin=-180, ymax=180] {#1.png};}

\newcommand{\dmcdomain}[2]{%
  \begin{minipage}[t]{\dimexpr\linewidth/5-0.6em\relax}\centering
    \includegraphics[width=\linewidth]{figures/dmc/#1.png}\\[0.4ex]
    \fontsize{6}{7.2}\selectfont #2%
  \end{minipage}}
\newsavebox{\rltasksbox}


\usetikzlibrary{decorations.pathreplacing,shapes.misc}

\definecolor{phS}{HTML}{4F6F8F} %
\definecolor{phT}{HTML}{8A5F73} %

\definecolor{cA}{HTML}{6F8FAF}
\definecolor{cB}{HTML}{B07979}
\definecolor{cC}{HTML}{78977A}
\definecolor{cM}{HTML}{D9D9D9}

\usepackage{cancel}
\usepackage{float}
\floatstyle{ruled}
\newfloat{algorithm}{tbp}{loa}
\floatname{algorithm}{Algorithm}

\iclrfinalcopy
\begin{document}
\maketitle
\ificlrfinal
  \lhead{Preprint}
\else
  \lhead{Under review as a conference paper at ICLR 2027}
\fi
\maketitle
\ificlrfinal
  \lhead{Preprint.}
\else
  \lhead{Under review as a conference paper at ICLR 2027}
\fi

\begin{abstract}
Sampling from discrete, unnormalized distributions without access to data is a
challenging problem. Neural samplers offer a promising approach by training
generative models from density evaluations directly. However, existing discrete neural samplers are prone to mode collapse, come without
convergence guarantees when trained via fixed-point iterations, and are often tied to a specific reference process such as masked or uniform diffusion. 
In this work,
we introduce Discrete Gibbs Iterative Neural Sampler, a fixed-point neural sampler that addresses these limitations,
enabling efficient, scalable learning, substantially reducing mode collapse in
practice. 
Besides masked diffusion, the framework also extends to transport between arbitrary distributions. 
DIG-IN mitigates mode collapse on Ising and Potts models and scales to amortized sampling with up to 2,500 sites, recovers alloy phase diagrams, and learns to sample biomolecular Boltzmann distributions from energy alone.
Its transport formulation accurately estimates solvation and conformational free energies. 
As a maximum-entropy RL policy, DIG-IN outperforms the Gaussian, diffusion, and flow policies we compare against across 20 continuous-control tasks.

\vspace{1ex}
{\centering
\faGlobe~\href{https://denisbless.github.io/DIG_IN_website/}{Project page}\qquad
\faGithub~\href{https://github.com/jiajunhe98/DIG-IN-Fixed-point-neural-samplers-on-discrete-spaces}{Code}\par}
\end{abstract}

\section{Introduction}
Sampling from discrete distributions is a fundamental problem in physics, chemistry, and machine learning, yet remains computationally intractable in many settings of practical interest. Consider configurations $x= (x^1,\ldots,x^D)\in\mathcal{X}=\mathcal{A}^D$, where $D$ is the number of sites and each $x^d$ takes values in a finite alphabet $\mathcal{A}$. Given an unnormalized density $\Tilde{p}:\mathcal{X}\to\mathbb{R}$, the target distribution is
\begin{equation}
  \pi(x)=\frac{\Tilde{p}(x)}{Z},
  \qquad
  Z=\sum_{x\in\mathcal{X}}\Tilde{p}(x).
  \label{eq:target}
\end{equation}
This formulation encompasses spin configurations, protein sequences, and distributions over solutions to combinatorial problems. Although the unnormalized density of an individual configuration can be evaluated, the exponentially large state space generally precludes enumeration. Markov chain Monte Carlo (MCMC) \citep{brooks2011handbook} avoids evaluating $Z$, but can mix prohibitively slowly when high probability regions are separated by high energy barriers or narrow transition pathways.

The continuous counterpart of this problem has seen substantial progress through \emph{neural samplers}, generative models trained on a target distribution specified by its energy function without precurated data. 
Boltzmann generators pioneered this idea for molecular equilibrium sampling with normalizing flows but suffered from scaling issues to systems of practical interest~\citep{noe2019boltzmann}. More recent methods~\citep{blessing2026bridge,havens2026flow} have developed scalable training procedures for diffusion models based on fixed-point iterations by exploiting gradient information of the target density combined with efficient reuse of evaluated configurations. 
These advances have enabled neural samplers to tackle realistic materials systems, with substantial improvements in efficiency compared to classical Monte Carlo methods~\citep{cheng2026atlas}.

\begin{figure}[t]
\centering

\begin{minipage}[c]{0.64\linewidth}
    \centering
    \includegraphics[width=\linewidth]{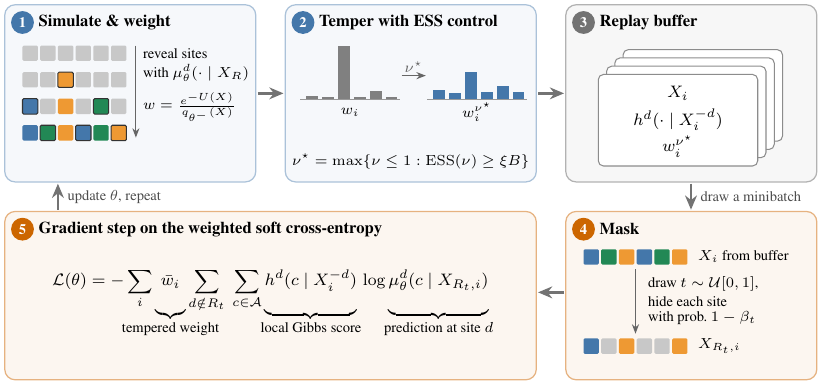}\vspace{-6pt}
    \caption{\textbf{DIG-IN iteration.} One training iteration consists of data collection, importance weight smoothing using ESS control, and parameter updates using a weighted cross-entropy loss.}
    \label{fig:alg_illustration}
\end{minipage}\hfill
\begin{minipage}[c]{0.35\linewidth}
\centering
\resizebox{\linewidth}{!}{%
\begin{minipage}{1.28\linewidth}
\begin{algorithm}[H]
\caption{DIG-IN}
\label{alg:masked}
\small

\textbf{Input:} target $\tilde p$, initial model, training loss;
buffer size $B$, refill size $b$, minibatch size $m$,
inner steps $K$, ESS fraction $\xi$.
\par\smallskip

\textbf{Repeat until training is complete:}
\begin{enumerate}[label=\arabic*:,leftmargin=1.7em,
  labelsep=0.4em,itemsep=2pt,topsep=3pt,parsep=0pt]
  \item Generate $b$ samples from current sampler, and append them to buffer.
  \item Store their sample and importance weights $w$; keep the newest $B$ buffer samples.
  \item \textbf{for} $k=1,\ldots,K$ \textbf{do}
  \item \hspace*{0.7em} Draw $m$ buffer entries uniformly.
  \item \hspace*{0.7em} Mask sites randomly for each sample.
  \item \hspace*{0.7em} Calculate  Gibbs labels for unmasked sites. 
  \item  \hspace*{0.7em} Choose $\nu^\star$ by ESS control ($\xi$).
  \item \hspace*{0.7em} Update model  with $\bar w$-weighted  cross-entropy loss with Gibbs label.
  \item \textbf{end for}
\end{enumerate}
\textbf{Return:} the trained sampler.

\end{algorithm}
\end{minipage}%
}
\end{minipage}

\end{figure}

In this work, we introduce Discrete Gibbs Iterative Neural Sampler (DIG-IN), a fixed-point neural sampler that is inspired by the principles underlying the successes of continuous neural samplers. 
At the core of our approach are \emph{Gibbs labels}, i.e., normalized distributions over single-site perturbations of a discrete configuration that provide a discrete analogue of gradient. 
In summary, DIG-IN \textit{i)} incorporates these Gibbs labels which provide a strong learning signal and help to avoid mode collapse (\cref{sec:gibbs_supervision}), \textit{ii)} achieves a balance between variance and bias through tempered importance weighting, and enjoys convergence guarantees for certain setting (\cref{sec:convergence}), and \textit{iii)} is extendable to allow for a transport between any two pairs of distributions, whereas existing discrete constructions are often tied to specific reference distributions or corruption processes (\cref{sec:general_interpolants}). 

Experiments on Ising and Potts models show that DIG-IN is the only method that avoids mode collapse in the ordered phase while giving accurate $\ln Z$ estimates. It scales
to amortized sampling over thermodynamic conditions on lattices with up to $D = 2500$ sites, and it accurately recovers the phase diagrams of two alloy systems (\cref{sec:experiments}). On a binned configuration space it acts as a Biomolecular ensemble generator for peptides trained from the energy alone, and its bridge estimates solvation and conformational free energies of alanine dipeptide.

\vspace{-5pt}
\section{Background: CTMCs and Masked Diffusion Models}
\label{sec:background}
\vspace{-5pt}
This section introduces the basics of CTMCs and masked diffusion.
These models can be formulated and trained in several ways; here, we focus on a simple background for our method.

Throughout, we write $N=|\mathcal{A}|$ for the alphabet size,
$x^{-d}$ for all coordinates except site $d$, and $x_R$ for the
values at a set of sites $R$. The notation $x^{d\leftarrow c}$
means replacing $x^d$ by $c$.

\paragraph{Continuous-time Markov chains (CTMCs).}
A CTMC evolves on a discrete state space by remaining in its current
state for a random time and then jumping to another. For a process
$(X_t)_{t\in[0,1]}$, let
$P_{t,t+h}(x,y)=\mathbb{P}(X_{t+h}=y\mid X_t=x)$ denote its
transition probabilities over an interval of length $h>0$.
The \emph{generator}, or rate matrix, describes their instantaneous
change, $Q_t=\lim_{h\downarrow0}(P_{t,t+h}-I)/h$,
where $I$ is the identity matrix.
This definition gives the small-time expansion
\begin{equation}
 \mathbb{P}(X_{t+h}=y\mid X_t=x)
 =\mathbf{1}_{\{y=x\}}+hQ_t(x,y)+o(h).
\end{equation}
For $y\neq x$, the entry $Q_t(x,y)\geq0$ is the rate of jumping
from $x$ to $y$. The diagonal entry
$Q_t(x,x)=-\sum_{y\neq x}Q_t(x,y)$ is the negative total rate of
leaving $x$, ensuring that each row sums to zero.

The process run backward in time is also a CTMC. 
Write $p_t(x)=\mathbb{P}(X_t=x)$ for the marginal distribution.
Its reverse generator
$\bar Q_t$, indexed by the same $t$ but moving backward in time, satisfies
\begin{equation}
 \bar Q_t(x,y)=Q_t(y,x) {p_t(y)}/{p_t(x)},
 \qquad y\neq x,\quad p_t(x)>0.
 \label{eq:nelson_relation}
\end{equation}
The marginal density ratio is known as \emph{concrete scores}
\citep{lou2024discrete}
This is the discrete analogue of Nelson's time-reversal relation
\citep{campbell2022continuous,sun2023score}. 
For generation, it is often sufficient to simulate the process in a single time direction. In contrast, path-space importance reweighting and  free-energy estimators generally require evaluating both the forward and reverse path probabilities, and therefore access to both generators~\citep{he2025feat,he2026free}.

As is common in discrete generative modeling, we consider processes
that change one site at each jump. For a state $x$ with $x^d=a$,
we denote the rate of replacing $a$ by $c\neq a$ by
\begin{equation}
 r_t^d(a,c\mid x^{-d})
 :=Q_t\bigl(x,x^{d\leftarrow c}\bigr).
\end{equation}
The reverse single-site rates $\bar r_t^d$ are defined analogously
using $\bar Q_t$. Together, the single-site rates determine when the
process jumps, which site changes, and its new symbol.

\paragraph{Masked diffusion.} 
A commonly used construction for CTMC in generative modeling is masked diffusion, which adds a new mask symbol $\mathtt{M}\notin\mathcal{A}$ to the alphabet. Masked diffusion then learns to reverse a process that progressively masks symbols until all tokens are fully masked.
We first introduce a decreasing schedule $\alpha_t$, with $\alpha_0=1$
and $\alpha_1=0$. Given a clean sample $X_0\sim\pi$, each coordinate
of $X_t$ retains its original symbol with probability $\alpha_t$ and
is masked otherwise, independently across sites. 
Once masked, a site
stays masked. 
Writing $\dot\alpha_t=\mathrm{d}\alpha_t/\mathrm{d}t$, the instantaneous masking rate is 
$r_t^d(c,\mathtt{M}\mid x^{-d})=-\dot\alpha_t/\alpha_t$ for $c\in\mathcal{A}$
and $t<1$. 

By \Cref{eq:nelson_relation}, reversing these masking jumps requires
the concrete scores. For masking, it suffices to learn a categorical
denoiser. Let $R=\{d:x^d\neq\mathtt{M}\}$ denote the revealed sites
of a partially masked state $x$. For $d\notin R$ and
$c\in\mathcal{A}$, the desired predictor is
$\mu_*^d(c\mid x_R)=\mathbb{P}_{X\sim\pi}(X^d=c\mid X_R=x_R)$.
This  recovers the time-reversal of a masking process as  the  concrete score is represented implicitly:
\begin{align}
    \frac{p_t(x^{d\leftarrow c})}{p_t(x)}
=\frac{\alpha_t}{1-\alpha_t}\mu_*^d(c\mid x_R)
\end{align}
As we can see, even the reverse process requires a time-dependent generator, a time-independent denoiser suffices.
Consequently, the generation can also be done by revealing one site at a time in a  random
order, yielding any-order autoregressive generation
\citep{hoogeboom2022autoregressive}.
 
Given target samples, the predictor can be learned by letting them reconstruct masked tokens
\citep{sahoo2024simple,shi2024simplified}: draw $X\sim\pi$ and
$t\sim\mathrm{Unif}(0,1)$, then mask each site independently with
probability $1-\alpha_t$ and write $R_t$ for the set of sites that remain
revealed. Then we train the predictor by the cross-entropy loss
\begin{equation}
\begin{aligned}
    \mathcal{L}_{\mathrm{CE}}(\theta)
    &=
    -\mathbb{E}_{X,t,R_t}
   [
        \sum_{d\notin R_t}
        \log \mu_\theta^d(X^d\mid X_{R_t})
   ] =
    -\mathbb{E}_{X,t,R_t}
   [
        \sum_{d\notin R_t}
        \sum_{c\in\mathcal A}
        \mathbf{1}_{\{X^d=c\}}
        \log \mu_\theta^d(c\mid X_{R_t})
 ],
\end{aligned}
\label{eq:masked_ce}
\end{equation}
where $\mathbf{1}_{\{X^d=c\}}$ is the one-hot label for the
observed symbol at site $d$.

\vspace{-5pt}
\section{Discrete Gibbs Iterative Neural Sampler} \label{sec:method}
\vspace{-5pt}
Here, we introduce Discrete Gibbs Iterative Neural Sampler (DIG-IN). We first derive
Gibbs labels from single-site energy differences and show that they provide a
lower-variance training signal for the masked denoiser (\Cref{sec:gibbs_supervision}).
Since target samples are unavailable, we then turn this into a fixed-point iteration
on self-generated samples, corrected by ESS-tempered importance weights, and state
convergence analysis (\Cref{sec:convergence}). Finally, we extend the
construction beyond masked diffusion to general discrete interpolants between two
distributions (\Cref{sec:general_interpolants}).

\vspace{-5pt}
\subsection{From energy differences to Gibbs soft labels}
\label{sec:gibbs_supervision}
\vspace{-5pt}
In continuous spaces, the score function $\nabla_x \log \pi(x)$ describes how a configuration $x$ changes to increase its density.
This information is crucial, as the unnormalized density only tells if a sample is good or bad, while the gradient tells how to change the sample,  which provides stronger supervision signal for the neural sampler.
For discrete configurations, the analogue is to compare the energies obtained by replacing the symbol at one site while holding the others fixed.
These comparisons determine the target's conditional distribution
over possible replacements, which we will use as additional learning signal for training a neural sampler as outlined below.

The gradient is an infinitesimal perturbation in continuous space.
The discrete counterpart is a single-site replacement. 
For a complete configuration $x$, replacing the symbol at site $d$ by $c$ gives the energy difference, and plays the role of a gradient \citep{meng2022concrete}:
\begin{equation}
    \nabla_{d,c}U(x)
    :=
    -\log {\pi(x^{d\leftarrow c})}+\log {\pi(x)}.
    \label{eq:discrete_energy_difference}
\end{equation}  
Normalizing the corresponding Boltzmann weights gives
\begin{equation}
    h^d(c\mid x^{-d})
    =
    \frac{\exp[-\nabla_{d,c}U(x)]}
         {\sum_{a\in\mathcal A}\exp[-\nabla_{d,a}U(x)]}
    =
    \mathbb P_{X\sim\pi}(X^d=c\mid X^{-d}=x^{-d}).
    \label{eq:heatbath}
\end{equation}
We call this distribution a \emph{Gibbs soft label}, since it
is exactly the conditional distribution used by Gibbs sampling.
Instead of asking the predictor to reconstruct only the sampled
symbol, this label specifies the relative probabilities of
all its alternatives under the target.

We use these labels to replace the one-hot targets in \Cref{eq:masked_ce}. 
Assume for now that we have samples from our target distribution $\pi$. We can then train our network with the cross-entropy loss:
\begin{equation}
    \mathcal L_{\mathrm{Gibbs}}(\theta)
    =
    -\mathbb E_{X\sim\pi,t,R_t}
    [
        \sum_{d\notin R_t}
        \sum_{c\in\mathcal A}
        h^d(c\mid X^{-d})
        \log\mu_\theta^d(c\mid X_{R_t})
   ].
    \label{eq:gibbs_ce}
\end{equation}

This has the same optimum as  \Cref{eq:masked_ce} yet with lower per-site  variance as stated below.

\begin{proposition}[Consistency and per-site variance reduction]
\label{prop:gibbs_consistency}
Let $X\sim\pi$. Fix a reveal set $R$, a masked site $d\notin R$,
and predictor parameters $\theta$. Define
\begin{align}
\ell_{\mathrm{CE}}^d
&=
-\log\mu_\theta^d(X^d\mid X_R),
\quad
\ell_{\mathrm{Gibbs}}^d
=
-\sum_{c\in\mathcal A}
h^d(c\mid X^{-d})
\log\mu_\theta^d(c\mid X_R).
\end{align}
If $\ell_{\mathrm{CE}}^d$ has finite second moment, then
\begin{align}
\mathbb E[\ell_{\mathrm{Gibbs}}^d\mid X_R]
&=
\mathbb E[\ell_{\mathrm{CE}}^d\mid X_R],\quad
\operatorname{Var}(\ell_{\mathrm{Gibbs}}^d\mid X_R)
\le
\operatorname{Var}(\ell_{\mathrm{CE}}^d\mid X_R).
\end{align}
\end{proposition}
This results follow from tower property and Rao-Blackwellization, since $\ell_{\mathrm{Gibbs}}^d=\mathbb E[\ell_{\mathrm{CE}}^d\mid X^{-d}]$.
 
\vspace{-5pt}
\subsection{From fixed-point consistency to fixed-point convergence}
\label{sec:convergence}
 \vspace{-5pt}
The previous construction leaves one important issue unresolved.
Proposition~\ref{prop:gibbs_consistency} assumes that the completed configurations are distributed according to $\pi$, whereas obtaining such samples is precisely the problem we aim to solve.
A similar obstacle arises in continuous-space neural samplers~\citep{havens2025adjoint,liu2025adjoint,shaul2025flow,blessing2026bridge}.
These methods address it through a self-improving bootstrapping procedure: samples are first drawn from the current sampler, then used to construct supervision as if they were drawn from the target, and the sampler is subsequently refitted using this supervision.
Repeating this procedure refines the sampler closer and closer to the target.
Crucially, when the sampler already matches the target distribution, the resulting update leaves it unchanged; hence, the target distribution is a fixed point of the iteration.
However, fixed-point consistency alone does not imply fixed-point convergence.
Starting from an imperfect sampler, the distribution of training contexts can be biased, and repeatedly fitting exact local
labels under this biased distribution may not recover the target.
In this subsection, we aim to address it.

A straightforward idea is to fix this mismatch using directly importance sampling:
let $q_{\theta^-}$ denote the frozen sampler that generates the current
replay buffer. 
For a generated configuration $x$, define the
unnormalized importance weight
$w(x)
    =
    {e^{-U(x)}} / {q_{\theta^-}(x)}.$
Here, $q_{\theta^-}(x)$ denotes the model likelihood assigned by the
generation procedure to the sampled configuration. For masked diffusion,
this probability is available from the categorical probabilities used
during generation; see \Cref{app:is_f} for further details. 
Applying the importance weighting to the Gibbs-label objective recovers the population loss that would have been obtained from target
samples: 
\begin{equation}
    \mathcal L_{\mathrm{IW}}(\theta;\theta^-)
    \propto
    - {
    \mathbb E_{X\sim q_{\theta^-}}
   [
        w(X)
        \mathbb E_{t,R_t}
      [
            \sum_{d\notin R_t}
            \sum_{c\in\mathcal A}
            h^d(c\mid X^{-d})
            \log\mu_\theta^d(c\mid X_{R_t})
         ]
    ]
    }.
    \label{eq:iw_gibbs_loss}
\end{equation}
This objective has the target
masked conditional $\mu_*$ as its optimum as well.

In practice, however, when the current sampler is far from the target, the importance weights can become highly concentrated, leading to prohibitively high variance.
This is a well-known challenge in importance-weighted learning, and prior work has addressed it through trust-region constraints \citep{schulman2015trust,blessing2026trust}, proximal updates \citep{guo2026proximal}, or smoothing \citep{cheng2026atlas}.
Following these principles, we adopt \emph{tempered importance weighting}, which introduces a controlled bias at each iteration in exchange for substantially reduced variance and more stable learning.
Specifically, for a tempering exponent $\nu\in[0,1]$, we replace the
importance weight $w(X)$ in \cref{eq:iw_gibbs_loss} by $
    \bar w(X)=w(X)^\nu$ and denote the corresponding loss as $\mathcal L_{\mathrm{FP}}(\theta;\theta^-, \nu)$.
We choose $\nu$ adaptively according to the effective
sample size of the $B$ buffered configurations $x_1,\ldots,x_B$, i.e.,
$
\operatorname{ESS}(\nu, \theta^-)
=
{\left(\sum_{i=1}^B    \bar  w (x_i) \right)}^2/
{\sum_{i=1}^B    \bar  w (x_i)^2} .
$
Given a desired fraction $\xi\in(0,1)$, we use the largest exponent
that maintains a sufficiently large effective sample size,
\begin{equation}
    \nu^\star(\theta^-)
    =
    \max\left\{
        \nu\in[0,1]:
        \operatorname{ESS}(\nu, \theta^-)\geq \xi B
    \right\}.
    \label{eq:adaptive_tempering}
\end{equation}
Hence, when the current sampler already has good overlap with the target, $\nu^\star$ approaches one and the update uses nearly exact importance correction.
When the overlap is poor, the correction is automatically softened to prevent a small number of configurations from dominating
the optimization. Hence, starting from some $\theta_0$, the fixed point iteration is given by
\begin{equation}
     \nu_i
    =
    \nu^\star(\theta_i), \quad \quad 
    \theta_{i+1} = \arg \min_{\theta} \ \mathcal{L}_{\mathrm{FP}}(\theta;\theta_i, \nu_i)
\end{equation}
Under certain settings and idealized conditions stated below, we provide convergence guarantees for the population version of this tempered fixed-point iteration; proofs are in \Cref{app:convergence}.

Intuitively, tempering creates a tradeoff within each outer iteration. Full
importance correction directly reproduces the target training
distribution but may concentrate almost all weight on a few samples.
Softer correction produces a more stable, but biased, training
distribution. The key question is whether repeatedly applying these
biased updates still moves the sampler toward the target.

To formalize this, let $\Phi_\nu(g)$ denote one idealized outer
iteration with tempering exponent $\nu$: generate complete
configurations using the current predictor $g$, form the corresponding
tempered Gibbs-label objective, and return its exact population
minimizer. Thus the population iteration is
$
    g_{i+1}=\Phi_\nu(g_i).
$

\begin{theorem}[Convergence of tempered fixed-point training]
\label{thm:fixed_point_convergence}
Suppose the state space is finite, the target and categorical
predictors are strictly positive, and sampling and population fitting
are exact. Assume also that the reveal policy is independent of the
generated values and that training covers every predictor head queried
during generation.
Then, tempering changes how quickly and under what conditions the
iteration approaches the target, but not its destination:

 (i) The target predictor $\mu_*$ is a fixed point of $\Phi_\nu$ for
every $\nu\in[0,1]$. At $\nu=1$, one population update recovers
$\mu_*$ from any positive initialization.
(ii) With one fixed reveal order, the iteration converges globally to
$\mu_*$ for every fixed $\nu\in[0,1]$. It also converges when the
exponent varies across iterations but remains bounded away from zero,
which includes the ESS-adaptive rule in the idealized population
setting.
(iii) With a general value-independent distribution over reveal
orders, global convergence is guaranteed under the sufficient
condition
$
2D(1-\nu)<1.
$ 
\end{theorem}

The fixed-point iteration, including pseudocode, is illustrated in \Cref{fig:alg_illustration}.

 \vspace{-5pt}
\subsection{Beyond masked diffusion}
\label{sec:general_interpolants}
 \vspace{-5pt}
We next extend the construction beyond masked diffusion to general discrete interpolants.
Unlike the masked case where we have specified the masking process and learn the unmasking process with a single conditional predictor, for a general interpolant, however, we seek two transition dynamics, a
forward and a backward process, that are exact time reversals of one another.

This mirrors the role of stochastic interpolants \citep{albergo2023stochastic} in continuous spaces, where a prescribed probability path can be associated with mutually
time-reversed forward and backward dynamics.
We proceed in two steps.
First, we establish a general discrete interpolant construction that
jointly defines forward and backward generators satisfying the
time-reversal relation.
We then show how the quantities defining these generators can be learned
directly from endpoint energy functions, extending the Gibbs-supervision principle developed above beyond masked diffusion.
We include detailed derivation in \Cref{app:general}.

\paragraph{Constructing the Discrete Interpolants.}  In summary, we first define an interpolant path between samples from the two endpoint distributions.
A simple choice  is to interpolate the endpoint samples $X_0\sim p_0$ and $X_1\sim p_1$ through the coordinate-factorized interpolant
\begin{equation}
    K_t(x\mid x_0,x_1)
    =\prod_d\kappa_t^d(x^d\mid x_0^d,x_1^d),
    \qquad
    p_t(x)=\mathbb E[K_t(x\mid X_0,X_1)].\label{eq:coord_interpolant}
\end{equation}
Then, we seek to learn a pair of forward and backward CTMCs that are (1) time-reversal and (2) shares the same marginal density $p_t$ along the time steps.
This leads to the following requirements, derived from Nelson's relation and the Kolmogorov equations.
Here, following the convention in discrete diffusion, we restrict the dynamics to single-site transitions:
\begin{align}
&\resizebox{0.92\linewidth}{!}{$\displaystyle
\textbf{Nelson's: }\quad
\bar r_t^d(a,c\mid x^{-d})
=
r_t^d(c,a\mid x^{-d})
\frac{p_t(x^d=c,x^{-d})}{p_t(x^d=a,x^{-d})}
=
r_t^d(c,a\mid x^{-d})
\frac{p_t(x^d=c\mid x^{-d})}{p_t(x^d=a\mid x^{-d})}
$};
\label{eq:nel}
\\
&\resizebox{0.92\linewidth}{!}{$\displaystyle
\textbf{Kolmogorov: }\quad
\partial_t p_t(x)
=
\sum_d \sum_{c\neq a}
r_t^d(c,a\mid x^{-d})\,
p_t(x^d=c,x^{-d})
-
r_t^d(a,c\mid x^{-d})\,
p_t(x^d=a,x^{-d}).
$}
\label{eq:kov}
\end{align}
\Cref{eq:nel} suggests a natural symmetric parameterization of the forward and backward generators: 
\begin{align}
    &r_t^d(c, a|x^{-d})  \!= \! m_t^d(c, a|x^{-d}) p_t(x^d = a| x^{-d}); \\& \bar r_t^d(a, c|x^{-d})  \!=\!  m_t^d(c, a|x^{-d}) p_t(x^d = c| x^{-d}),
\end{align}
and we only need to figure out the shared component $m_t^d$ satisfying \Cref{eq:kov}. 
Since our interpolant is coordinate-wise, the time derivative of marginal density can be decomposed into summation of local terms: $\partial_t p_t=\sum_d\partial_t^{(d)}p_t$ with 
$\partial_t^{(d)}p_t(x^{-d})=0$. 
Plugging in these symmetric construction and the local decomposition, a sufficient condition for the Kolmogorov equation to hold can be imposed coordinate-wise as
\begin{equation}
    \partial_t^{(d)}p_t(x^d=a\mid x^{-d})
    =
   \sum_{c\neq a} \underbrace{ p_t(x^d=a\mid x^{-d})
    p_t(x^d=c\mid x^{-d})
    \left[
        m_t^d(c,a\mid x^{-d})
        -
        m_t^d(a,c\mid x^{-d})
    \right]}_{\text{local flux } J^d_t(c, a |x^{-d})}.\label{eq:local_flux}
\end{equation}
Define the temporal potential:
    $\phi_t^d(a\mid x^{-d})
    =
    \partial_t^{(d)}
    \log p_t(x^d=a\mid x^{-d})$
and we can obtain from \Cref{eq:local_flux} that
$
    \phi_t^d(a\mid x^{-d})
    =
    \sum_c
    p_t(x^d=c\mid x^{-d})
    \left[
        m_t^d(c,a\mid x^{-d})
        -
        m_t^d(a,c\mid x^{-d})
    \right]$.
 This shows that the Kolmogorov equation constrains only the antisymmetric part of $m_t^d$.
 Recall that Nelson's relation is satisfied by construction for any choice of $m_t^d$ as well.
We therefore can have the freedom to take the simplest circulation-free choice,
\begin{equation}
m_t^d(c,a\mid x^{-d})
-
m_t^d(a,c\mid x^{-d})
=
\phi_t^d(a\mid x^{-d})
-
\phi_t^d(c\mid x^{-d}).\label{eq:m-difference}
\end{equation}

A non-negative  realization is $ 
    m_t^d(c,a\mid x^{-d})
    =
    \left[
        \phi_t^d(a\mid x^{-d})
        -
        \phi_t^d(c\mid x^{-d})
    \right]_+, $
which  yields
\begin{align}
    r_t^d(c,a\mid x^{-d})
    &=
    p_t(x^d=a\mid x^{-d})
    \left[
        \phi_t^d(a\mid x^{-d})
        -
        \phi_t^d(c\mid x^{-d})
    \right]_+,
    \\
    \bar r_t^d(a,c\mid x^{-d})
    &=
    p_t(x^d=c\mid x^{-d})
    \left[
        \phi_t^d(a\mid x^{-d})
        -
        \phi_t^d(c\mid x^{-d})
    \right]_+.
\end{align}
In other words, we only need to parameterize and learn $  p_t^d(\cdot \mid x^{-d})$ and $\phi_t^d( \cdot \mid x^{-d})$.

 \paragraph{Minimum kinetic energy local flux}
 This choice corresponds to an antisymmetric local flux:
 \begin{align}
      J_t^d(c, a|x^{-d}) = p_t(x^d=c|x^{-d}) \partial_t^{(d)} p_t(x^d=a|x^{-d}) - p_t(x^d=a|x^{-d}) \partial_t^{(d)} p_t(x^d=c|x^{-d})\label{eq:flux}
 \end{align}
Notably, this constructed flux uniquely minimizes the probability-weighted kinetic energy:
at $t$ with $p_t>0$, define the kinetic energy for a local flux $j$ as 
\begin{equation}
\mathcal E_t(j)=
\frac12\sum_d\sum_{x^{-d}}p_t(x^{-d})\sum_a \sum_{c \neq a}
\frac{j_t^d(a,c\mid x^{-d})^2}
     {p_t(x^d=a|x^{-d}) p_t(x^d=c|x^{-d}) }
\end{equation}
The unique minimizer that satisfies Kolmogorov is  given by \Cref{eq:flux}.
Similar construction was also derived and explored by \citet{shaul2025flow}.

\begin{table}[t]
\caption{\textbf{Ising and Potts lattice spin models.} Ising and Potts model at three different inverse temperatures $\beta$ corresponding to disordered, critical and ordered phases. We consider three evaluation criteria: The effective sample size ($\ESS$), the error in normalization constant estimation $|\Delta\ln Z|$ and the total variation distance between the empirical distribution over the $q$ modes of the reference target density, denoted as $\mTVD$. The table reports Mean${\pm}$standard
deviation over \NSeeds{} seeds. Blue shading is used to indicate methods whose pessimistic value lies within one standard deviation of the best method. Dashes indicate that the method does not support computing $\ln Z$ and $\ESS$.}
\label{tab:lattice}
\centering
\scriptsize
\setlength{\tabcolsep}{1.5pt}
\begin{tabular}{@{}l ccc@{\hspace{3pt}} ccc@{\hspace{3pt}} ccc@{}}
\toprule
& \multicolumn{3}{c}{disordered} & \multicolumn{3}{c}{critical} & \multicolumn{3}{c}{ordered} \\
\cmidrule(lr){2-4}\cmidrule(lr){5-7}\cmidrule(lr){8-10}
& $|\Delta\ln Z|$ & $\ESS$ & $\mTVD$ & $|\Delta\ln Z|$ & $\ESS$ & $\mTVD$ & $|\Delta\ln Z|$ & $\ESS$ & $\mTVD$ \\
Method & $\times10^{3}\,{\downarrow}$ & ${\uparrow}$ & ${\downarrow}$ & $\times10^{3}\,{\downarrow}$ & ${\uparrow}$ & ${\downarrow}$ & $\times10^{3}\,{\downarrow}$ & ${\uparrow}$ & ${\downarrow}$ \\
\midrule
\multicolumn{10}{@{}l}{\emph{Ising, $25\times25$, $q=2$}} \\
MDNS         & \IsingDisMDNSZ  & \IsingDisMDNSESS  & \IsingDisMDNSMTV  & \IsingCritMDNSZ  & \IsingCritMDNSESS  & \IsingCritMDNSMTV  & \IsingOrdMDNSZ  & \IsingOrdMDNSESS  & \IsingOrdMDNSMTV \\
PDNS         & \IsingDisPDNSZ  & \IsingDisPDNSESS  & \IsingDisPDNSMTV  & \IsingCritPDNSZ  & \IsingCritPDNSESS  & \IsingCritPDNSMTV  & \IsingOrdPDNSZ  & \IsingOrdPDNSESS  & \IsingOrdPDNSMTV \\
DASBS         & \IsingDisASBSZ  & \IsingDisASBSESS  & \IsingDisASBSMTV  & \IsingCritASBSZ  & \IsingCritASBSESS  & \IsingCritASBSMTV  & \IsingOrdASBSZ  & \IsingOrdASBSESS  & \IsingOrdASBSMTV \\
DASBS, memoryless     & \IsingDisASBSMZ & \IsingDisASBSMESS & \IsingDisASBSMMTV & \IsingCritASBSMZ & \IsingCritASBSMESS & \IsingCritASBSMMTV & \IsingOrdASBSMZ & \IsingOrdASBSMESS & \IsingOrdASBSMMTV \\
LEAPS        & \IsingDisLEAPSZ & \IsingDisLEAPSESS & \IsingDisLEAPSMTV & \IsingCritLEAPSZ & \IsingCritLEAPSESS & \IsingCritLEAPSMTV & \IsingOrdLEAPSZ & \IsingOrdLEAPSESS & \IsingOrdLEAPSMTV \\
\method{} (ours)   & \IsingDisDSIZ   & \IsingDisDSIESS   & \IsingDisDSIMTV   & \IsingCritDSIZ   & \IsingCritDSIESS   & \IsingCritDSIMTV   & \IsingOrdDSIZ   & \IsingOrdDSIESS   & \IsingOrdDSIMTV \\
\addlinespace
\multicolumn{10}{@{}l}{\emph{Potts, $16\times16$, $q=3$}} \\
MDNS         & \PottsDisMDNSZ  & \PottsDisMDNSESS  & \PottsDisMDNSMTV  & \PottsCritMDNSZ  & \PottsCritMDNSESS  & \PottsCritMDNSMTV  & \PottsOrdMDNSZ  & \PottsOrdMDNSESS  & \PottsOrdMDNSMTV \\
PDNS         & \PottsDisPDNSZ  & \PottsDisPDNSESS  & \PottsDisPDNSMTV  & \PottsCritPDNSZ  & \PottsCritPDNSESS  & \PottsCritPDNSMTV  & \PottsOrdPDNSZ  & \PottsOrdPDNSESS  & \PottsOrdPDNSMTV \\
DASBS         & \PottsDisASBSZ  & \PottsDisASBSESS  & \PottsDisASBSMTV  & \PottsCritASBSZ  & \PottsCritASBSESS  & \PottsCritASBSMTV  & \PottsOrdASBSZ  & \PottsOrdASBSESS  & \PottsOrdASBSMTV \\
DASBS, memoryless     & \PottsDisASBSMZ & \PottsDisASBSMESS & \PottsDisASBSMMTV & \PottsCritASBSMZ & \PottsCritASBSMESS & \PottsCritASBSMMTV & \PottsOrdASBSMZ & \PottsOrdASBSMESS & \PottsOrdASBSMMTV \\
LEAPS        & \PottsDisLEAPSZ & \PottsDisLEAPSESS & \PottsDisLEAPSMTV & \PottsCritLEAPSZ & \PottsCritLEAPSESS & \PottsCritLEAPSMTV & \PottsOrdLEAPSZ & \PottsOrdLEAPSESS & \PottsOrdLEAPSMTV \\
\method{} (ours)   & \PottsDisDSIZ   & \PottsDisDSIESS   & \PottsDisDSIMTV   & \PottsCritDSIZ   & \PottsCritDSIESS   & \PottsCritDSIMTV   & \PottsOrdDSIZ   & \PottsOrdDSIESS   & \PottsOrdDSIMTV \\
\bottomrule
\end{tabular}
\end{table}

\paragraph{Learning the discrete interpolants.}
Finally, we consider learning $p_t^d$ and $\phi_t^d$ from energy and
buffer samples.
In our setting, samples from $p_0$ are directly available or easy to
obtain, while $p_1$ is specified only through the energy $U_1$.
For a generated endpoint $X_1$, we replace its value at site $d$ by the
Gibbs soft label
$h_1^d(a\mid X_1^{-d})
    \propto
    \exp[-U_1(X_1^{d\leftarrow a})].
$ 
Under  independent coupling $p_0\otimes p_1$, define
\begin{align}
\ell_{t,d}(c)
&=
\sum_a h_1^d(a\mid X_1^{-d})
\kappa_t^d(c\mid X_0^d,a),
&
\dot\ell_{t,d}(c)
&=
\sum_a h_1^d(a\mid X_1^{-d})
\partial_t\kappa_t^d(c\mid X_0^d,a).
\end{align}
The two fields can then be learned with
\begin{align}
\resizebox{0.93\linewidth}{!}{$\displaystyle
\mathcal L_q(\theta)
=
-\mathbb E\sum_{d,c}\ell_{t,d}(c)\log p_{\theta,t}^d(c\mid X_t^{-d}),
\quad
\mathcal L_\phi(\psi)
=
\mathbb E\sum_{d,c}
\left[
\frac12\ell_{t,d}(c)\phi_{\psi,t}^d(c\mid X_t^{-d})^2
-\dot\ell_{t,d}(c)\phi_{\psi,t}^d(c\mid X_t^{-d})
\right]
$}
\end{align}
The intermediate state $X_t$ is sampled directly from
$K_t(\cdot\mid X_0,X_1)$, so the same endpoint pair and its energy-based
labels can be reused across many draws of $t$ and $X_t$.
As in the masked case, because $X_1$ is generated by the current sampler
rather than drawn from $p_1$, we weight each endpoint pair by the tempered
importance weight $\bar w$ with ESS control.
The two objectives above are therefore replaced in practice by their
weighted versions.
Detailed form and derivation can be found in  \Cref{app:learn}.

Since both $p_t^d(\cdot\mid x^{-d})$ and $\phi_t^d(\cdot\mid x^{-d})$ must not depend on the queried value $x^d$, we parameterize them with a
hollow (or blind-spot) architecture \citep{laine2019high,sun2022score,zhao2026informed}, whose architecture prevents the
$d$-th output from accessing $x^d$ while retaining the full context
$x^{-d}$.

\begin{figure}[t]
\centering
\begin{minipage}[c]{0.57\linewidth}
\centering
\tikzset{external/export next=false}
\begin{tikzpicture}
\begin{groupplot}[
    group style={
        group name=fpmaps,
        group size=2 by 1,
        horizontal sep=0.95cm
    },
    fpmap,
    point meta min=0,
    point meta max=\FPMTVRWScale
]
  \nextgroupplot[
      fpising,
      scale only axis,
      unit vector ratio={},
      width=2.05cm,
      height=1.55cm,
      title={Ising, $50\times50$},
      xlabel={field $h$},
      ylabel={$\beta$}
  ]
    \addplot[fpdata]
      table[x=h, y=beta, meta=v]
      {data/lattice/fp_ising_mtv_rw.dat};

  \nextgroupplot[
      fppotts,
      scale only axis,
      unit vector ratio={},
      width=2.05cm,
      height=1.55cm,
      title={Potts, $24\times24$},
      xlabel={field $h$},
      ylabel={}
  ]
    \addplot[fpdata]
      table[x=h, y=beta, meta=v]
      {data/lattice/fp_potts_mtv_rw.dat};
\end{groupplot}

\begin{axis}[
    at={(fpmaps c2r1.north east)},
    anchor=north west,
    xshift=0.2cm,
    hide axis,
    scale only axis,
    width=1pt,
    height=1.55cm,
    xmin=0,
    xmax=1,
    ymin=0,
    ymax=1,
    colormap/viridis,
    point meta min=0,
    point meta max=\FPMTVRWScale,
    colorbar,
    colorbar style={
        width=0.18cm,
        height=1.55cm,
        scaled y ticks=false,
        yticklabel style={
            /pgf/number format/.cd,
            fixed,
            precision=3
        },
        tick label style={font=\scriptsize},
        ylabel={$\mTVD$},
        ylabel style={font=\small}
    }
]
  \addplot[draw=none] coordinates {(0,0)};
\end{axis}
\end{tikzpicture}
\end{minipage}\hfill%
\begin{minipage}[c]{0.40\linewidth}
\centering
\tikzset{external/export next=false}
\begin{tikzpicture}
\begin{groupplot}[
    group style={
        group size=3 by 1,
        horizontal sep=0.40cm
    },
    cfgpanel,
    scale only axis,
    width=0.95cm,
    height=0.95cm,
    colormap name=potts,
    point meta min=0,
    point meta max=2,
    title style={
        font=\scriptsize,
        yshift=-0.9ex
    }
]
  \nextgroupplot[title={$\beta=\FPPottsCfgBetaDis$}]
    \addplot[
        matrix plot*,
        point meta=explicit,
        mesh/cols=24
    ]
      table[x=x, y=y, meta=v]
      {data/lattice/fp_potts_cfg_Dis.dat};

  \nextgroupplot[title={$\beta=\FPPottsCfgBetaCrit$}]
    \addplot[
        matrix plot*,
        point meta=explicit,
        mesh/cols=24
    ]
      table[x=x, y=y, meta=v]
      {data/lattice/fp_potts_cfg_Crit.dat};

  \nextgroupplot[title={$\beta=\FPPottsCfgBetaOrd$}]
    \addplot[
        matrix plot*,
        point meta=explicit,
        mesh/cols=24
    ]
      table[x=x, y=y, meta=v]
      {data/lattice/fp_potts_cfg_Ord.dat};
\end{groupplot}
\end{tikzpicture}

\par\vspace{0.1em}
\scriptsize
\setlength{\tabcolsep}{1.5pt}
\setlength{\aboverulesep}{0.3ex}
\setlength{\belowrulesep}{0.3ex}
\renewcommand{\arraystretch}{0.95}
\begin{tabular}{@{}lcc@{}}
\toprule
$\E_{(\beta,h)}[\cdot]$ & Ising & Potts \\
\midrule
${|\Delta\ln Z|}\downarrow$ ($10^{-3}$)
  & \FPIsingZ & \FPPottsZ \\
${\ESS}\uparrow$
  & \FPIsingESS & \FPPottsESS \\
${\mTVD}\downarrow$
  & \FPIsingMTVRW & \FPPottsMTVRW \\
\bottomrule
\end{tabular}
\end{minipage}

\caption{\textbf{Amortized training across $(\beta,h)$ of $\method{}$.}
\emph{Left:} a low $\mTVD$ across all $(\beta,h)$ indicates that our
method avoids mode collapse and accurately learns the relative mode
weights, which change under an external field $h$.
\emph{Right, top:} Potts samples of the amortized model at $h=0$ in
the disordered phase, near the critical point and in the ordered phase.
\emph{Right, bottom:} average $\ESS$, $|\Delta\ln Z|$ and $\mTVD$
over the $9\times9$ grid of $(\beta,h)$ values, all computed with the
importance weights. The results show mean $\pm$ standard deviation
over three seeds.}
\label{fig:fp-tv}
\end{figure}

 \vspace{-5pt}
\section{Related work}
\label{sec:related}
 \vspace{-5pt}
 \paragraph{Neural samplers on discrete spaces.} Discrete neural samplers include AR models
\citep{wu2019solving,nicoli2020asymptotically}, GFlowNets
\citep{bengio2021flow,malkin2022trajectory}, and  CTMCs including LEAPS, MDNS, PDNS, Discrete ASBS, JANUS
\citep{holderrieth2025leaps,zhu2025mdns,guo2026asbs,guo2026proximal,blessing2026janus}.
Ours combines Gibbs soft-label supervision with tempered
importance weighting, achieving a stable and scalable training and also mitigating mode collapsing.

\paragraph{Discrete interpolant and flow models. } 
Our method supports construction for general interpolants between two arbitrary distributions.
Discrete flow and generator matching also support learning dynamics between general distributions
\citep{gat2024discrete,holderrieth2025generator}.
However, they typically do not learn a pair of time-reversed forward and backward dynamics.
Recently, \citet{he2026free} develop a framework for learning mutually
time-reversed forward and backward dynamics.
However, they all consider data-based training, while our method learns the local conditional distribution and temporal potential from endpoint energies.

Appendix~\ref{app:related} provides a more detailed discussion on the related works.

 \vspace{-5pt}
\section{Experiments}\label{sec:experiments}
 \vspace{-5pt} 
We perform comprehensive experiments.
We compare with several neural sampler baselines on the Ising and Potts lattice spin models at different stages of the phase transition (\cref{sec:lattice}), showing that DIG-IN achieves the best performance.
Next, we test our method on larger lattices, where we train an amortized sampler that is conditioned on two thermodynamic parameters
(\cref{sec:amortized}). 
We then consider two alloys, where the
quantity of interest is a phase diagram (\cref{sec:alloy}). Moving to molecules, we
discretize the configuration space of two peptides and use \method{} as a
sampler trained from the energy alone (\cref{sec:bg}), and estimate free-energy differences
between two thermodynamic states with the bridge of \cref{sec:general_interpolants}
(\cref{sec:fe}). 
Lastly, we apply \method{} to
maximum entropy reinforcement learning, where the energy is a learned action value
(\cref{sec:rl}). We refer to \cref{app:setup} for details on the experimental setup and
to \cref{app:further} for further experiments, including a component ablation
and a bridge between two Boltzmann distributions.

\vspace{-5pt}
\subsection{Ising and Potts lattice spin models}\label{sec:lattice}
\vspace{-5pt}
We use the Ising model on a $25\times25$ lattice ($D=625$) and the 
Potts model on a $16\times16$ lattice ($D=256$).  Both place a $q$-state spin
$\sigma_i\in\{1,\dots,q\}$ ($q=2$ for Ising and $q=3$ for Potts)  on every site of a periodic square lattice and
favor equal neighbors through the energies, i.e., $E(\sigma)=-\sum_{\langle ij\rangle}\delta_{\sigma_i\sigma_j}$
for Potts and $E(\sigma)=-\sum_{\langle ij\rangle}s_is_j$ with spins
$s_i=\pm1$ for Ising with the convention
$\delta_{\sigma_i\sigma_j}=(1+s_is_j)/2$. Hence, the target is $p_\beta(\sigma)\propto
e^{-\beta E(\sigma)}$ at inverse temperature $\beta$.  Each model is sampled
at a {disordered}, a {critical} and an {ordered} $\beta$
(\cref{app:lattice-tasks}).  Above $\beta_c$ the target has $q$
symmetry-related modes.  We compare against MDNS \citep{zhu2025mdns}, PDNS \citep{guo2026proximal}, DASBS \citep{guo2026asbs} and LEAPS \citep{holderrieth2025leaps}.  All methods are trained for the same number of gradient steps and are evaluated identically.  Besides the effective sample size ($\ESS$) of the importance weights and the error $|\Delta\ln Z|$ of the normalization-constant estimate, we report the mode TVD ($\mTVD$).  Each draw is assigned to one of the $q$ modes and $\mTVD$ is the total variation distance between the resulting categorical distribution over the $q$ modes and that of the reference. A detailed overview of computational costs is given in \cref{tab:compute}.

Away from the ordered phase the interpolant-based samplers,
\method, MDNS and PDNS are all close to the exact $\ln Z$ on Ising and
Potts, with \method{} leading in every $|\Delta\ln Z|$ column except critical Potts, where MDNS and PDNS are marginally better.  Memoryless DASBS and LEAPS are one to three orders of magnitude
further off. However, the ordered phase, which is multimodal, separates the samplers, where all methods except for DIG-IN and LEAPS suffer from mode collapse as shown by $\mTVD$. The high $\ln Z$ error from LEAPS indicates that training is unstable leaving \method{} as the only approach that does not suffer from mode collapse while giving accurate $\ln Z$ estimates.

\begin{figure}[t]
\centering
\tikzset{external/export next=false}
\resizebox{\linewidth}{!}{%
\begin{tikzpicture}
\begin{groupplot}[
    group style={
        group name=alloyrow,
        group size=4 by 1,
        horizontal sep=0.50cm
    },
    pd,
    scale only axis,
    width=3.0cm,
    height=1.45cm,
    xlabel={},
    ylabel={},
    title style={font=\scriptsize}
]
  \nextgroupplot[
      title={Ag--Pd: \textcolor{cDSI}{\textbf{\method{}}}},
      ylabel={$T$ (K)}
  ]
    \addplot[refties] table {data/alloy/agpd125_ref_ties.dat};
    \addplot[ties=cDSI] table {data/alloy/agpd125_dsi_ties.dat};
    \addplot[bounds=cDSI] table {data/alloy/agpd125_dsi_bounds.dat};
    \addplot[refbounds] table {data/alloy/agpd125_ref_bounds.dat};

  \nextgroupplot[
      title={Ag--Pd: \textcolor{cFrOGS}{\textbf{FrOGS}}},
      ylabel={},
      ytick=\empty,
      extra y ticks={},
      minor y tick num=0,
      minor ytick={}
  ]
    \addplot[refties] table {data/alloy/agpd125_ref_ties.dat};
    \addplot[ties=cFrOGS] table {data/alloy/agpd125_frogs_ties.dat};
    \addplot[bounds=cFrOGS] table {data/alloy/agpd125_frogs_bounds.dat};
    \addplot[refbounds] table {data/alloy/agpd125_ref_bounds.dat};
    \panelmark{AgPd}{FrOGS}

  \nextgroupplot[
      title={Cu--Au: \textcolor{cDSI}{\textbf{\method{}}}},
      ylabel={},
      ytick=\empty,
      extra y ticks={},
      minor y tick num=0,
      minor ytick={}
  ]
    \addplot[refties] table {data/alloy/cuau128_ref_ties.dat};
    \addplot[ties=cDSI] table {data/alloy/cuau128_dsi_ties.dat};
    \addplot[bounds=cDSI] table {data/alloy/cuau128_dsi_bounds.dat};
    \addplot[refbounds] table {data/alloy/cuau128_ref_bounds.dat};

  \nextgroupplot[
      title={Cu--Au: \textcolor{cFrOGS}{\textbf{FrOGS}}},
      ylabel={},
      ytick=\empty,
      extra y ticks={},
      minor y tick num=0,
      minor ytick={}
  ]
    \addplot[refties] table {data/alloy/cuau128_ref_ties.dat};
    \addplot[ties=cFrOGS] table {data/alloy/cuau128_frogs_ties.dat};
    \addplot[bounds=cFrOGS] table {data/alloy/cuau128_frogs_bounds.dat};
    \addplot[refbounds] table {data/alloy/cuau128_ref_bounds.dat};
    \panelmark{CuAu}{FrOGS}
\end{groupplot}

\path
  (alloyrow c1r1.south west) --
  (alloyrow c4r1.south east)
  node[
      midway,
      anchor=north,
      yshift=-15pt,
      font=\footnotesize
  ] {composition $c$};

\end{tikzpicture}%
}

\par\vspace{0.3ex}
\begin{minipage}[c]{0.30\linewidth}
\centering
\tikzset{external/export next=false}
\begin{tikzpicture}[
    x={(1.0cm,0)},
    y={(0,1.0cm)},
    z={(0.42cm,0.28cm)},
    scale=0.70
]
  \begin{scope}
    \fcccell{\fccmixed}{\fccmixed}{\fccmixed}
    \node[font=\scriptsize, anchor=north, align=center]
      at (0.71,-0.24,0) {Ag--Pd};
  \end{scope}
  \begin{scope}[xshift=2.05cm]
    \fcccell{\fccAu}{\fccCu}{\fccCu}
    \node[font=\scriptsize, anchor=north, align=center]
      at (0.71,-0.24,0) {Cu$_3$Au};
  \end{scope}
  \begin{scope}[xshift=4.1cm]
    \fcccell{\fccCu}{\fccCu}{\fccAu}
    \node[font=\scriptsize, anchor=north, align=center]
      at (0.71,-0.24,0) {CuAu};
  \end{scope}
\end{tikzpicture}
\end{minipage}\hfill%
\begin{minipage}[c]{0.68\linewidth}
\centering
\scriptsize
\setlength{\tabcolsep}{2pt}
\renewcommand{\arraystretch}{0.95}
\resizebox{\linewidth}{!}{%
\begin{tabular}{@{}lcc@{\hspace{0.7em}}lcc@{}}
\toprule
\multicolumn{3}{c}{Ag--Pd, $D=125$}
& \multicolumn{3}{c}{Cu--Au, $D=128$} \\
\cmidrule(lr){1-3}\cmidrule(l){4-6}
$\E_{(T,\Delta\mu)}[\cdot]$
& FrOGS & \method{} (ours)
& $\E_{(T,\Delta\mu)}[\cdot]$
& FrOGS & \method{} (ours) \\
\midrule
${|\Delta c|}\downarrow$ \,$(\AgPdCErrUnit)$
& \AgPdFrOGSCErr
& \AgPdDSICErr
& ${|\Delta c|}\downarrow$ \,$(\CuAuCErrUnit)$
& \CuAuFrOGSCErr
& \CuAuDSICErr \\
${|\Delta \ln Z|}/D\downarrow$ \,$(\AgPdZErrUnit)$
& \AgPdFrOGSZErr
& \AgPdDSIZErr
& ${|\Delta \ln Z|}/D\downarrow$ \,$(\CuAuZErrUnit)$
& \CuAuFrOGSZErr
& \CuAuDSIZErr \\
${\ESS}\uparrow$
& \AgPdFrOGSESS
& \AgPdDSIESS
& ${\ESS}\uparrow$
& \CuAuFrOGSESS
& \CuAuDSIESS \\
\bottomrule
\end{tabular}%
}
\end{minipage}

\caption{\textbf{Alloy phase diagrams.}
\emph{Top:} Ag--Pd ($D=125$) and Cu--Au
($D=128$). Grey dashed lines show reference phase
boundaries; colored lines show model predictions.
\emph{Bottom left:} example structures. Ag--Pd is a random fcc alloy,
shown with mixed occupancy, and separates into Ag-rich and Pd-rich
fcc phases below its miscibility gap; Cu--Au orders into Cu$_3$Au,
CuAu and CuAu$_3$.
\emph{Bottom right:} composition error $|\Delta c|$,
$|\Delta\ln Z|/D$ and $\ESS$, averaged over the training
$(T,\DMu)$ conditions. Results are mean $\pm$ standard deviation
over three seeds.}
\label{fig:alloy-pd}
\end{figure}

\vspace{-5pt}
\subsection{Amortized training}\label{sec:amortized}
\vspace{-5pt}
We train an amortized version of \method{}
conditioned on $(\beta,h)$ for the Ising and Potts model, where $h$ is an external field. The external field results in a new target density given as $\log\tilde p(\sigma|\beta,h)=\beta(\sum_{\langle ij\rangle}s_is_j+h\sum_is_i)$ for Ising and $\log\tilde p(\sigma|\beta,h)=\beta(\sum_{\langle ij\rangle}\delta_{\sigma_i\sigma_j}+h\,N_0(\sigma))$ for Potts, where $N_0(\sigma)$ counts the sites in state~0. For the Ising model, we use a $50\times50$ lattice and therefore $D=2500$. We consider $\beta\in[\FPIsingBetaLo,\FPIsingBetaHi]$, and 
$|h|\le\FPIsingHmax$. For the Potts model we use a $24\times 24$ lattice ($D=576$) and $\beta\in[\FPPottsBetaLo,\FPPottsBetaHi]$, $|h|\le\FPPottsHmax$. In this setting, a single sampler has to  represent the disordered phase, criticality and the ordered phase. We evaluate \method{} on a  $9\times9$ grid of $(\beta,h)$.

When amortizing across multiple conditions and a phase transition in high dimensions, \method{} is able to achieve high ESS values and low $\ln Z$ error. Moreover, under an external field $h$, the $q$ different modes do not have equal relative weight, meaning that the neural sampler not only has to avoid mode collapse but also needs to be able to learn the relative weights of the modes. \cref{fig:fp-tv} shows that \method{} is indeed able to achieve both, mode coverage and correct relative mode weighting, as indicated by the low $\mTVD$ across all thermodynamic conditions $(\beta, h).$

\vspace{-5pt}
\subsection{Alloy phase diagrams}\label{sec:alloy}
\vspace{-5pt}
Mapping the phase diagram of an alloy requires sampling atomic configurations across a plane of chemical
conditions $(T,\DMu)$.  We train one amortized sampler per alloy, conditioned on $(T,\DMu)$, and
compare \method{} with FrOGS \citep{min2026frogs}, which places a conditioned
autoregressive prior \citep{du2026scaling} in front of learned LEAPS transport.  The systems are those of
\citet{damewood2022sampling} and \citet{min2026frogs}, i.e., Ag--Pd on a
$5{\times}5{\times}5$ fcc cell ($D=125$), with a single miscibility gap, and
Cu--Au on an $8{\times}4{\times}4$ cell ($D=128$), with three ordered phases.
The energy is a DFT-fitted cluster expansion and the target is semi-grand
canonical over $T\in[200,900]$\,K and $\DMu\in[-0.4,0.4]$\,eV.  Each method trains
one model per system for 60k steps on the same energy, grid and conditioning
distribution with \AlloyNSeeds{} seeds, and draws \AlloyNDraws{} samples at
each of the grid points. Phase
diagrams are computed using $\log Z(T,\DMu)$ estimates paired with the common-tangent
construction of \citet{min2026frogs}. For further details see \cref{app:alloy-tasks,app:phase-diagram,app:alloy-metrics}.

\method{} leads on every quantity measured, on both systems
(\cref{fig:alloy-pd}).  Its free energy is one to two orders of magnitude closer to the reference and its $\ESS$ is significantly higher.  Moreover, \method{} correctly recovers the phase boundaries without adding spurious ones on both systems. In contrast, FrOGS fails at finding all phase boundaries and adds several spurious ones.

\begin{figure}[t]
\centering

\tikzset{external/export next=false}
\resizebox{0.94\linewidth}{!}{%
\begin{tikzpicture}

\begin{groupplot}[
    group style={
        group name=bgrama,
        group size=6 by 1,
        horizontal sep=0.30cm
    },
    ramapanel,
    width=1.85cm,
    height=1.85cm,
    xlabel={$\phi$}
]

  \nextgroupplot[
      title={MD},
      ylabel={$\psi$}
  ]
    \rama{data/peptides/rama_aldp_md_r0}

  \nextgroupplot[
      title={\method{}\\as drawn},
      yticklabels={,,}
  ]
    \rama{data/peptides/rama_aldp_internal_raw_r0}

  \nextgroupplot[
      title={\method{}\\reweighted},
      yticklabels={,,}
  ]
    \rama{data/peptides/rama_aldp_internal_rw_r0}

  \nextgroupplot[
      title={MD},
      yticklabels={,,}
  ]
    \rama{data/peptides/rama_al4_md_r0}

  \nextgroupplot[
      title={\method{}\\as drawn},
      yticklabels={,,}
  ]
    \rama{data/peptides/rama_al4_internal_raw_r0}

  \nextgroupplot[
      title={\method{}\\reweighted},
      yticklabels={,,}
  ]
    \rama{data/peptides/rama_al4_internal_rw_r0}

\end{groupplot}

\node[
    anchor=south,
    font=\footnotesize
] at ($(bgrama c1r1.outer north)!0.5!(bgrama c3r1.outer north)$)
{ALDP};

\node[
    anchor=south,
    font=\footnotesize
] at ($(bgrama c4r1.outer north)!0.5!(bgrama c6r1.outer north)$)
{AL4, residue Ala\,2};

\ramacolorbar[]{bgrama c6r1.north east}{1.85cm}

\end{tikzpicture}%
}

\vspace{0.3em}

\footnotesize
\setlength{\tabcolsep}{4.5pt}
\setlength{\aboverulesep}{0.25ex}
\setlength{\belowrulesep}{0.25ex}
\renewcommand{\arraystretch}{0.93}

\resizebox{0.90\linewidth}{!}{%
\begin{tabular}{@{}lcccccc@{}}
\toprule
&
\multicolumn{2}{c}{$\ESS\uparrow$}
&
\multicolumn{2}{c}{energy $\mathcal W_1\downarrow$ (kJ/mol)}
&
\multicolumn{2}{c}{Ramachandran JS$\downarrow$ ($10^{-3}$)}
\\
\cmidrule(lr){2-3}
\cmidrule(lr){4-5}
\cmidrule(l){6-7}

& MLE on MD & \method{}
& MLE on MD & \method{}
& MLE on MD & \method{}
\\

\midrule

ALDP, $D=60$
& \PepDiIntMLECmpESS
& \PepDiIntCmpESS
& \PepDiIntMLECmpWOne
& \PepDiIntCmpWOne
& \PepDiIntMLECmpJS
& \PepDiIntCmpJS
\\

AL4, $D=120$
& \PepTetraIntMLECmpESS
& \PepTetraIntCmpESS
& \PepTetraIntMLECmpWOne
& \PepTetraIntCmpWOne
& \PepTetraIntMLECmpJS
& \PepTetraIntCmpJS
\\

\bottomrule
\end{tabular}%
}

\caption{\textbf{Discretized sampler trained from the energy alone},
in internal coordinates.
\emph{Top:} Ramachandran surfaces $\ln p(\phi,\psi)$ of the MD reference
and of \method{}'s samples as drawn and reweighted, for ALDP and for
residue Ala\,2 of AL4 (Ala\,3 and Ala\,4 in \cref{fig:bg-rama-all});
the three seeds are pooled.
\emph{Bottom:} effective sample size, energy $\mathcal W_1$ and
Ramachandran JS divergence against MD, computed with the importance
weights on \PepDiNSamples{} samples per run, mean $\pm$ standard
deviation over \PepNSeeds{} seeds, for \method{} and the same network
trained by maximum likelihood on MD frames
(MLE on MD, \cref{app:bg-mle}). The better of the two is highlighted
in blue.}

\label{fig:bg}
\end{figure}

\vspace{-5pt}
\subsection{Biomolecular Ensemble Sampling}\label{sec:bg}
\vspace{-5pt}

In this section, we apply our method to sample from a discretized Boltzmann distribution of molecules.
We discretize
the configuration space, in internal coordinates, and use \method{} as a Biomolecular ensemble generator
\emph{trained from the energy alone, without any MD data}.  
Each of the
$D$ coordinates of the molecule is confined to an interval that is split into $L$ uniform
bins. 
A draw of \method{} picks one bin per coordinate and a uniform point inside each bin
then gives a continuous configuration $z$ whose density is the bin probability divided by the
cell volume, so the importance weights $w(z)=e^{-U(z)}\,\mathrm{vol}/q_\theta(b)$ correct the
draws to the continuous Boltzmann density (\cref{eq:iw_mask}).  
The Gibbs label of coordinate
$d$ is the softmax of $-U$ over its $L$ bin centres with all other coordinates held fixed, one
batch of $DL$ energy evaluations per configuration.

We consider alanine dipeptide (ALDP, 22 atoms, $D=60$) and alanine tetrapeptide (AL4, 42
atoms, $D=120$) in implicit solvent at 300\,K \citep{lindorff2010improved,onufriev2004exploring}.
The coordinates are the bond lengths, angles, and dihedrals, each split into $L=64$
bins (\cref{app:bg}).  
We use long parallel-tempering MD as reference, only for evaluation.  
Every run draws \PepDiNSamples{} samples, and we report the $\ESS$, the Wasserstein-1 distance $\mathcal W_1$ between the
energy distributions and the Jensen--Shannon (JS) divergence between the Ramachandran histograms
of the backbone dihedrals $(\phi,\psi)$ of every residue \citep{ramachandran1963stereochemistry}.
As a baseline that learns from data instead, we train the same network by maximum likelihood on
MD frames (MLE on MD), every tenth frame of the reference, with the same masking, optimizer and
number of steps (\cref{app:bg-mle}).
\method{} matches the MD reference on both peptides (\cref{fig:bg}).  On ALDP, its draws are
close to MD already as drawn, with about 50\% ESS, and the reweighted Ramachandran plot is indistinguishable from the reference.  On AL4, all three residues are recovered
(\cref{fig:bg-rama-all} shows Ala\,3 and Ala\,4), and the reweighted basin fractions of
every residue match MD to within $0.002$ (\cref{tab:bg-basins}). 
Notably, even without any MD data,
\method{} reaches a higher $\ESS$ than MLE on MD.

\Cref{app:bg-bondvec} further evaluate DIG-IN on Cartesian bond vectors instead of internal coordinates, the vector from each atom to its parent in a spanning tree of the molecule.  
There, \method{} still samples ALDP with an $\ESS$ of \PepDiCartESS{}, while $\ESS$ on AL4 is still not good enough. 
In spite of this, \method{} still perform better than MLE on MD, whose $\ESS$  falls below $10^{-3}$ on both peptides.

\begin{figure}[t]
\centering
\tikzexternaldisable
\setlength{\abovecaptionskip}{3pt}
\setlength{\belowcaptionskip}{0pt}

\begin{minipage}[t]{0.49\linewidth}
\vspace{0pt}
\centering

\begin{tikzpicture}
\begin{groupplot}[
    group style={
        group size=2 by 1,
        horizontal sep=0.90cm
    },
    convpanel,
    scale only axis,
    unit vector ratio={},
    width=2.0cm,
    height=1.30cm,
    title style={font=\scriptsize},
    label style={font=\scriptsize},
    tick label style={font=\scriptsize}
]
  \nextgroupplot[
      title={ALDP-S},
      xlabel={paths $M$},
      ylabel={$\ln Z_1/Z_0$},
      ymin=\FESConvYmin,
      ymax=\FESConvYmax,
      xmin=100,
      xmax=3e5
  ]
    \refband{\FESRefPlain}{\FESRefErr}{black}{solid}
    \refband{\FESFEATPlain}{\FESFEATErr}{cMDNS}{densely dashed}
    \addplot[convseed]
      table[x=M, y=s0] {data/freeenergy/conv_aldps.dat};
    \addplot[convseed]
      table[x=M, y=s1] {data/freeenergy/conv_aldps.dat};
    \addplot[convseed]
      table[x=M, y=s2] {data/freeenergy/conv_aldps.dat};

  \nextgroupplot[
      title={ALDP-T},
      xlabel={paths $M$},
      ymin=\FETConvYmin,
      ymax=\FETConvYmax,
      xmin=100,
      xmax=3e5
  ]
    \refband{\FETRefPlain}{\FETRefErr}{black}{solid}
    \refband{\FETFEATPlain}{\FETFEATErr}{cMDNS}{densely dashed}
    \addplot[convseed]
      table[x=M, y=s0] {data/freeenergy/conv_aldpt.dat};
    \addplot[convseed]
      table[x=M, y=s1] {data/freeenergy/conv_aldpt.dat};
    \addplot[convseed]
      table[x=M, y=s2] {data/freeenergy/conv_aldpt.dat};
\end{groupplot}
\end{tikzpicture}

\par\vspace{0.2em}
\begingroup
\scriptsize
\setlength{\tabcolsep}{2pt}
\setlength{\aboverulesep}{0.25ex}
\setlength{\belowrulesep}{0.25ex}
\renewcommand{\arraystretch}{0.95}
\resizebox{\linewidth}{!}{%
\begin{tabular}{@{}lcccc@{}}
\toprule
& \multicolumn{3}{c}{$\ln Z_1/Z_0$}
& $\ESS\uparrow$ \\
\cmidrule(lr){2-4}\cmidrule(l){5-5}
& \rlkey{black}\;reference
& \dashkey{cMDNS}\;FEAT
& \rlkey{cDSI}\;\method{}
& \method{} \\
\midrule
ALDP-S
& \FESRef & \FESFEAT & \FESEst & \FESESS \\
ALDP-T
& \FETRef & \FETFEAT & \FETEst & \FETESS \\
\bottomrule
\end{tabular}%
}
\par
\endgroup

\caption{\textbf{Free-energy differences with \method{}.}
ALDP-S: vacuum to implicit solvent; ALDP-T: two conformational
states in vacuum. Blue curves use the first $M$ paths of each seed;
black bands show the reference $\pm$ one standard error;
orange dashed lines show FEAT.
References use MBAR on MD of both states for ALDP-S and umbrella
sampling for ALDP-T. FEAT uses samples of both states, as reported
by \citet{he2025feat}. Mean $\pm$ standard deviation over \FENSeeds{} seeds.}
\label{fig:fe}

\end{minipage}\hfill%
\begin{minipage}[t]{0.49\linewidth}
\vspace{0pt}
\centering

\begin{tikzpicture}
\begin{axis}[
    rlpanel,
    scale only axis,
    unit vector ratio={},
    width=\dimexpr\linewidth-3.95cm\relax,
    height=1.40cm,
    xlabel={environment steps ($\times10^6$)},
    ylabel={IQM$/1000$},
    label style={font=\scriptsize},
    tick label style={font=\scriptsize}
]
  \addplot[rlband=cDPPO]
    table[x=step, y=y] {data/rl/band_DPPO.dat};
  \addplot[rlband=cPPO]
    table[x=step, y=y] {data/rl/band_PPO.dat};
  \addplot[rlband=cSPO]
    table[x=step, y=y] {data/rl/band_SPO.dat};
  \addplot[rlband=cFPO]
    table[x=step, y=y] {data/rl/band_FPO.dat};
  \addplot[rlband=cREPPO]
    table[x=step, y=y] {data/rl/band_REPPO.dat};
  \addplot[rlband=cAMDP]
    table[x=step, y=y] {data/rl/band_AMDP.dat};
  \addplot[rlband=cFFM]
    table[x=step, y=y] {data/rl/band_FFM.dat};
  \addplot[rlband=cTRUDI]
    table[x=step, y=y] {data/rl/band_TRUDI.dat};
  \addplot[rlband=cDSI]
    table[x=step, y=y] {data/rl/band_DIGIN.dat};

  \addplot[rlline=cDPPO]
    table[x=step, y=iqm] {data/rl/agg_DPPO.dat};
  \addplot[rlline=cPPO]
    table[x=step, y=iqm] {data/rl/agg_PPO.dat};
  \addplot[rlline=cSPO]
    table[x=step, y=iqm] {data/rl/agg_SPO.dat};
  \addplot[rlline=cFPO]
    table[x=step, y=iqm] {data/rl/agg_FPO.dat};
  \addplot[rlline=cREPPO]
    table[x=step, y=iqm] {data/rl/agg_REPPO.dat};
  \addplot[rlline=cAMDP]
    table[x=step, y=iqm] {data/rl/agg_AMDP.dat};
  \addplot[rlline=cFFM]
    table[x=step, y=iqm] {data/rl/agg_FFM.dat};
  \addplot[rlline=cTRUDI]
    table[x=step, y=iqm] {data/rl/agg_TRUDI.dat};
  \addplot[rlline=cDSI, line width=1.3pt]
    table[x=step, y=iqm] {data/rl/agg_DIGIN.dat};
\end{axis}
\end{tikzpicture}

\par\vspace{0.2em}
\begingroup
\scriptsize
\setlength{\tabcolsep}{2pt}
\setlength{\aboverulesep}{0.25ex}
\setlength{\belowrulesep}{0.25ex}
\renewcommand{\arraystretch}{0.90}
\begin{tabular*}{\linewidth}{@{\extracolsep{\fill}}llc}
\toprule
method & policy & final IQM $\uparrow$ \\
\midrule
\rlkey{cDSI}\;\textcolor{cDSI}{\textbf{\method{} (ours)}}
  & binned & \RLDIGINCell \\
\rlkey{cTRUDI}\;TRUDI
  & diffusion & \RLTRUDICell \\
\rlkey{cAMDP}\;AMDP
  & diffusion & \RLAMDPCell \\
\rlkey{cREPPO}\;REPPO
  & Gaussian & \RLREPPOCell \\
\rlkey{cFFM}\;FFM
  & flow & \RLFFMCell \\
\rlkey{cFPO}\;FPO
  & flow & \RLFPOCell \\
\rlkey{cSPO}\;SPO
  & Gaussian & \RLSPOCell \\
\rlkey{cPPO}\;PPO
  & Gaussian & \RLPPOCell \\
\rlkey{cDPPO}\;DPPO
  & diffusion & \RLDPPOCell \\
\bottomrule
\end{tabular*}
\par
\endgroup

\caption{\textbf{Maximum entropy RL on DMC.}
Curves show the interquartile mean (IQM) of return over
\RLNTasks{} continuous control tasks and \RLNSeeds{} seeds per task,
with $95\%$ stratified bootstrap intervals.
The table reports final IQM and its $95\%$ interval.
Baseline curves and intervals are from \citet{trudi2025}.
The best result is highlighted in blue.}
\label{fig:rl}

\end{minipage}
\end{figure}

\vspace{-5pt}
\subsection{Free energy estimation}\label{sec:fe}
\vspace{-5pt}

In this section, we evaluate our generalized formulation in \Cref{sec:general_interpolants}.
Since our generalized formulation in \Cref{sec:general_interpolants} learns a forward and backward process as time-reversal in optimality, we can apply the (generalized) FEAT  \citep{he2025feat,he2026free} on our learned sampler to estimate the free energy difference between two systems.

We use the two alanine dipeptide tasks of FEAT \citep{he2025feat}.  ALDP-S is the solvation free
energy, from vacuum to implicit solvent.  ALDP-T is the free-energy difference between two
conformational states in vacuum, separated by the barrier along $\phi$: $S_b$ (C7eq and C5) and
$S_a$ (C7ax).  Endpoint 0 is given by MD samples, and endpoint 1 only by its energy.  We train
the bridge of \cref{sec:general_interpolants} on the internal coordinates of ALDP.  
From \cref{fig:fe}, our method achieves accurate free energy values even though it is trained without both sides' data.

\vspace{-5pt}
\subsection{Maximum entropy reinforcement learning}\label{sec:rl}
\vspace{-5pt}
In maximum entropy reinforcement learning (RL), the improvement step of soft policy
iteration sets the policy at every state $s$ to the Boltzmann distribution
$\pi(a\mid s)\propto\exp(Q(s,a)/\tau)$ of the soft action value $Q$ at temperature $\tau$
\citep{haarnoja2018sac}.  Learning the policy is thus an amortized sampling problem over
states, whose energy $-Q(s,\cdot)/\tau$ is only available through a learned critic.  The
actions are continuous, $a\in[-1,1]^m$.  We discretize each of their $m$ dimensions into
$L=64$ bins, so that an action is a configuration of $D=m$ sites over $N=L$ bins, and use
the masked denoiser, conditioned on $s$, as the policy.
In an on-policy actor--critic loop, a critic $Q_\phi$ is fitted to the policy's own
transitions by temporal-difference learning, and the Gibbs labels are read from the critic
at the bin centres $c$, with the other dimensions held at the stored action, i.e.,
$h^d(c\mid s,a^{-d})\propto\exp(Q_\phi(s,a^{d\leftarrow c})/\tau)$.  The
temperature $\tau$ is adapted to a target entropy, and a trust region bounds each policy
update.  For further details see \cref{app:rl}.
We use the \RLNTasks{} continuous control tasks of the DeepMind Control (DMC) suite in
MuJoCo Playground \citep{mujoco_playground_2025}, with 1 to 6 action dimensions, in the
on-policy setting of \citet{ffm2026} and with \RLNSeeds{} seeds per task.  We compare against Gaussian policies (REPPO \citep{reppo2025},
PPO \citep{schulman2017ppo}, SPO \citep{xie2025simple}), diffusion policies (TRUDI
\citep{celik2025dime,trudi2025}, AMDP \citep{amdp2025}, DPPO \citep{ren2025dppo}) and flow policies (FPO
\citep{mcallister2026flow}, FFM \citep{ffm2026}). As the performance measure we follow common practice and report the
interquartile mean (IQM) of the return \citep{agarwal2021deep}.  \method{} attains the highest final IQM of these
methods.

\vspace{-5pt}

\section{Conclusion}
\label{sec:conclusion}
\vspace{-5pt}
We introduced DIG-IN, a fixed-point neural sampler for discrete, unnormalized distributions. DIG-IN uses Gibbs labels which turn local energy differences into a soft target for the masked denoiser that has the same optimum as the one-hot objective but lower variance. 
DIG-IN trains on self-generated configurations and corrects the resulting bias with importance weights tempered under ESS control. 
The same Gibbs supervision carries over from masked diffusion to general discrete interpolants.

\paragraph{Limitations and future work.} Despite lowering or matching the computational costs compared to most baselines, Gibbs labels require computing single-site energy differences. For local
energies such as lattice models and cluster expansions, this table costs about
one energy evaluation. For a black-box energy, however, it requires
$D(N-1)+1$ evaluations per configuration. This motivates cheaper label estimators, for instance by subsampling sites. 
While we establish a convergence guarantee for a certain setting of our algorithm, it does not cover the practical setting in our experiments all.
 Moreover, the benefits and extension of general interpolants are worth further studying. The broader applications of discrete samplers in biology and chemistry are left as future work.

\section*{Acknowledgments}
JH acknowledges support from the University of Cambridge Harding Distinguished Postgraduate
Scholars Programme.  DB acknowledges support by funding from a Google PhD fellowship in Machine Learning and ML Foundations.

\section*{AI use statement}

LLMs were used to draft parts of the paper from the authors' notes and polish the manuscript for grammar and consistency; help develop preliminary proofs; implement parts of the experiment code and the figure and table scripts; help with literature search; and review the draft for errors before submission. The research question, method, and experimental design originate from the authors.

\bibliographystyle{iclr2027_conference}
\bibliography{references}

\begin{thebibliography}{118}
\providecommand{\natexlab}[1]{#1}
\providecommand{\url}[1]{\texttt{#1}}
\expandafter\ifx\csname urlstyle\endcsname\relax
  \providecommand{\doi}[1]{doi: #1}\else
  \providecommand{\doi}{doi: \begingroup \urlstyle{rm}\Url}\fi

\bibitem[Agarwal et~al.(2021)Agarwal, Schwarzer, Castro, Courville, and Bellemare]{agarwal2021deep}
Rishabh Agarwal, Max Schwarzer, Pablo~Samuel Castro, Aaron~C. Courville, and Marc~G. Bellemare.
\newblock Deep reinforcement learning at the edge of the statistical precipice.
\newblock In \emph{Advances in Neural Information Processing Systems}, 2021.

\bibitem[Albergo et~al.(2025)Albergo, Boffi, and Vanden-Eijnden]{albergo2023stochastic}
Michael~S. Albergo, Nicholas~M. Boffi, and Eric Vanden-Eijnden.
\newblock Stochastic interpolants: A unifying framework for flows and diffusions.
\newblock \emph{Journal of Machine Learning Research}, 26\penalty0 (209):\penalty0 1--80, 2025.

\bibitem[Arriola et~al.(2025)Arriola, Gokaslan, Chiu, Yang, Qi, Han, Sahoo, and Kuleshov]{arriola2025block}
Marianne Arriola, Aaron Gokaslan, Justin~T. Chiu, Zhihan Yang, Zhixuan Qi, Jiaqi Han, Subham~Sekhar Sahoo, and Volodymyr Kuleshov.
\newblock Block diffusion: Interpolating between autoregressive and diffusion language models.
\newblock In \emph{International Conference on Learning Representations}, 2025.

\bibitem[Austin et~al.(2021)Austin, Johnson, Ho, Tarlow, and van~den Berg]{austin2021structured}
Jacob Austin, Daniel~D. Johnson, Jonathan Ho, Daniel Tarlow, and Rianne van~den Berg.
\newblock Structured denoising diffusion models in discrete state-spaces.
\newblock In \emph{Advances in Neural Information Processing Systems}, 2021.

\bibitem[Bengio et~al.(2021)Bengio, Jain, Korablyov, Precup, and Bengio]{bengio2021flow}
Emmanuel Bengio, Moksh Jain, Maksym Korablyov, Doina Precup, and Yoshua Bengio.
\newblock Flow network based generative models for non-iterative diverse candidate generation.
\newblock In \emph{Advances in Neural Information Processing Systems}, 2021.

\bibitem[Bennett(1976)]{bennett1976efficient}
Charles~H. Bennett.
\newblock Efficient estimation of free energy differences from {M}onte {C}arlo data.
\newblock \emph{Journal of Computational Physics}, 22:\penalty0 245--268, 1976.

\bibitem[Bia{\l}as et~al.(2022)Bia{\l}as, Korcyl, and Stebel]{bialas2022hierarchical}
Piotr Bia{\l}as, Piotr Korcyl, and Tomasz Stebel.
\newblock Hierarchical autoregressive neural networks for statistical systems.
\newblock \emph{Computer Physics Communications}, 281:\penalty0 108502, 2022.

\bibitem[Bia{\l}as et~al.(2023)Bia{\l}as, Korcyl, and Stebel]{bialas2023analysis}
Piotr Bia{\l}as, Piotr Korcyl, and Tomasz Stebel.
\newblock Analysis of autocorrelation times in neural {Markov} chain {Monte Carlo} simulations.
\newblock \emph{Physical Review E}, 107\penalty0 (1):\penalty0 015303, 2023.

\bibitem[Bia{\l}as et~al.(2026)Bia{\l}as, Korcyl, Stebel, Stefa{\'n}ski, and Zapolski]{bialas2026sampling}
Piotr Bia{\l}as, Piotr Korcyl, Tomasz Stebel, Adam Stefa{\'n}ski, and Dawid Zapolski.
\newblock Sampling two-dimensional spin systems with transformers.
\newblock \emph{arXiv preprint arXiv:2604.27738}, 2026.

\bibitem[Biazzo(2023)]{biazzo2023autoregressive}
Indaco Biazzo.
\newblock The autoregressive neural network architecture of the {Boltzmann} distribution of pairwise interacting spins systems.
\newblock \emph{Communications Physics}, 6:\penalty0 296, 2023.

\bibitem[Birge(1932)]{birge1932calculation}
Raymond~T. Birge.
\newblock The calculation of errors by the method of least squares.
\newblock \emph{Physical Review}, 40:\penalty0 207--227, 1932.

\bibitem[Blessing et~al.(2026{\natexlab{a}})Blessing, Berner, Richter, Domingo~i Enrich, Du, Vahdat, and Neumann]{blessing2026trust}
Denis Blessing, Julius Berner, Lorenz Richter, Carles Domingo~i Enrich, Yuanqi Du, Arash Vahdat, and Gerhard Neumann.
\newblock Trust region constrained measure transport in path space for stochastic optimal control and inference.
\newblock \emph{Advances in Neural Information Processing Systems}, 38:\penalty0 165462--165510, 2026{\natexlab{a}}.

\bibitem[Blessing et~al.(2026{\natexlab{b}})Blessing, Celik, and Neumann]{ffm2026}
Denis Blessing, Onur Celik, and Gerhard Neumann.
\newblock Reinforcement learning via fixed-point flow matching, 2026{\natexlab{b}}.
\newblock Under review.

\bibitem[Blessing et~al.(2026{\natexlab{c}})Blessing, Cheng, Schebek, Rogal, Li, Domingo-Enrich, and Du]{blessing2026janus}
Denis Blessing, Mouyang Cheng, Maximilian Schebek, Jutta Rogal, Mingda Li, Carles Domingo-Enrich, and Yuanqi Du.
\newblock {JANUS}: A multi-modal foundation neural sampler for disordered materials.
\newblock \emph{arXiv preprint arXiv:2608.19116}, 2026{\natexlab{c}}.

\bibitem[Blessing et~al.(2026{\natexlab{d}})Blessing, Richter, Berner, Malitskiy, and Neumann]{blessing2026bridge}
Denis Blessing, Lorenz Richter, Julius Berner, Egor Malitskiy, and Gerhard Neumann.
\newblock Bridge matching sampler: Scalable sampling via generalized fixed-point diffusion matching.
\newblock \emph{arXiv preprint arXiv:2603.00530}, 2026{\natexlab{d}}.
\newblock URL \url{https://arxiv.org/abs/2603.00530}.

\bibitem[Bond-Taylor et~al.(2022)Bond-Taylor, Hessey, Sasaki, Breckon, and Willcocks]{bondtaylor2022unleashing}
Sam Bond-Taylor, Peter Hessey, Hiroshi Sasaki, Toby~P. Breckon, and Chris~G. Willcocks.
\newblock Unleashing transformers: Parallel token prediction with discrete absorbing diffusion for fast high-resolution image generation from vector-quantized codes.
\newblock In \emph{European Conference on Computer Vision}, 2022.

\bibitem[Brooks et~al.(2011)Brooks, Gelman, Jones, and Meng]{brooks2011handbook}
Steve Brooks, Andrew Gelman, Galin Jones, and Xiao-Li Meng.
\newblock \emph{Handbook of markov chain monte carlo}.
\newblock CRC press, 2011.

\bibitem[Campbell et~al.(2022)Campbell, Benton, De~Bortoli, Rainforth, Deligiannidis, and Doucet]{campbell2022continuous}
Andrew Campbell, Joe Benton, Valentin De~Bortoli, Tom Rainforth, George Deligiannidis, and Arnaud Doucet.
\newblock A continuous time framework for discrete denoising models.
\newblock In \emph{Advances in Neural Information Processing Systems}, 2022.

\bibitem[Campbell et~al.(2024)Campbell, Yim, Barzilay, Rainforth, and Jaakkola]{campbell2024generative}
Andrew Campbell, Jason Yim, Regina Barzilay, Tom Rainforth, and Tommi Jaakkola.
\newblock Generative flows on discrete state-spaces: Enabling multimodal flows with applications to protein co-design.
\newblock In \emph{International Conference on Machine Learning}, 2024.

\bibitem[Carter et~al.(2026)Carter, Choi, Tamogashev, Elvira, and Whitammer]{carter2026discrete}
Arran Carter, Sanghyeok Choi, Kirill Tamogashev, V{\'\i}ctor Elvira, and Esmeralda~S Whitammer.
\newblock Discrete diffusion samplers and bridges: Off-policy algorithms and applications in latent spaces.
\newblock \emph{arXiv preprint arXiv:2602.05961}, 2026.

\bibitem[Celik et~al.(2025)Celik, Li, Blessing, Li, Palenicek, Peters, Chalvatzaki, and Neumann]{celik2025dime}
Onur Celik, Zechu Li, Denis Blessing, Ge~Li, Daniel Palenicek, Jan Peters, Georgia Chalvatzaki, and Gerhard Neumann.
\newblock {DIME}: Diffusion-based maximum entropy reinforcement learning.
\newblock In \emph{International Conference on Machine Learning}, 2025.

\bibitem[Chang et~al.(2022)Chang, Zhang, Jiang, Liu, and Freeman]{chang2022maskgit}
Huiwen Chang, Han Zhang, Lu~Jiang, Ce~Liu, and William~T. Freeman.
\newblock {MaskGIT}: Masked generative image transformer.
\newblock In \emph{IEEE/CVF Conference on Computer Vision and Pattern Recognition}, 2022.

\bibitem[Chen et~al.(2025)Chen, Ou, and Li]{chen2025neuralflow}
Wuhao Chen, Zijing Ou, and Yingzhen Li.
\newblock Neural flow samplers with shortcut models.
\newblock \emph{arXiv preprint arXiv:2502.07337}, 2025.
\newblock URL \url{https://arxiv.org/abs/2502.07337}.

\bibitem[Cheng et~al.(2026)Cheng, Blessing, Yu, Neumann, Li, Domingo-Enrich, and Du]{cheng2026atlas}
Mouyang Cheng, Denis Blessing, Botao Yu, Gerhard Neumann, Mingda Li, Carles Domingo-Enrich, and Yuanqi Du.
\newblock Atlas: A foundation neural sampler for amorphous materials.
\newblock \emph{arXiv preprint arXiv:2607.19198}, 2026.

\bibitem[Damewood et~al.(2022)Damewood, Schwalbe-Koda, and G{\'o}mez-Bombarelli]{damewood2022sampling}
James Damewood, Daniel Schwalbe-Koda, and Rafael G{\'o}mez-Bombarelli.
\newblock Sampling lattices in semi-grand canonical ensemble with autoregressive machine learning.
\newblock \emph{npj Computational Materials}, 8\penalty0 (1):\penalty0 61, 2022.

\bibitem[Del~Bono et~al.(2025)Del~Bono, Ricci-Tersenghi, and Zamponi]{delbono2025nearest}
Luca~Maria Del~Bono, Federico Ricci-Tersenghi, and Francesco Zamponi.
\newblock Nearest-neighbours neural network architecture for efficient sampling of statistical physics models.
\newblock \emph{Machine Learning: Science and Technology}, 6\penalty0 (2):\penalty0 025029, 2025.

\bibitem[Du et~al.(2026{\natexlab{a}})Du, Nam, Choi, Guo, Edamadaka, Sha, Pan, Chen, Tao, and G{\'o}mez-Bombarelli]{du2026metadns}
Xiaochen Du, Juno Nam, Jaemoo Choi, Wei Guo, Sathya Edamadaka, Junyi Sha, Elton Pan, Yongxin Chen, Molei Tao, and Rafael G{\'o}mez-Bombarelli.
\newblock {MetaDNS}: Enhancing exploration in discrete neural samplers via well-tempered metadynamics.
\newblock \emph{arXiv preprint arXiv:2605.21722}, 2026{\natexlab{a}}.

\bibitem[Du et~al.(2026{\natexlab{b}})Du, Nam, Liu, and G{\'o}mez-Bombarelli]{du2026scaling}
Xiaochen Du, Juno Nam, Sulin Liu, and Rafael G{\'o}mez-Bombarelli.
\newblock Scaling autoregressive models for lattice thermodynamics.
\newblock \emph{arXiv preprint arXiv:2603.14695}, 2026{\natexlab{b}}.

\bibitem[Eastman et~al.(2017)Eastman, Swails, Chodera, McGibbon, Zhao, Beauchamp, Wang, Simmonett, Harrigan, Stern, Wiewiora, Brooks, and Pande]{eastman2017openmm}
Peter Eastman, Jason Swails, John~D. Chodera, Robert~T. McGibbon, Yutong Zhao, Kyle~A. Beauchamp, Lee-Ping Wang, Andrew~C. Simmonett, Matthew~P. Harrigan, Chaya~D. Stern, Rafal~P. Wiewiora, Bernard~R. Brooks, and Vijay~S. Pande.
\newblock {OpenMM} 7: Rapid development of high performance algorithms for molecular dynamics.
\newblock \emph{PLoS Computational Biology}, 13\penalty0 (7):\penalty0 e1005659, 2017.

\bibitem[Eastman et~al.(2024)Eastman, Galvelis, Pel{\'a}ez, Abreu, Farr, Gallicchio, Gorenko, Henry, Hu, Huang, Kr{\"a}mer, Michel, Mitchell, Pande, Rodrigues, Rodriguez-Guerra, Simmonett, Singh, Swails, Turner, Wang, Zhang, Chodera, De~Fabritiis, and Markland]{eastman2024openmm}
Peter Eastman, Raimondas Galvelis, Ra{\'u}l~P. Pel{\'a}ez, Charlles R.~A. Abreu, Stephen~E. Farr, Emilio Gallicchio, Anton Gorenko, Michael~M. Henry, Frank Hu, Jing Huang, Andreas Kr{\"a}mer, Julien Michel, Joshua~A. Mitchell, Vijay~S. Pande, Jo{\~a}o P. G. L.~M. Rodrigues, Jaime Rodriguez-Guerra, Andrew~C. Simmonett, Sukrit Singh, Jason Swails, Philip Turner, Yuanqing Wang, Ivy Zhang, John~D. Chodera, Gianni De~Fabritiis, and Thomas~E. Markland.
\newblock {OpenMM} 8: Molecular dynamics simulation with machine learning potentials.
\newblock \emph{The Journal of Physical Chemistry B}, 128\penalty0 (1):\penalty0 109--116, 2024.

\bibitem[Farebrother et~al.(2024)Farebrother, Orbay, Vuong, Ali~Taiga, Chebotar, Xiao, Irpan, Levine, Castro, Faust, Kumar, and Agarwal]{pmlr-v235-farebrother24a}
Jesse Farebrother, Jordi Orbay, Quan Vuong, Adrien Ali~Taiga, Yevgen Chebotar, Ted Xiao, Alex Irpan, Sergey Levine, Pablo~Samuel Castro, Aleksandra Faust, Aviral Kumar, and Rishabh Agarwal.
\newblock Stop regressing: Training value functions via classification for scalable deep {RL}.
\newblock In \emph{International Conference on Machine Learning}, 2024.

\bibitem[Gat et~al.(2024)Gat, Remez, Shaul, Kreuk, Chen, Synnaeve, Adi, and Lipman]{gat2024discrete}
Itai Gat, Tal Remez, Neta Shaul, Felix Kreuk, Ricky T.~Q. Chen, Gabriel Synnaeve, Yossi Adi, and Yaron Lipman.
\newblock Discrete flow matching.
\newblock In \emph{Advances in Neural Information Processing Systems}, 2024.

\bibitem[Ghazvininejad et~al.(2019)Ghazvininejad, Levy, Liu, and Zettlemoyer]{ghazvininejad2019mask}
Marjan Ghazvininejad, Omer Levy, Yinhan Liu, and Luke Zettlemoyer.
\newblock Mask-{P}redict: Parallel decoding of conditional masked language models.
\newblock In \emph{Conference on Empirical Methods in Natural Language Processing}, 2019.

\bibitem[Gong et~al.(2025)Gong, Agarwal, Zhang, Ye, Zheng, Li, An, Zhao, Bi, Han, Peng, and Kong]{gong2025scaling}
Shansan Gong, Shivam Agarwal, Yizhe Zhang, Jiacheng Ye, Lin Zheng, Mukai Li, Chenxin An, Peilin Zhao, Wei Bi, Jiawei Han, Hao Peng, and Lingpeng Kong.
\newblock Scaling diffusion language models via adaptation from autoregressive models.
\newblock In \emph{International Conference on Learning Representations}, 2025.

\bibitem[Guo et~al.(2026{\natexlab{a}})Guo, Choi, Zhu, Tao, and Chen]{guo2026proximal}
Wei Guo, Jaemoo Choi, Yuchen Zhu, Molei Tao, and Yongxin Chen.
\newblock Proximal diffusion neural sampler.
\newblock In \emph{International Conference on Learning Representations}, 2026{\natexlab{a}}.

\bibitem[Guo et~al.(2026{\natexlab{b}})Guo, Zhu, Du, Nam, Chen, G{\'o}mez-Bombarelli, Liu, Tao, and Choi]{guo2026asbs}
Wei Guo, Yuchen Zhu, Xiaochen Du, Juno Nam, Yongxin Chen, Rafael G{\'o}mez-Bombarelli, Guan-Horng Liu, Molei Tao, and Jaemoo Choi.
\newblock Discrete adjoint {Schr{\"o}dinger} bridge sampler.
\newblock \emph{arXiv preprint arXiv:2602.08243}, 2026{\natexlab{b}}.
\newblock URL \url{https://arxiv.org/abs/2602.08243}.

\bibitem[Haarnoja et~al.(2018)Haarnoja, Zhou, Abbeel, and Levine]{haarnoja2018sac}
Tuomas Haarnoja, Aurick Zhou, Pieter Abbeel, and Sergey Levine.
\newblock Soft actor-critic: Off-policy maximum entropy deep reinforcement learning with a stochastic actor.
\newblock In \emph{International Conference on Machine Learning}, pp.\  1861--1870, 2018.

\bibitem[Havens et~al.(2025)Havens, Miller, Yan, Domingo-Enrich, Sriram, Wood, Levine, Hu, Amos, Karrer, et~al.]{havens2025adjoint}
Aaron Havens, Benjamin~Kurt Miller, Bing Yan, Carles Domingo-Enrich, Anuroop Sriram, Brandon Wood, Daniel Levine, Bin Hu, Brandon Amos, Brian Karrer, et~al.
\newblock Adjoint sampling: Highly scalable diffusion samplers via adjoint matching.
\newblock \emph{arXiv preprint arXiv:2504.11713}, 2025.

\bibitem[Havens et~al.(2026)Havens, Karrer, and Shaul]{havens2026flow}
Aaron~J Havens, Brian Karrer, and Neta Shaul.
\newblock Flow sampling : Learning to sample from unnormalized densities via denoising conditional processes.
\newblock In \emph{Forty-third International Conference on Machine Learning}, 2026.
\newblock URL \url{https://openreview.net/forum?id=YlcyOMTPNl}.

\bibitem[He et~al.(2025)He, Du, Vargas, Wang, Gomes, Hern{\'a}ndez-Lobato, and Vanden-Eijnden]{he2025feat}
Jiajun He, Yuanqi Du, Francisco Vargas, Yuanqing Wang, Carla~P Gomes, Jos{\'e}~Miguel Hern{\'a}ndez-Lobato, and Eric Vanden-Eijnden.
\newblock Feat: Free energy estimators with adaptive transport.
\newblock \emph{arXiv preprint arXiv:2504.11516}, 2025.

\bibitem[He et~al.(2026)He, Ou, Vargas, Li, Hern{\'a}ndez-Lobato, Domingo-Enrich, and Du]{he2026free}
Jiajun He, Zijing Ou, Francisco Vargas, Yingzhen Li, Jos{\'e}~Miguel Hern{\'a}ndez-Lobato, Carles Domingo-Enrich, and Yuanqi Du.
\newblock Free energy estimation on any state space.
\newblock \emph{arXiv preprint arXiv:2605.31063}, 2026.
\newblock \doi{10.48550/arXiv.2605.31063}.
\newblock URL \url{https://arxiv.org/abs/2605.31063}.

\bibitem[Hibat-Allah et~al.(2021)Hibat-Allah, Inack, Wiersema, Melko, and Carrasquilla]{hibatallah2021variational}
Mohamed Hibat-Allah, Estelle~M. Inack, Roeland Wiersema, Roger~G. Melko, and Juan Carrasquilla.
\newblock Variational neural annealing.
\newblock \emph{Nature Machine Intelligence}, 3\penalty0 (11):\penalty0 952--961, 2021.

\bibitem[Holderrieth et~al.(2025{\natexlab{a}})Holderrieth, Albergo, and Jaakkola]{holderrieth2025leaps}
Peter Holderrieth, Michael~S. Albergo, and Tommi Jaakkola.
\newblock {LEAPS}: A discrete neural sampler via locally equivariant networks.
\newblock \emph{arXiv preprint arXiv:2502.10843}, 2025{\natexlab{a}}.
\newblock URL \url{https://arxiv.org/abs/2502.10843}.

\bibitem[Holderrieth et~al.(2025{\natexlab{b}})Holderrieth, Havasi, Yim, Shaul, Gat, Jaakkola, Karrer, Chen, and Lipman]{holderrieth2025generator}
Peter Holderrieth, Marton Havasi, Jason Yim, Neta Shaul, Itai Gat, Tommi Jaakkola, Brian Karrer, Ricky T.~Q. Chen, and Yaron Lipman.
\newblock Generator matching: Generative modeling with arbitrary {M}arkov processes.
\newblock In \emph{International Conference on Learning Representations}, 2025{\natexlab{b}}.

\bibitem[Hoogeboom et~al.(2021)Hoogeboom, Nielsen, Jaini, Forr{\'e}, and Welling]{hoogeboom2021argmax}
Emiel Hoogeboom, Didrik Nielsen, Priyank Jaini, Patrick Forr{\'e}, and Max Welling.
\newblock Argmax flows and multinomial diffusion: Learning categorical distributions.
\newblock In \emph{Advances in Neural Information Processing Systems}, 2021.

\bibitem[Hoogeboom et~al.(2022)Hoogeboom, Gritsenko, Bastings, Poole, van~den Berg, and Salimans]{hoogeboom2022autoregressive}
Emiel Hoogeboom, Alexey~A. Gritsenko, Jasmijn Bastings, Ben Poole, Rianne van~den Berg, and Tim Salimans.
\newblock Autoregressive diffusion models.
\newblock In \emph{International Conference on Learning Representations}, 2022.

\bibitem[Hopkins et~al.(2015)Hopkins, Le~Grand, Walker, and Roitberg]{hopkins2015long}
Chad~W. Hopkins, Scott Le~Grand, Ross~C. Walker, and Adrian~E. Roitberg.
\newblock Long-time-step molecular dynamics through hydrogen mass repartitioning.
\newblock \emph{Journal of Chemical Theory and Computation}, 11\penalty0 (4):\penalty0 1864--1874, 2015.

\bibitem[Jain et~al.(2022)Jain, Bengio, Hernandez-Garcia, Rector-Brooks, Dossou, Ekbote, Fu, Zhang, Kilgour, Zhang, Simine, Das, and Bengio]{jain2022biological}
Moksh Jain, Emmanuel Bengio, Alex Hernandez-Garcia, Jarrid Rector-Brooks, Bonaventure F.~P. Dossou, Chanakya Ekbote, Jie Fu, Tianyu Zhang, Michael Kilgour, Dinghuai Zhang, Lena Simine, Payel Das, and Yoshua Bengio.
\newblock Biological sequence design with {GFlowNets}.
\newblock In \emph{International Conference on Machine Learning}, 2022.

\bibitem[Kaufman(1949)]{kaufman1949crystal}
Bruria Kaufman.
\newblock Crystal statistics. {II}. {P}artition function evaluated by spinor analysis.
\newblock \emph{Physical Review}, 76\penalty0 (8):\penalty0 1232--1243, 1949.

\bibitem[Kim et~al.(2025)Kim, Shah, Kontonis, Kakade, and Chen]{kim2025train}
Jaeyeon Kim, Kulin Shah, Vasilis Kontonis, Sham Kakade, and Sitan Chen.
\newblock Train for the worst, plan for the best: Understanding token ordering in masked diffusions.
\newblock In \emph{International Conference on Machine Learning}, 2025.

\bibitem[Kingma \& Ba(2015)Kingma and Ba]{kingma2015adam}
Diederik~P. Kingma and Jimmy Ba.
\newblock Adam: A method for stochastic optimization.
\newblock In \emph{International Conference on Learning Representations}, 2015.

\bibitem[Kong et~al.(1994)Kong, Liu, and Wong]{kong1994sequential}
Augustine Kong, Jun~S. Liu, and Wing~Hung Wong.
\newblock Sequential imputations and {Bayesian} missing data problems.
\newblock \emph{Journal of the American Statistical Association}, 89\penalty0 (425):\penalty0 278--288, 1994.

\bibitem[Laine et~al.(2019)Laine, Karras, Lehtinen, and Aila]{laine2019high}
Samuli Laine, Tero Karras, Jaakko Lehtinen, and Timo Aila.
\newblock High-quality self-supervised deep image denoising.
\newblock \emph{Advances in neural information processing systems}, 32, 2019.

\bibitem[Le et~al.(2026)Le, Celik, Blessing, Hoang, Voelcker, Brunnbauer, Richter, Volpp, and Neumann]{trudi2025}
Huy Le, Onur Celik, Denis Blessing, Tai Hoang, Claas~A. Voelcker, Axel Brunnbauer, Felix Richter, Michael Volpp, and Gerhard Neumann.
\newblock Trust-region diffusion policies for massively parallel on-policy {RL}.
\newblock In \emph{International Conference on Machine Learning}, 2026.

\bibitem[Leimkuhler \& Matthews(2013)Leimkuhler and Matthews]{leimkuhler2013rational}
Benedict Leimkuhler and Charles Matthews.
\newblock Rational construction of stochastic numerical methods for molecular sampling.
\newblock \emph{Applied Mathematics Research eXpress}, 2013\penalty0 (1):\penalty0 34--56, 2013.

\bibitem[Lindorff-Larsen et~al.(2010)Lindorff-Larsen, Piana, Palmo, Maragakis, Klepeis, Dror, and Shaw]{lindorff2010improved}
Kresten Lindorff-Larsen, Stefano Piana, Kim Palmo, Paul Maragakis, John~L. Klepeis, Ron~O. Dror, and David~E. Shaw.
\newblock Improved side-chain torsion potentials for the {A}mber {ff99SB} protein force field.
\newblock \emph{Proteins: Structure, Function, and Bioinformatics}, 78\penalty0 (8):\penalty0 1950--1958, 2010.

\bibitem[Liu et~al.(2025{\natexlab{a}})Liu, Choi, Chen, Miller, and Chen]{liu2025adjoint}
Guan-Horng Liu, Jaemoo Choi, Yongxin Chen, Benjamin~Kurt Miller, and Ricky T.~Q. Chen.
\newblock Adjoint {Schr{\"o}dinger} bridge sampler.
\newblock \emph{arXiv preprint arXiv:2506.22565}, 2025{\natexlab{a}}.
\newblock URL \url{https://arxiv.org/abs/2506.22565}.

\bibitem[Liu et~al.(2025{\natexlab{b}})Liu, Nam, Campbell, St{\"a}rk, Xu, Jaakkola, and G{\'o}mez-Bombarelli]{liu2025think}
Sulin Liu, Juno Nam, Andrew Campbell, Hannes St{\"a}rk, Yilun Xu, Tommi Jaakkola, and Rafael G{\'o}mez-Bombarelli.
\newblock Think while you generate: Discrete diffusion with planned denoising.
\newblock In \emph{International Conference on Learning Representations}, 2025{\natexlab{b}}.

\bibitem[Loshchilov \& Hutter(2019)Loshchilov and Hutter]{loshchilov2019decoupled}
Ilya Loshchilov and Frank Hutter.
\newblock Decoupled weight decay regularization.
\newblock In \emph{International Conference on Learning Representations}, 2019.

\bibitem[Lou et~al.(2024)Lou, Meng, and Ermon]{lou2024discrete}
Aaron Lou, Chenlin Meng, and Stefano Ermon.
\newblock Discrete diffusion modeling by estimating the ratios of the data distribution.
\newblock In \emph{International Conference on Machine Learning}, 2024.

\bibitem[Malkin et~al.(2022)Malkin, Jain, Bengio, Sun, and Bengio]{malkin2022trajectory}
Nikolay Malkin, Moksh Jain, Emmanuel Bengio, Chen Sun, and Yoshua Bengio.
\newblock Trajectory balance: Improved credit assignment in {GFlowNets}.
\newblock In \emph{Advances in Neural Information Processing Systems}, 2022.

\bibitem[McAllister et~al.(2026)McAllister, Ge, Yi, Kim, Weber, Choi, Feng, and Kanazawa]{mcallister2026flow}
David McAllister, Songwei Ge, Brent Yi, Chung~Min Kim, Ethan Weber, Hongsuk Choi, Haiwen Feng, and Angjoo Kanazawa.
\newblock Flow matching policy gradients.
\newblock In \emph{International Conference on Learning Representations}, 2026.

\bibitem[McNaughton et~al.(2020)McNaughton, Milo{\v{s}}evi{\'c}, Perali, and Pilati]{mcnaughton2020boosting}
B.~McNaughton, M.~V. Milo{\v{s}}evi{\'c}, A.~Perali, and S.~Pilati.
\newblock Boosting {Monte Carlo} simulations of spin glasses using autoregressive neural networks.
\newblock \emph{Physical Review E}, 101\penalty0 (5):\penalty0 053312, 2020.

\bibitem[Meng et~al.(2022)Meng, Choi, Song, and Ermon]{meng2022concrete}
Chenlin Meng, Kristy Choi, Jiaming Song, and Stefano Ermon.
\newblock Concrete score matching: Generalized score matching for discrete data.
\newblock In \emph{Advances in Neural Information Processing Systems}, 2022.

\bibitem[Min et~al.(2026)Min, Hofgard, and Smidt]{min2026frogs}
Kyucheol Min, Elyssa Hofgard, and Tess Smidt.
\newblock {FrOGS}: Discrete neural sampler for independent alloy configurations across chemical conditions.
\newblock \emph{arXiv preprint arXiv:2609.02948}, 2026.

\bibitem[Nicoli et~al.(2020)Nicoli, Nakajima, Strodthoff, Samek, M{\"u}ller, and Kessel]{nicoli2020asymptotically}
Kim~A. Nicoli, Shinichi Nakajima, Nils Strodthoff, Wojciech Samek, Klaus-Robert M{\"u}ller, and Pan Kessel.
\newblock Asymptotically unbiased estimation of physical observables with neural samplers.
\newblock \emph{Physical Review E}, 101\penalty0 (2):\penalty0 023304, 2020.

\bibitem[Nie et~al.(2025{\natexlab{a}})Nie, Zhu, Du, Pang, Liu, Zeng, Lin, and Li]{nie2025scaling}
Shen Nie, Fengqi Zhu, Chao Du, Tianyu Pang, Qian Liu, Guangtao Zeng, Min Lin, and Chongxuan Li.
\newblock Scaling up masked diffusion models on text.
\newblock In \emph{International Conference on Learning Representations}, 2025{\natexlab{a}}.

\bibitem[Nie et~al.(2025{\natexlab{b}})Nie, Zhu, You, Zhang, Ou, Hu, Zhou, Lin, Wen, and Li]{nie2025large}
Shen Nie, Fengqi Zhu, Zebin You, Xiaolu Zhang, Jingyang Ou, Jun Hu, Jun Zhou, Yankai Lin, Ji-Rong Wen, and Chongxuan Li.
\newblock Large language diffusion models.
\newblock In \emph{Advances in Neural Information Processing Systems}, 2025{\natexlab{b}}.

\bibitem[Nisonoff et~al.(2025)Nisonoff, Xiong, Allenspach, and Listgarten]{nisonoff2025unlocking}
Hunter Nisonoff, Junhao Xiong, Stephan Allenspach, and Jennifer Listgarten.
\newblock Unlocking guidance for discrete state-space diffusion and flow models.
\newblock In \emph{International Conference on Learning Representations}, 2025.

\bibitem[No{\'e} et~al.(2019)No{\'e}, Olsson, K{\"o}hler, and Wu]{noe2019boltzmann}
Frank No{\'e}, Simon Olsson, Jonas K{\"o}hler, and Hao Wu.
\newblock Boltzmann generators: Sampling equilibrium states of many-body systems with deep learning.
\newblock \emph{Science}, 365\penalty0 (6457):\penalty0 eaaw1147, 2019.

\bibitem[Onufriev et~al.(2004)Onufriev, Bashford, and Case]{onufriev2004exploring}
Alexey Onufriev, Donald Bashford, and David~A. Case.
\newblock Exploring protein native states and large-scale conformational changes with a modified generalized {B}orn model.
\newblock \emph{Proteins: Structure, Function, and Bioinformatics}, 55\penalty0 (2):\penalty0 383--394, 2004.

\bibitem[Ou et~al.(2025{\natexlab{a}})Ou, Nie, Xue, Zhu, Sun, Li, and Li]{ou2025absorbing}
Jingyang Ou, Shen Nie, Kaiwen Xue, Fengqi Zhu, Jiacheng Sun, Zhenguo Li, and Chongxuan Li.
\newblock Your absorbing discrete diffusion secretly models the conditional distributions of clean data.
\newblock In \emph{International Conference on Learning Representations}, 2025{\natexlab{a}}.

\bibitem[Ou et~al.(2025{\natexlab{b}})Ou, Zhang, and Li]{ou2025discrete}
Zijing Ou, Ruixiang Zhang, and Yingzhen Li.
\newblock Discrete neural flow samplers with locally equivariant transformer.
\newblock In \emph{Advances in Neural Information Processing Systems}, 2025{\natexlab{b}}.

\bibitem[Pan et~al.(2021)Pan, Zhou, Zhou, and Zhang]{pan2021solving}
Feng Pan, Pengfei Zhou, Hai-Jun Zhou, and Pan Zhang.
\newblock Solving statistical mechanics on sparse graphs with feedback-set variational autoregressive networks.
\newblock \emph{Physical Review E}, 103\penalty0 (1):\penalty0 012103, 2021.

\bibitem[Peng et~al.(2025)Peng, Bezemek, Patel, Rector-Brooks, Yao, Bose, Tong, and Chatterjee]{peng2025path}
Fred~Zhangzhi Peng, Zachary Bezemek, Sawan Patel, Jarrid Rector-Brooks, Sherwood Yao, Avishek~Joey Bose, Alexander Tong, and Pranam Chatterjee.
\newblock Path planning for masked diffusion model sampling.
\newblock \emph{arXiv preprint arXiv:2502.03540}, 2025.

\bibitem[Perez et~al.(2018)Perez, Strub, de~Vries, Dumoulin, and Courville]{perez2018film}
Ethan Perez, Florian Strub, Harm de~Vries, Vincent Dumoulin, and Aaron Courville.
\newblock {FiLM}: Visual reasoning with a general conditioning layer.
\newblock In \emph{Proceedings of the AAAI Conference on Artificial Intelligence}, pp.\  3942--3951, 2018.

\bibitem[Ramachandran et~al.(1963)Ramachandran, Ramakrishnan, and Sasisekharan]{ramachandran1963stereochemistry}
G.~N. Ramachandran, C.~Ramakrishnan, and V.~Sasisekharan.
\newblock Stereochemistry of polypeptide chain configurations.
\newblock \emph{Journal of Molecular Biology}, 7\penalty0 (1):\penalty0 95--99, 1963.

\bibitem[Rehman et~al.(2026)Rehman, Tan, Bengio, Bose, and Tong]{rehman2026autoregressive}
Danyal Rehman, Charlie~B. Tan, Yoshua Bengio, Avishek~Joey Bose, and Alexander Tong.
\newblock Autoregressive {B}oltzmann generators.
\newblock In \emph{International Conference on Machine Learning}, 2026.

\bibitem[Ren et~al.(2025)Ren, Lidard, Ankile, Simeonov, Agrawal, Majumdar, Burchfiel, Dai, and Simchowitz]{ren2025dppo}
Allen~Z. Ren, Justin Lidard, Lars~Lien Ankile, Anthony Simeonov, Pulkit Agrawal, Anirudha Majumdar, Benjamin Burchfiel, Hongkai Dai, and Max Simchowitz.
\newblock Diffusion policy policy optimization.
\newblock In \emph{International Conference on Learning Representations}, 2025.

\bibitem[Sahoo et~al.(2024)Sahoo, Arriola, Schiff, Gokaslan, Marroquin, Chiu, Rush, and Kuleshov]{sahoo2024simple}
Subham~Sekhar Sahoo, Marianne Arriola, Yair Schiff, Aaron Gokaslan, Edgar Marroquin, Justin~T. Chiu, Alexander Rush, and Volodymyr Kuleshov.
\newblock Simple and effective masked diffusion language models.
\newblock In \emph{Advances in Neural Information Processing Systems}, 2024.

\bibitem[Sahoo et~al.(2025)Sahoo, Deschenaux, Gokaslan, Wang, Chiu, and Kuleshov]{sahoo2025diffusion}
Subham~Sekhar Sahoo, Justin Deschenaux, Aaron Gokaslan, Guanghan Wang, Justin Chiu, and Volodymyr Kuleshov.
\newblock The diffusion duality.
\newblock In \emph{International Conference on Machine Learning}, 2025.

\bibitem[Sanchez et~al.(1984)Sanchez, Ducastelle, and Gratias]{sanchez1984generalized}
J.~M. Sanchez, F.~Ducastelle, and D.~Gratias.
\newblock Generalized cluster description of multicomponent systems.
\newblock \emph{Physica A: Statistical Mechanics and its Applications}, 128\penalty0 (1--2):\penalty0 334--350, 1984.

\bibitem[Sanokowski et~al.(2023)Sanokowski, Berghammer, Hochreiter, and Lehner]{sanokowski2023variational}
Sebastian Sanokowski, Wilhelm Berghammer, Sepp Hochreiter, and Sebastian Lehner.
\newblock Variational annealing on graphs for combinatorial optimization.
\newblock In \emph{Advances in Neural Information Processing Systems}, 2023.

\bibitem[Sanokowski et~al.(2024)Sanokowski, Hochreiter, and Lehner]{sanokowski2024diffusion}
Sebastian Sanokowski, Sepp Hochreiter, and Sebastian Lehner.
\newblock A diffusion model framework for unsupervised neural combinatorial optimization.
\newblock In \emph{International Conference on Machine Learning}, 2024.

\bibitem[Sanokowski et~al.(2025)Sanokowski, Berghammer, Ennemoser, Wang, Hochreiter, and Lehner]{sanokowski2025scalable}
Sebastian Sanokowski, Wilhelm Berghammer, Martin Ennemoser, Haoyu~Peter Wang, Sepp Hochreiter, and Sebastian Lehner.
\newblock Scalable discrete diffusion samplers: Combinatorial optimization and statistical physics.
\newblock In \emph{International Conference on Learning Representations}, 2025.

\bibitem[Schiff et~al.(2025)Schiff, Sahoo, Phung, Wang, Boshar, Dalla-Torre, de~Almeida, Rush, Pierrot, and Kuleshov]{schiff2025simple}
Yair Schiff, Subham~Sekhar Sahoo, Hao Phung, Guanghan Wang, Sam Boshar, Hugo Dalla-Torre, Bernardo~P. de~Almeida, Alexander Rush, Thomas Pierrot, and Volodymyr Kuleshov.
\newblock Simple guidance mechanisms for discrete diffusion models.
\newblock In \emph{International Conference on Learning Representations}, 2025.

\bibitem[Schulman et~al.(2015)Schulman, Levine, Abbeel, Jordan, and Moritz]{schulman2015trust}
John Schulman, Sergey Levine, Pieter Abbeel, Michael~I. Jordan, and Philipp Moritz.
\newblock Trust region policy optimization.
\newblock In \emph{International Conference on Machine Learning}, 2015.

\bibitem[Schulman et~al.(2017)Schulman, Wolski, Dhariwal, Radford, and Klimov]{schulman2017ppo}
John Schulman, Filip Wolski, Prafulla Dhariwal, Alec Radford, and Oleg Klimov.
\newblock Proximal policy optimization algorithms.
\newblock \emph{arXiv preprint arXiv:1707.06347}, 2017.

\bibitem[Shaul et~al.(2025)Shaul, Gat, Havasi, Severo, Sriram, Holderrieth, Karrer, Lipman, and Chen]{shaul2025flow}
Neta Shaul, Itai Gat, Marton Havasi, Daniel Severo, Anuroop Sriram, Peter Holderrieth, Brian Karrer, Yaron Lipman, and Ricky T.~Q. Chen.
\newblock Flow matching with general discrete paths: A kinetic-optimal perspective.
\newblock In \emph{International Conference on Learning Representations}, 2025.
\newblock URL \url{https://arxiv.org/abs/2412.03487}.

\bibitem[Shi et~al.(2024)Shi, Han, Wang, Doucet, and Titsias]{shi2024simplified}
Jiaxin Shi, Kehang Han, Zhe Wang, Arnaud Doucet, and Michalis~K. Titsias.
\newblock Simplified and generalized masked diffusion for discrete data.
\newblock In \emph{Advances in Neural Information Processing Systems}, 2024.

\bibitem[Shih et~al.(2022)Shih, Sadigh, and Ermon]{shih2022training}
Andy Shih, Dorsa Sadigh, and Stefano Ermon.
\newblock Training and inference on any-order autoregressive models the right way.
\newblock In \emph{Advances in Neural Information Processing Systems}, 2022.

\bibitem[Shirts \& Chodera(2008)Shirts and Chodera]{shirts2008statistically}
Michael~R. Shirts and John~D. Chodera.
\newblock Statistically optimal analysis of samples from multiple equilibrium states.
\newblock \emph{The Journal of Chemical Physics}, 129\penalty0 (12):\penalty0 124105, 2008.

\bibitem[Sugita \& Okamoto(1999)Sugita and Okamoto]{sugita1999replica}
Yuji Sugita and Yuko Okamoto.
\newblock Replica-exchange molecular dynamics method for protein folding.
\newblock \emph{Chemical Physics Letters}, 314\penalty0 (1--2):\penalty0 141--151, 1999.

\bibitem[Sun et~al.(2022)Sun, Yu, Dai, Schuurmans, and Dai]{sun2022score}
Haoran Sun, Lijun Yu, Bo~Dai, Dale Schuurmans, and Hanjun Dai.
\newblock Score-based continuous-time discrete diffusion models.
\newblock \emph{arXiv preprint arXiv:2211.16750}, 2022.

\bibitem[Sun et~al.(2023)Sun, Yu, Dai, Schuurmans, and Dai]{sun2023score}
Haoran Sun, Lijun Yu, Bo~Dai, Dale Schuurmans, and Hanjun Dai.
\newblock Score-based continuous-time discrete diffusion models.
\newblock In \emph{International Conference on Learning Representations}, 2023.

\bibitem[Swendsen \& Wang(1987)Swendsen and Wang]{swendsen1987nonuniversal}
Robert~H. Swendsen and Jian-Sheng Wang.
\newblock Nonuniversal critical dynamics in {M}onte {C}arlo simulations.
\newblock \emph{Physical Review Letters}, 58\penalty0 (2):\penalty0 86--88, 1987.

\bibitem[Theis et~al.(2016)Theis, van~den Oord, and Bethge]{theis2016note}
Lucas Theis, A{\"a}ron van~den Oord, and Matthias Bethge.
\newblock A note on the evaluation of generative models.
\newblock In \emph{International Conference on Learning Representations}, 2016.

\bibitem[Thilges et~al.(2026)Thilges, Celik, Blessing, Gospodinov, and Neumann]{amdp2025}
Serge Thilges, Onur Celik, Denis Blessing, Emiliyan Gospodinov, and Gerhard Neumann.
\newblock Scalable maximum entropy reinforcement learning for diffusion policies via adjoint matching.
\newblock \emph{arXiv preprint arXiv:2606.22630}, 2026.

\bibitem[Torrie \& Valleau(1977)Torrie and Valleau]{torrie1977nonphysical}
G.~M. Torrie and J.~P. Valleau.
\newblock Nonphysical sampling distributions in {M}onte {C}arlo free-energy estimation: Umbrella sampling.
\newblock \emph{Journal of Computational Physics}, 23\penalty0 (2):\penalty0 187--199, 1977.

\bibitem[Uria et~al.(2014)Uria, Murray, and Larochelle]{uria2014deep}
Benigno Uria, Iain Murray, and Hugo Larochelle.
\newblock A deep and tractable density estimator.
\newblock In \emph{International Conference on Machine Learning}, 2014.

\bibitem[van~den Oord et~al.(2016)van~den Oord, Kalchbrenner, and Kavukcuoglu]{vandenoord2016pixel}
A{\"a}ron van~den Oord, Nal Kalchbrenner, and Koray Kavukcuoglu.
\newblock Pixel recurrent neural networks.
\newblock In \emph{Proceedings of the 33rd International Conference on Machine Learning}, pp.\  1747--1756, 2016.

\bibitem[Voelcker et~al.(2026)Voelcker, Brunnbauer, Hussing, Nauman, Abbeel, Grosu, Eaton, Farahmand, and Gilitschenski]{reppo2025}
Claas~A. Voelcker, Axel Brunnbauer, Marcel Hussing, Michal Nauman, Pieter Abbeel, Radu Grosu, Eric Eaton, Amir-massoud Farahmand, and Igor Gilitschenski.
\newblock Relative entropy pathwise policy optimization.
\newblock In \emph{International Conference on Learning Representations}, 2026.

\bibitem[Wang et~al.(2025)Wang, Schiff, Sahoo, and Kuleshov]{wang2025remasking}
Guanghan Wang, Yair Schiff, Subham~Sekhar Sahoo, and Volodymyr Kuleshov.
\newblock Remasking discrete diffusion models with inference-time scaling.
\newblock In \emph{Advances in Neural Information Processing Systems}, 2025.

\bibitem[Wang et~al.(2024)Wang, Zheng, Ye, Xue, Huang, and Gu]{wang2024diffusion}
Xinyou Wang, Zaixiang Zheng, Fei Ye, Dongyu Xue, Shujian Huang, and Quanquan Gu.
\newblock Diffusion language models are versatile protein learners.
\newblock In \emph{International Conference on Machine Learning}, 2024.

\bibitem[Wolff(1989)]{wolff1989collective}
Ulli Wolff.
\newblock Collective {M}onte {C}arlo updating for spin systems.
\newblock \emph{Physical Review Letters}, 62\penalty0 (4):\penalty0 361--364, 1989.

\bibitem[Woo et~al.(2025)Woo, Kim, Kim, Seong, and Ahn]{woo2025energy}
Dongyeop Woo, Minsu Kim, Minkyu Kim, Kiyoung Seong, and Sungsoo Ahn.
\newblock Energy-based generator matching: A neural sampler for general state space.
\newblock \emph{arXiv preprint arXiv:2505.19646}, 2025.
\newblock URL \url{https://arxiv.org/abs/2505.19646}.

\bibitem[Wu et~al.(2019)Wu, Wang, and Zhang]{wu2019solving}
Dian Wu, Lei Wang, and Pan Zhang.
\newblock Solving statistical mechanics using variational autoregressive networks.
\newblock \emph{Physical Review Letters}, 122\penalty0 (8):\penalty0 080602, 2019.

\bibitem[Xie et~al.(2025)Xie, Zhang, Yang, Hutter, and Xu]{xie2025simple}
Zhengpeng Xie, Qiang Zhang, Fan Yang, Marco Hutter, and Renjing Xu.
\newblock Simple policy optimization.
\newblock In \emph{International Conference on Machine Learning}, 2025.

\bibitem[Xu et~al.(2025)Xu, Geffner, Kreis, Nie, Xu, Leskovec, Ermon, and Vahdat]{xu2025energy}
Minkai Xu, Tomas Geffner, Karsten Kreis, Weili Nie, Yilun Xu, Jure Leskovec, Stefano Ermon, and Arash Vahdat.
\newblock Energy-based diffusion language models for text generation.
\newblock In \emph{International Conference on Learning Representations}, 2025.

\bibitem[Ye et~al.(2025)Ye, Xie, Zheng, Gao, Wu, Jiang, Li, and Kong]{ye2025dream}
Jiacheng Ye, Zhihui Xie, Lin Zheng, Jiahui Gao, Zirui Wu, Xin Jiang, Zhenguo Li, and Lingpeng Kong.
\newblock Dream 7{B}: Diffusion large language models.
\newblock \emph{arXiv preprint arXiv:2508.15487}, 2025.

\bibitem[Zakka et~al.(2025)Zakka, Tabanpour, Liao, Haiderbhai, Holt, Luo, Allshire, Frey, Sreenath, Kahrs, Sferrazza, Tassa, and Abbeel]{mujoco_playground_2025}
Kevin Zakka, Baruch Tabanpour, Qiayuan Liao, Mustafa Haiderbhai, Samuel Holt, Jing~Yuan Luo, Arthur Allshire, Erik Frey, Koushil Sreenath, Lueder~A. Kahrs, Carlo Sferrazza, Yuval Tassa, and Pieter Abbeel.
\newblock {MuJoCo Playground}: An open-source framework for {GPU}-accelerated robot learning and sim-to-real transfer, 2025.
\newblock URL \url{https://github.com/google-deepmind/mujoco_playground}.

\bibitem[Zhang et~al.(2022)Zhang, Malkin, Liu, Volokhova, Courville, and Bengio]{zhang2022generative}
Dinghuai Zhang, Nikolay Malkin, Zhen Liu, Alexandra Volokhova, Aaron Courville, and Yoshua Bengio.
\newblock Generative flow networks for discrete probabilistic modeling.
\newblock In \emph{International Conference on Machine Learning}, 2022.

\bibitem[Zhang et~al.(2023)Zhang, Dai, Malkin, Courville, Bengio, and Pan]{zhang2023let}
Dinghuai Zhang, Hanjun Dai, Nikolay Malkin, Aaron~C Courville, Yoshua Bengio, and Ling Pan.
\newblock Let the flows tell: Solving graph combinatorial problems with gflownets.
\newblock \emph{Advances in neural information processing systems}, 36:\penalty0 11952--11969, 2023.

\bibitem[Zhang et~al.(2025)Zhang, Zhai, Zhang, Thornton, Ou, Susskind, and Jaitly]{zhang2025target}
Ruixiang Zhang, Shuangfei Zhai, Yizhe Zhang, James Thornton, Zijing Ou, Joshua Susskind, and Navdeep Jaitly.
\newblock Target concrete score matching: A holistic framework for discrete diffusion.
\newblock In \emph{International Conference on Machine Learning}, 2025.

\bibitem[Zhao et~al.(2024)Zhao, Shi, Chen, Druckmann, Mackey, and Linderman]{zhao2024informed}
Yixiu Zhao, Jiaxin Shi, Feng Chen, Shaul Druckmann, Lester Mackey, and Scott Linderman.
\newblock Informed correctors for discrete diffusion models.
\newblock \emph{arXiv preprint arXiv:2407.21243}, 2024.

\bibitem[Zhao et~al.(2026)Zhao, Shi, Chen, Druckmann, Mackey, and Linderman]{zhao2026informed}
Yixiu Zhao, Jiaxin Shi, Feng Chen, Shaul Druckmann, Lester Mackey, and Scott Linderman.
\newblock Informed correctors for discrete diffusion models.
\newblock \emph{Advances in Neural Information Processing Systems}, 38:\penalty0 125510--125538, 2026.

\bibitem[Zheng et~al.(2025)Zheng, Chen, Mao, Liu, Zhu, and Zhang]{zheng2025masked}
Kaiwen Zheng, Yongxin Chen, Hanzi Mao, Ming-Yu Liu, Jun Zhu, and Qinsheng Zhang.
\newblock Masked diffusion models are secretly time-agnostic masked models and exploit inaccurate categorical sampling.
\newblock In \emph{International Conference on Learning Representations}, 2025.

\bibitem[Zhu et~al.(2025)Zhu, Guo, Choi, Liu, Chen, and Tao]{zhu2025mdns}
Yuchen Zhu, Wei Guo, Jaemoo Choi, Guan-Horng Liu, Yongxin Chen, and Molei Tao.
\newblock {MDNS}: Masked diffusion neural sampler via stochastic optimal control.
\newblock \emph{arXiv preprint arXiv:2508.10684}, 2025.
\newblock URL \url{https://arxiv.org/abs/2508.10684}.

\end{thebibliography}

\clearpage
\appendix
\appendixcontentstrue
\section*{Appendix contents}
\label{app:contents}
\begingroup
\small
\setcounter{tocdepth}{2}
\makeatletter\@starttoc{toc}\makeatother
\endgroup
\clearpage

\section{Additional method details}

\subsection{Damping the fixed-point iteration}
\label{app:damping}

Each step of the fixed-point iteration $\theta_{i+1}=\arg\min_\theta
\mathcal L_{\mathrm{FP}}(\theta;\theta_i,\nu_i)$ fits the predictor to Gibbs
labels computed on samples of the current sampler $q_{\theta_i}$. To further stabilize training, we follow \cite{blessing2026bridge} and use a damped fixed-point iteration, which replaces the Gibbs label by a mixture of the Gibbs label and the current
model's own prediction, i.e.,
\begin{equation}
  h_\eta^d(c\mid X)=\eta\,h^d(c\mid X^{-d})+(1-\eta)\,\mu_{\theta_i}^d(c\mid X_R),
  \qquad \eta\in(0,1],
  \label{eq:damped_label}
\end{equation}
and use it in place of $h^d$ in the tempered loss:
\begin{equation}
  \mathcal L_{\mathrm{FP}}^{\eta}(\theta;\theta_i,\nu_i)
  \propto
  -\mathbb E_{X\sim q_{\theta_i}}
  \Big[
    \bar w(X)\,
    \mathbb E_R\Big[
      \sum_{d\notin R}\sum_{c\in\mathcal A}
      h_\eta^d(c\mid X)\log\mu_\theta^d(c\mid X_R)
    \Big]
  \Big],
  \label{eq:damped_loss}
\end{equation}
which gives the damped iteration
$\theta_{i+1}=\arg\min_\theta\mathcal L_{\mathrm{FP}}^{\eta}(\theta;\theta_i,\nu_i)$. Here $\bar w$ are the tempered importance weights introduced in \cref{sec:method}.
Setting $\eta=1$ recovers the undamped iteration.  In practice
$\mu_{\theta_i}$ is the EMA model, which is also used to collect samples. If the minimization  of \eqref{eq:damped_loss} is exact, the damped update is therefore
the average of the undamped update $\mu_{\theta_{i+1}}^{\eta=1}$ and the
current predictor,
\begin{equation}
  \mu_{\theta_{i+1}}
  =\eta\,\mu_{\theta_{i+1}}^{\eta=1}+(1-\eta)\,\mu_{\theta_i}.
  \label{eq:damped_update}
\end{equation}
If the sampler matches the target exactly, the undamped update leaves it unchanged,
and by \eqref{eq:damped_update} the damped update does too. Conversely, a fixed point of the damped iteration is also fixed point of the undamped one.

\subsection{Importance weighting and Free Energy Estimation}\label{app:is_f}

In this section, we provide details for calculating the importance weights $w$ for both training and free energy estimation. 

\paragraph{Masked diffusion.}
For a generated configuration $x$, the log importance weight is
\begin{equation}
\log w
=
\log\tilde p(x)
-
\sum_{\text{reveals }(R,d)}
\log\mu_\theta^d(x^d\mid x_R),
\label{eq:iw_mask}
\end{equation}
where the sum runs over the reveals, with $R$ the revealed set
immediately before generating site $d$.
This follows from the ratio of the forward masking and backward
generation path densities with masking and unmasking order probability canceled (a more detailed derivation can be found in Appendix G in \cite{he2026free}).
\paragraph{Any-order generation.}
For a configuration $x$ generated in a sampled reveal order,
the importance weight is the same as in \cref{eq:iw_mask}.
The product follows the sampled order, using the probability
assigned to each generated token given the previously revealed
tokens. When the order is sampled independently of token values,
its probability cancels between the augmented target and proposal.

\paragraph{General forward-backward processes.}
For endpoints $p_i=\tilde p_i/Z_i$, $i\in\{0,1\}$, with evaluable
unnormalized densities, a trajectory generated from $p_0$ using
$Q_t$ has log importance weight
\begin{align}
\log w
&=
\log\frac{\tilde p_1(X_1)}{\tilde p_0(X_0)}
+
\sum_{\substack{0<t<1\\X_{t^-}\neq X_t}}
\log
\frac{\bar Q_t(X_t,X_{t^-})}
     {Q_t(X_{t^-},X_t)}+
\int_0^1
\left[
\bar Q_t(X_t,X_t)-Q_t(X_t,X_t)
\right]\mathrm dt.
\label{eq:iw_general}
\end{align}
For a time grid $0=t_0<\cdots<t_K=1$, the discrete-time
log importance weight is
\begin{equation}
\log w
=
\log\frac{\tilde p_1(X_1)}{\tilde p_0(X_0)}
+
\sum_{k=0}^{K-1}
\log
\frac{
\bar P_{t_k|t_{k+1}}(X_{t_k}|X_{t_{k+1}})
}{
P_{t_{k+1}|t_k}(X_{t_{k+1}}|X_{t_k})
}.
\label{eq:iw_general_discrete}
\end{equation}
Here $P_{t_{k+1}|t_k}$ and $\bar P_{t_k|t_{k+1}}$ are the
forward and backward transition densities used by the
discrete-time sampler.
The sum includes steps with $X_{t_{k+1}}=X_{t_k}$, so the
holding-time contribution is already included.

\paragraph{Free-energy estimation.}
For $M$ independent masked-diffusion or any-order samples, with
weights $w_i$ calculated using \eqref{eq:iw_mask}, we have
\begin{equation}
\mathbb E[w]=Z,
\qquad
\widehat Z=\frac1M\sum_{i=1}^M w_i,
\qquad
\widehat F=-\log\widehat Z,
\label{eq:fe_absolute}
\end{equation}
where $F=-\log Z$ is the dimensionless free energy.
For $M$ independent forward trajectories initialized from $p_0$,
the weights in \eqref{eq:iw_general} or
\eqref{eq:iw_general_discrete} instead satisfy
\begin{equation}
\mathbb E[w]=\frac{Z_1}{Z_0},
\qquad
\widehat{\Delta F}
=
-\log\!\left(\frac1M\sum_{i=1}^M w_i\right),
\label{eq:fe_difference}
\end{equation}
which estimates $\Delta F=-\log(Z_1/Z_0)$.
This is an instance of the escorted Jarzynski and Crooks in discrete space \citep{he2025feat,he2026free}.

\section{Convergence of DIG-IN
}
\label{app:convergence}

This section studies an idealized population version of the masked
training loop, designed to isolate the dynamics of the fixed-point
iteration from finite-sample and optimization errors. We replace the
neural network by an unrestricted table of positive categorical heads
$g$, one for every training-supported context. The state space is
finite, the target distribution $p$ is strictly positive, and both the
sampling distribution and the tempered training objective are evaluated
exactly. Each outer iteration returns the exact population minimizer
rather than the result of finite neural-network optimization.

The reveal-order law $\rho$ is fixed and independent of the generated
values. Training masks are drawn independently of both the endpoint and
the generating order, and their support covers every predictor head
that can be queried during generation. Under these assumptions, one
outer iteration is the deterministic map $g\mapsto\Phi_\nu(g)$ given
contextwise by \cref{eq:population_map}. The analysis below establishes
when repeated application of this map converges to the target heads.

\subsection{Iterative predictor updates and convergence statement}
\label{app:convergence_statement}
With $g$ denoting the frozen categorical predictors, for a reveal order
$\sigma=(\sigma_1,\ldots,\sigma_D)$, let
$q_g(x\mid\sigma)$ be the probability that the categorical predictors
$g$ generate the complete configuration $x$ by revealing coordinates
in the order $\sigma$:
\[
q_g(x\mid\sigma)
=
\prod_{k=1}^D
g^{\sigma_k}\!\left(
    x^{\sigma_k}
    \,\middle|\,
    x_{\sigma_{<k}}
\right),
\qquad
\sigma_{<k}=(\sigma_1,\ldots,\sigma_{k-1}).
\]
Marginalizing over the order $\sigma$ while tempering the augmented
weights induces the endpoint distribution
\begin{equation}
 r_{\nu,g}(x)\propto
 p(x)^\nu\sum_\sigma\rho(\sigma)q_g(x\mid\sigma)^{1-\nu}.
 \label{eq:tempered_endpoint}
\end{equation}
Note that the fixed-order case is recovered by setting $\rho = \delta_{\sigma_0}$.

Let $\Phi_\nu(g)$ be the exact population minimizer of the Gibbs-label cross-entropy loss $\mathcal{L}_{\nu,g}$ when the completed configurations are
distributed according to $r_{\nu,g}$, which admits the following expression:
\begin{equation}
 \mathcal{L}_{\nu,g}(\widetilde g)
 =
 -\mathbb{E}_{X\sim r_{\nu,g},\,R}
 \left[
   \sum_{d\notin R}\sum_{c\in\mathcal A}
   h^d(c\mid X^{-d})
   \log \widetilde g^d(c\mid X_R)
 \right].
 \label{eq:masked_loss}
\end{equation}
For each context $(d,R,x_R)$,
\begin{equation}
 [\Phi_\nu(g)]^d(c\mid x_R)=
 \E_{X\sim r_{\nu,g}}
       [h^d(c\mid X^{-d})\mid X_R=x_R].
 \label{eq:population_map}
\end{equation}
The right-hand side is a conditional expectation under the current
softened buffer, rather than an assumed target expectation.
The following theorem, which is a longer version of \cref{thm:fixed_point_convergence}, studies the behavior resulting from iteratively optimizing $\mathcal{L}_{\nu,g}$ and resetting $g$ to $\Phi_\nu(g)$.

\begin{theorem}[Convergence of the population iteration]
\label{thm:main_convergence}
Assume a finite product space, $p(x)>0$, positive categorical heads,
exact sampling, and exact population fitting for the problem \eqref{eq:masked_loss}. Assume that reveal orders $\sigma$ are sampled from $\rho$. Given an arbitrary categorical predictor $g_0$, define the sequence $(g_n)_{n \geq 0}$ such that $g_{n+1} = \Phi_\nu(g_n)$ for all $n \geq 0$. 
For any fixed order law $\rho$, define
\begin{equation} \label{eq:Delta_g}
\Delta(g)=\max_{R,d,x_R}\osc_c\log[g^d(c\mid x_R)/g_*^d(c\mid x_R)],
\end{equation}
where the maximum is over trained contexts and the oscillation $\osc_c f(c)$ is defined as $\max_c f(c) - \min_c f(c)$ for an arbitrary function $f$ over a finite set, and $g_*^d(c\mid x_R)=p(X^d=c\mid X_R=x_R)$ is the target head.
\begin{enumerate}[label=(\roman*),itemsep=3pt]
\item The target heads $g_*^d(c\mid x_R)=p(X^d=c\mid X_R=x_R)$
satisfy $\Phi_\nu(g_*)=g_*$ for every $\nu\in[0,1]$.
Full weighting gives $\Phi_1(g)=g_*$ for every positive initialization. 
\item 
Then,
\begin{equation}
 \Delta(\Phi_\nu(g))\leq2D(1-\nu)\Delta(g).
 \label{eq:anyorder_contraction}
\end{equation}
Thus a fixed $\nu>1-1/(2D)$ guarantees global convergence of this sequence to $g_*$ (since $g$ is defined over a finite sets, all norms are equivalent).
\item With one fixed reveal order (i.e. setting $\rho = \delta_{\sigma_0}$ in \eqref{eq:tempered_endpoint}) and training on its prefixes,
every fixed $\nu\in(0,1)$ gives global convergence of the sequence $(g_n)_{n \geq 0}$ to $g_*$ with error $\Delta (g_n) = O(k^{D-1}(1-\nu)^k)$ after $k$ outer iterations.
\item With the same fixed-order training scheme, the label-only ($\nu=0$)
iteration $g_{k+1}=\Phi_0(g_k)$ also converges globally to $g_*$ from
every positive initialization.
\end{enumerate}
\end{theorem}

The any-order threshold in \emph{(ii)} is a conservative sufficient condition: in general we are interested in taking values of $\nu$ substantially smaller than 1 to avoid high importance weight variance.
The fixed-order result in \emph{(iii)} holds for an arbitrary $\nu$, and covers an
ordered masked bridge. 

The following subsections prove the four statements: statements \emph{(i)} and \emph{(ii)} in \Cref{app:anyorder_proof}, statement \emph{(iii)} in \Cref{app:fixed_order_proof}, and statement \emph{(iv)} in \Cref{sec:label_only_convergence}.
and in \Cref{sec:ess_adaptation} we show that the ESS rule \eqref{eq:adaptive_tempering} supplies a uniformly positive exponent in
the fixed-order population setting.
In \Cref{app:local_proof}, we also show that tempering multiplies the Jacobian of the population
update at the target by $1-\nu$, yielding a complementary local
stability condition.
All of these results establish a target-directed population iteration: they do not assert that finite neural optimization is exact. In \Cref{sec:approximate_fitting} we explore how robust our arguments are to approximate fitting.

\subsection{Gibbs-label fitting does not amplify log-ratio error}

We begin by introducing a lemma that constitutes the backbone of our argument.
\begin{lemma} \label{lem:a_b_K}
Let $a,b : \mathcal{X} \to (0,+\infty)$ be arbitrary positive scalar-valued functions, with $\mathcal{X}$ a finite set, and $K$ be a nonnegative linear map with positive output, defined as
\begin{equation} \label{eq:K_a_def}
(K a)(y) = \sum_{x} K(y,x) a(x).
\end{equation}
Then,
\begin{equation}
 \min_x\frac{a(x)}{b(x)}
 \leq \frac{(Ka)(y)}{(Kb)(y)}
 \leq \max_x\frac{a(x)}{b(x)}.
 \label{eq:positive_average}
\end{equation}
\end{lemma}
\begin{proof}
These inequalities hold because $(Ka)(y) = \sum_x K(y,x) b(x) \frac{a(x)}{b(x)}$, which means that $(Ka)(y)/(Kb)(y)$ can be regarded as the weighted average of the terms $\frac{a(x)}{b(x)}$ with normalized weights proportional to $K(y,x)b(x)$.
\end{proof}

Based on this result, the following proposition, which we use in \Cref{app:anyorder_proof}, shows that a certain operator associated to Gibbs-label fitting is non-expansive.

\begin{proposition}[Non-expansiveness of Gibbs-label fitting]
\label{prop:fitting_nonexpansion}
Let $p$ and $r$ be strictly positive distributions on the finite
product space $\mathcal X$. Define the exact Gibbs-label fitting
operator by
\[
[\mathcal T(r)]^d(c\mid x_R)
=
\mathbb E_r\!\left[
    h^d(c\mid X^{-d})
    \,\middle|\,
    X_R=x_R
\right].
\]
Then
\begin{equation}
\max_{(R,d,x_R)\in\mathcal C_{\mathrm{train}}}
\operatorname{osc}_c
\log
\frac{[\mathcal T(r)]^d(c\mid x_R)}
     {g_*^d(c\mid x_R)}
\leq
\operatorname{osc}_x
\log\frac{r(x)}{p(x)}.
\label{eq:fitting_nonexpansion}
\end{equation}
\end{proposition}

\begin{proof}
Since by definition, $h^d(c\mid X^{-d})=\E_p[\one_{X^d=c}\mid X^{-d}]$, the target distribution $p$ is a fixed point of $\mathcal T$ as a consequence of the tower property of conditional expectation:
\begin{align} 
\begin{split} \label{eq:T_p}
 [\mathcal T(p)]^d(c\mid x_R)
 &=\E_p[h^d(c\mid X^{-d})\mid X_R=x_R] = \E_p[\one_{X^d=c}\mid X_R=x_R] \\ &= p^d( X^d = c\mid x_R) = g_*^d(c\mid x_R).
\end{split}
\end{align}

Fix a context $(R,d,u)$, where $d\notin R$ and $u\in\mathcal A^R$, 
and define the nonnegative linear map
\[
 (K_{R,d,u}a)(c)
 =
 \sum_{x:\,x_R=u}
 a(x)h^d(c\mid x^{-d}).
\]
Applying this map to an arbitrary $r$ gives
\[
 (K_{R,d,u}r)(c)
 =
 r_R(u)[\mathcal T(r)]^d(c\mid u).
\]
Applying it to $p$ and using equation \eqref{eq:T_p} yields
\[
 (K_{R,d,u}p)(c)
 =
 p_R(u)g_*^d(c\mid u).
\]
Hence,
\begin{align} \label{eq:ratio_Ks}
 \frac{(K_{R,d,u}r)(c)}
      {(K_{R,d,u}p)(c)}
 =
 \frac{r_R(u)}{p_R(u)}
 \frac{[\mathcal T(r)]^d(c\mid u)}
      {g_*^d(c\mid u)}.
\end{align}
The first factor on the right-hand side is constant in $c$ and
therefore disappears under $\osc_c\log$. 
By Lemma~\ref{lem:a_b_K},
\[
 \min_x\frac{r(x)}{p(x)}
 \leq
 \frac{(K_{R,d,u}r)(c)}
      {(K_{R,d,u}p)(c)}
 \leq
 \max_x\frac{r(x)}{p(x)}.
\]
Taking logarithms and then the oscillation over $c$, and plugging equation~\eqref{eq:ratio_Ks} (where the factor $r_R(u)/p_R(u)$ can be simplified because it does not depend on $c$), yields
\[
 \osc_c\log
 \frac{[\mathcal T(r)]^d(c\mid u)}
      {g_*^d(c\mid u)}
 \leq
 \osc_x\log\frac{r(x)}{p(x)}.
\]
Since this holds for every context $(R,d,u)$,
\begin{equation}
 \max_{R,d,x_R}
 \osc_c\log
 \frac{[\mathcal T(r)]^d(c\mid x_R)}
      {g_*^d(c\mid x_R)}
 \leq
 \osc_x\log\frac{r(x)}{p(x)}.
\end{equation}
This inequality, which we will leverage in the next subsection, shows that for any context, the oscillation of the log-ratio to the optimal categorical predictor is bounded by the log-ratio.
\end{proof}

\subsection{Fixed-point exactness and global contraction for arbitrary reveal orders: proof of \cref{thm:main_convergence}\emph{(i)--(ii)}}
\label{app:anyorder_proof}

\paragraph{Fixed point and full importance weighting.}
For every reveal order
$\sigma=(\sigma_1,\ldots,\sigma_D)$, the target heads satisfy the
chain-rule identity
\begin{align}
q_{g_*}(x\mid\sigma) =
\prod_{k=1}^D
g_*^{\sigma_k}
\left(
x^{\sigma_k}\mid x_{\sigma_{<k}}
\right) =
\prod_{k=1}^D
p\left(
X^{\sigma_k}=x^{\sigma_k}
\mid
X_{\sigma_{<k}}=x_{\sigma_{<k}}
\right)
=
p(x).
\end{align}
Consequently, substituting $g=g_*$ into
equation \eqref{eq:tempered_endpoint} gives $r_{\nu,g_*}(x,\sigma)= \rho(\sigma)p(x)$
for every $\nu\in[0,1]$. Its endpoint marginal is therefore $p$, and
the Gibbs-label identity gives
$\Phi_\nu(g_*) = \mathcal T(p) =
g_*$. Thus $g_*$ is a fixed point for every tempering exponent.

At $\nu=1$, the proposal contribution in equation
\eqref{eq:tempered_endpoint} disappears:
$P_{1,g}(x,\sigma)
=
\rho(\sigma)p(x)$
for every positive $g$. Hence
\[
\Phi_1(g)=\mathcal T(p)=g_*,
\]
which proves one-step recovery under full importance weighting.

\paragraph{Accumulation of head-wise error.}
Now fix an arbitrary positive predictor $g$ and write $\Delta=\Delta(g)$, where $\Delta(\cdot)$ is defined in \eqref{eq:tempered_endpoint}.
For every training-supported context $(R,d,x_R)$,
\begin{equation} \label{eq:equality_one}
\sum_c
g_*^d(c\mid x_R)
\frac{g^d(c\mid x_R)}
     {g_*^d(c\mid x_R)}
=
1.
\end{equation}
Therefore,
\begin{equation} \label{eq:min_1_max}
\min_c
\frac{g^d(c\mid x_R)}
     {g_*^d(c\mid x_R)}
\leq 1
\leq
\max_c
\frac{g^d(c\mid x_R)}
     {g_*^d(c\mid x_R)}.
\end{equation}
If we write $\ell(c) = \log \big( g^d(c\mid x_R) / g_*^d(c\mid x_R) \big)$, the definition of $\Delta(\cdot)$ and equation \eqref{eq:min_1_max} imply, respectively, that
\begin{align}
    \max_c \ell(c) - \min_c \ell(c) \leq \Delta, \qquad \min_c \ell(c) \leq 0 \leq \max_c \ell(c).
\end{align}
Combining these inequalities, we obtain that for all $c$,
\begin{equation} \label{eq:abs_log_ratio}
\left|
\log
\frac{g^d(c\mid x_R)}
     {g_*^d(c\mid x_R)}
\right|
\leq\Delta.
\end{equation}

For a reveal order $\sigma$, define the joint proposal error
\begin{equation}
a_\sigma(x)
=
\log\frac{q_g(x\mid\sigma)}{p(x)} =
\sum_{k=1}^D
\log
\frac{
g^{\sigma_k}
\left(
x^{\sigma_k}\mid x_{\sigma_{<k}}
\right)
}{
g_*^{\sigma_k}
\left(
x^{\sigma_k}\mid x_{\sigma_{<k}}
\right)
}.
\end{equation}
Every factor queried by the reveal policy is training-supported, so
each summand has absolute value at most $\Delta$. Hence
\begin{equation}
-D\Delta
\leq
a_\sigma(x)
\leq
D\Delta
\qquad
\text{for every }(x,\sigma).
\label{eq:order_proposal_error}
\end{equation}

\paragraph{Tempering and endpoint error.}
Using the definition of $a_\sigma$, the tempered endpoint ratio can be
written as
\begin{equation} \label{eq:r_nu_g_p}
\frac{r_{\nu,g}(x)}{p(x)}
=
C_{\nu,g}^{-1}
\sum_\sigma
\rho(\sigma)
\exp\!\left\{
    (1-\nu)a_\sigma(x)
\right\}.
\end{equation}
By \cref{eq:order_proposal_error},
\begin{equation} \label{eq:exponent}
e^{-D(1-\nu)\Delta}
\leq
\sum_\sigma
\rho(\sigma)
e^{(1-\nu)a_\sigma(x)}
\leq
e^{D(1-\nu)\Delta}.
\end{equation}
Thus, plugging \eqref{eq:r_nu_g_p} into \eqref{eq:exponent} and taking logarithms and the oscillation yields
\begin{equation}
\operatorname{osc}_x
\log\frac{r_{\nu,g}(x)}{p(x)}
\leq
2D(1-\nu)\Delta(g).
\label{eq:tempered_endpoint_error}
\end{equation}
Note that in this equation, the normalizing constant $C_{\nu,g}^{-1}$ is independent of $x$ and
therefore disappears under log-ratio oscillation.

Since $\Phi_\nu(g)=\mathcal T(r_{\nu,g})$ by the definitions of $\Phi_\nu$ and $\mathcal T$,
\cref{prop:fitting_nonexpansion,eq:tempered_endpoint_error} give
\begin{equation}
\Delta(\Phi_\nu(g))
\leq
\operatorname{osc}_x
\log\frac{r_{\nu,g}(x)}{p(x)} \leq
2D(1-\nu)\Delta(g),
\end{equation}
which proves \cref{eq:anyorder_contraction}.

\paragraph{Iteration and convergence.}
Let $c_\nu=2D(1-\nu)$.
If $c_\nu<1$, equivalently $\nu>1-\frac{1}{2D}$,
then repeated application of the contraction gives
\[
\Delta(g_k)
\leq
c_\nu^k\Delta(g_0)
\longrightarrow0.
\]
Because the heads are normalized categorical distributions, vanishing
log-ratio oscillation implies
\[
g_k^d(\cdot\mid x_R)
\longrightarrow
g_*^d(\cdot\mid x_R)
\]
at every training-supported context. Consequently,
$q_{g_k}(\cdot\mid\sigma)\to p$ for every order in the support of
$\rho$, and the corresponding marginal proposal also converges to
$p$. The contraction also establishes uniqueness in this regime. Indeed,
if $g=\Phi_\nu(g)$, then
$\Delta(g)
\leq
c_\nu\Delta(g)$,
and $c_\nu<1$ implies $\Delta(g)=0$, hence $g=g_*$ on all
training-supported contexts.

For varying exponents $\nu_k$, define $c_k=2D(1-\nu_k)$.
The same argument gives $\Delta(g_k)
\leq
\big(\prod_{j=0}^{k-1}c_j\big)\Delta(g_0)$.
It is therefore sufficient that
$\prod_{j=0}^{k-1}c_j\longrightarrow0$.
In particular, convergence holds if
$\sup_k c_k<1$.

The factor $2D$ is a worst-case bound obtained by adding the maximum
possible error across all $D$ reveals. It is a sufficient theoretical
bound, not a necessary condition, and can be conservative in high
dimensions. It should therefore not be interpreted as prescribing the
empirically optimal tempering exponent.

\subsection{Global convergence for a fixed reveal order:
proof of \cref{thm:main_convergence}\emph{(iii)}}
\label{app:fixed_order_proof}

Relabel the coordinates so that the fixed reveal order is
$1,\ldots,D$. At outer iteration $k$, write the corresponding
order-conditional proposal as
\[
q_k(x)
=
\prod_{d=1}^D q_k^d(x^d\mid x^{<d}),
\]
where $q_k^d$ denotes the current categorical head evaluated at the
prefix $x^{<d}=(x^1,\ldots,x^{d-1})$. For a tempering exponent
$\nu_k$, the effective endpoint distribution is
\[
r_k(x)
=
\frac{q_k(x)^{1-\nu_k}p(x)^{\nu_k}}
     {\sum_{x'}q_k(x')^{1-\nu_k}p(x')^{\nu_k}}.
\]

\paragraph{Population update at a prefix.}
Fix a coordinate $d$ and a prefix $u\in\mathcal A^{d-1}$. Write
$y=X^{>d}$ for the suffix following site $d$. Exact Gibbs-label
fitting gives
\begin{equation}
q_{k+1}^d(c\mid u)
=
\sum_y
r_k(X^{>d}=y\mid X^{<d}=u)
h^d(c\mid u,y).
\label{eq:ar_update}
\end{equation}
The value of $X^d$ is absent from the averaging distribution because
the Gibbs label $h^d(c\mid X^{-d})$ depends only on
$X^{-d}=(u,y)$.

Under the target distribution, the same Gibbs-label identity gives
\begin{equation}
p(X^d=c\mid X^{<d}=u)
=
\sum_y
p(X^{>d}=y\mid X^{<d}=u)
h^d(c\mid u,y).
\label{eq:target_ar_update}
\end{equation}
Thus the updated-head error is controlled by the discrepancy between
the suffix laws of $r_k$ and $p$.

\paragraph{Head-wise error.}
Define
\[
e_k^d
=
\max_u
\operatorname{osc}_c
\log
\frac{
q_k^d(c\mid u)
}{
p(X^d=c\mid X^{<d}=u)
}.
\]
We claim that
\begin{equation}
e_{k+1}^D=0,
\qquad
e_{k+1}^d
\leq
(1-\nu_k)
\left(
e_k^d+2\sum_{j>d}e_k^j
\right),
\quad d<D.
\label{eq:ar_tempered_recursion}
\end{equation}

For $d=D$, the suffix is empty. Hence
equation \eqref{eq:ar_update} reduces directly to
\[
q_{k+1}^D(c\mid u)
=
h^D(c\mid u)
=
p(X^D=c\mid X^{<D}=u),
\]
which proves $e_{k+1}^D=0$.

Now fix $d<D$ and a prefix $u$. For compactness, let
\[
\mu_k(c)=q_k^d(c\mid u),
\qquad
\mu_*(c)=p(X^d=c\mid X^{<d}=u),
\]
and define
\[
v_k(c,y)
=
\frac{
\mu_k(c)q_k(y\mid u,c)
}{
\mu_*(c)p(y\mid u,c)
}.
\]
Conditioning the softened distribution \(r_k\) on the prefix \(u\)
cancels all proposal and target factors preceding site \(d\). Therefore,
\begin{equation}
\frac{r_k(c,y\mid u)}{p(c,y\mid u)}
\propto
v_k(c,y)^{1-\nu_k}.
\label{eq:fixed_order_conditional_ratio}
\end{equation}

Marginalizing $c$ is a positive linear operation. Applying
Lemma \ref{lem:a_b_K} to \cref{eq:fixed_order_conditional_ratio} gives
\begin{align}
&\min_c \frac{r_k(c,y\mid u)}{p(c,y\mid u)} \leq \frac{r_k(y\mid u)}{p(y\mid u)} \leq \max_c \frac{r_k(c,y\mid u)}{p(c,y\mid u)}, \\
&\implies \operatorname{osc}_y
\log
\frac{r_k(y\mid u)}{p(y\mid u)}
\leq
(1-\nu_k)
\operatorname{osc}_{c,y}\log v_k(c,y).
\label{eq:suffix_error_bound}
\end{align}
Likewise, equations \eqref{eq:ar_update} and \eqref{eq:target_ar_update} apply the same
positive Gibbs-label kernel to the two suffix distributions. Another
application of Lemma \ref{lem:a_b_K} therefore gives
\begin{align}
&\min_y \frac{r_k(y\mid u)}{p(y\mid u)} \leq \frac{
q_{k+1}^d(c\mid u)}{p(X^d=c\mid X^{<d}=u)} \leq \max_y \frac{r_k(y\mid u)}{p(y\mid u)},
\\ &\implies \operatorname{osc}_c
\log
\frac{
q_{k+1}^d(c\mid u)
}{
p(X^d=c\mid X^{<d}=u)
}
\leq
\operatorname{osc}_y
\log
\frac{r_k(y\mid u)}{p(y\mid u)}.
\label{eq:head_from_suffix_error}
\end{align}

It remains to control the right-hand side. The autoregressive
factorizations of $q_k$ and $p$ give
\begin{equation}
\log v_k(c,y)
=
\log\frac{\mu_k(c)}{\mu_*(c)}
+
\sum_{j>d}
\log
\frac{
q_k^j\left(
y^j\mid u,c,y^{d+1:j-1}
\right)
}{
p\left(
X^j=y^j
\mid
X^{<j}=(u,c,y^{d+1:j-1})
\right)
}.
\label{eq:fixed_order_error_decomposition}
\end{equation}

The first term in \eqref{eq:fixed_order_error_decomposition} has
oscillation at most $e_k^d$ by definition. For every later head $j>d$ and every
fixed prefix, let  
\begin{equation}
\ell_{y^{d+1:j-1}}(y^j) = \log
\frac{
q_k^j\left(
y^j\mid u,c,y^{d+1:j-1}
\right)
}{
p\left(
X^j=y^j
\mid
X^{<j}=(u,c,y^{d+1:j-1})
\right)
}
\end{equation}
We can mirror the normalization argument of equations \eqref{eq:equality_one}-\eqref{eq:abs_log_ratio} to conclude that $|\ell_{y^{d+1:j-1}}(c)| \leq \osc_c \ell_{y^{d+1:j-1}}(c) \leq e_k^j$, since $e_k^j$ is a uniform bound of such oscillations over all prefixes. Hence, we can further bound the oscillation of $\ell_{y^{d+1:j-1}}(c)$ with respect to both $c$ and $y^{d+1:j-1}$:
\begin{equation}
    \osc_{c,y^{d+1:j-1}} \ell_{y^{d+1:j-1}}(c) \leq 2 |\ell_{y^{d+1:j-1}}(c)| \leq 2 e_k^j.
\end{equation}

By subadditivity of
oscillation, we conclude that 
\[
\operatorname{osc}_{c,y}\log v_k(c,y) 
\leq
e_k^d+2\sum_{j>d}e_k^j.
\]
Combining this bound with
\eqref{eq:suffix_error_bound}, \eqref{eq:head_from_suffix_error} and maximizing
over $u$ proves \cref{eq:ar_tempered_recursion}.

\paragraph{Triangular iteration.}
Let $e_k=(e_k^1,\ldots,e_k^D)^\top$
and define the upper-triangular matrix $H\in\mathbb R^{D\times D}$ by
\[
H_{dj}
=
\begin{cases}
1, & j=d,\\
2, & j>d,\\
0, & j<d.
\end{cases}
\]
Then equation \eqref{eq:ar_tempered_recursion} implies, componentwise, $e_{k+1} \leq (1-\nu_k)He_k$.
Iterating gives
\begin{equation}
e_k
\leq
\bigg\{
\prod_{i=0}^{k-1}(1-\nu_i)
\bigg\}
H^k e_0.
\label{eq:fixed_order_matrix_bound}
\end{equation}

Write $H=I+N$, where $N$ is strictly upper triangular. Since
$N^D=0$,
$
H^k
=
(I+N)^k
=
\sum_{m=0}^{D-1}
\binom{k}{m}N^m$.
Thus $\lVert H^k\rVert=O(k^{D-1})$ in any matrix norm.

For a fixed exponent $\nu\in(0,1)$,
\cref{eq:fixed_order_matrix_bound} therefore gives
\[
\max_d e_k^d
=
O\!\left(
k^{D-1}(1-\nu)^k
\right).
\]
The exponential factor dominates the polynomial factor, so
$e_k^d\to0$ for every site $d$.

Because both $q_k^d(\cdot\mid u)$ and
$p(X^d=\cdot\mid X^{<d}=u)$ are normalized categorical
distributions, vanishing log-ratio oscillation implies $q_k^d(\cdot\mid u)
\longrightarrow
p(X^d=\cdot\mid X^{<d}=u)$
at every prefix. Consequently,
\[
q_k(x)
=
\prod_{d=1}^Dq_k^d(x^d\mid x^{<d})
\longrightarrow
p(x),
\]
which proves global convergence for every positive initialization.

More generally, if the exponents vary but satisfy
$\nu_k\geq\nu_{\min}>0$,
then $\prod_{i=0}^{k-1}(1-\nu_i)
\leq
(1-\nu_{\min})^k$,
and the same argument gives $\max_d e_k^d
=
O\!\left(
k^{D-1}(1-\nu_{\min})^k
\right)$.

\subsection{Global convergence of label-only fixed-order fitting: proof of \cref{thm:main_convergence}\emph{(iv)}}
\label{sec:label_only_convergence}

The fixed-order proof in \cref{app:fixed_order_proof} uses the
exponential factor $(1-\nu)^k$ and therefore does not cover the
endpoint $\nu=0$. Nevertheless, exact label-only fitting also converges
globally under the same strict-positivity assumptions. The argument
uses contraction of the Gibbs transition kernels rather than tempered
importance correction.

\begin{proposition}[Fixed-order convergence without importance weighting]
\label{prop:label_only_convergence}
Consider one fixed reveal order and exact population fitting with
$\nu=0$. If $p$ is strictly positive on the finite product space, then
the fixed-order heads converge globally to the target heads.
\end{proposition}

\begin{proof}
Relabel the fixed order as $1,\ldots,D$. Strict positivity of $p$ on
the finite state space implies that the Gibbs labels are uniformly
positive:
\[
\epsilon
=
\min_{d,x^{-d},c}
h^d(c\mid x^{-d})
>0.
\]

Fix a site $d$, a prefix $u=x^{<d}$, and a possible value
$a\in\mathcal A$ at site $d$. Define
\begin{equation}
K_{d,k}(a,c\mid u)
=
\sum_y
q_k(X^{>d}=y\mid X^{<d}=u,X^d=a)
h^d(c\mid u,y).
\label{eq:label_only_kernel}
\end{equation}
This kernel first samples the suffix from the current proposal,
conditional on the prefix and current symbol \(a\), and then resamples
site \(d\) from its target Gibbs conditional.

When $\nu=0$, the buffer distribution is $q_k$. Marginalizing the
current value \(a\) in the population update gives
\begin{equation}
q_{k+1}^d(c\mid u)
=
\sum_{a\in\mathcal A}
q_k^d(a\mid u)K_{d,k}(a,c\mid u).
\label{eq:label_only_head_update}
\end{equation}
Thus the current head is updated by applying the Markov kernel
$K_{d,k}$.

Every entry of $K_{d,k}$ is at least $\epsilon$, because it is an
average of Gibbs labels: $K_{d,k}(a,c\mid u)\geq\epsilon$.
Its total-variation contraction coefficient is therefore uniformly
bounded by
$\kappa
=
1-N\epsilon
<1$,
where \(N=|\mathcal A|\). In particular, for any categorical
distributions \(\mu\) and \(\widetilde\mu\),
\begin{equation}
\left\|
\mu K_{d,k}
-
\widetilde\mu K_{d,k}
\right\|_{\mathrm{TV}}
\leq
\kappa
\left\|
\mu-\widetilde\mu
\right\|_{\mathrm{TV}}.
\label{eq:gibbs_kernel_contraction}
\end{equation}

We now proceed backward through the reveal order. For $d=D$, the suffix
is empty, so
\[
q_{k+1}^D(c\mid u)
=
h^D(c\mid u)
=
p(X^D=c\mid X^{<D}=u).
\]
Hence the last head becomes exact after one iteration.

Suppose inductively that all later heads \(j>d\) converge to their
target conditionals. Their product then satisfies
\[
q_k(X^{>d}=y\mid u,a)
\longrightarrow
p(X^{>d}=y\mid u,a)
\]
uniformly over the finitely many \((u,a,y)\). Consequently,
$K_{d,k}$ converges to the target two-block Gibbs kernel
\[
K_d^*(a,c\mid u)
=
\sum_y
p(X^{>d}=y\mid u,a)
h^d(c\mid u,y).
\]

The target head $\pi_d(c\mid u) := p(X^d=c\mid X^{<d}=u)$ is stationary for this limiting kernel. Indeed,
\begin{equation}
\sum_a
\pi_d(a\mid u)K_d^*(a,c\mid u)
=
\sum_{a,y}
p(a,y\mid u)h^d(c\mid u,y) =
\sum_y
p(y\mid u)p(c\mid u,y) 
=
\pi_d(c\mid u).
\label{eq:target_gibbs_stationarity}
\end{equation}

Define the maximum head error
\[
A_k^d
=
\max_u
\left\|
q_k^d(\cdot\mid u)
-
\pi_d(\cdot\mid u)
\right\|_{\mathrm{TV}}.
\]
Using \cref{eq:label_only_head_update,eq:gibbs_kernel_contraction} and
adding and subtracting
$\pi_d(\cdot\mid u)K_{d,k}$ gives
\begin{equation}
A_{k+1}^d
\leq
\kappa A_k^d+B_k^d,
\label{eq:label_only_error_recursion}
\end{equation}
where
\[
B_k^d
=
\max_u
\left\|
\pi_d(\cdot\mid u)K_{d,k}(\cdot,\cdot\mid u)
-
\pi_d(\cdot\mid u)K_d^*(\cdot,\cdot\mid u)
\right\|_{\mathrm{TV}}.
\]
The induction hypothesis implies \(B_k^d\to0\). Since
\(\kappa<1\), unrolling \cref{eq:label_only_error_recursion} gives
\(A_k^d\to0\). Backward induction over
\(d=D,D-1,\ldots,1\) proves convergence of every head and hence of the
fixed-order joint proposal to \(p\).
\end{proof}

This result shows that tempered importance weighting is not the only
source of contraction. Even at $\nu=0$, repeated Gibbs-label fitting
propagates the correct target conditionals backward through a fixed
reveal order. Positive tempering strengthens this mechanism by adding
the explicit factor \(1-\nu\) in
\cref{eq:ar_tempered_recursion}.

\subsection{ESS adaptation preserves fixed-order convergence} \label{sec:ess_adaptation}

The fixed-order result remains valid when the tempering exponent is
selected adaptively by the empirical ESS rule, provided that each
selected exponent is followed by the corresponding exact population
fit. The key point is that the selected exponents remain uniformly
bounded away from zero.

\begin{proposition}[Uniform positivity of ESS-selected exponents]
\label{prop:ess_fixed_order}
Consider the fixed-order population iteration, and suppose that at
iteration $k$ the exponent $\nu_k$ is selected from a finite batch by
\cref{eq:adaptive_tempering}. Then, after the first population update,
there exists a constant $\nu_{\min}>0$, independent of $k$ and of the
sampled batch, such that
\[
\nu_k\geq\nu_{\min}.
\]
Consequently, if each update is the exact population update
$g_{k+1}=\Phi_{\nu_k}(g_k)$, then $g_k$ converges to $g_*$.
\end{proposition}

\begin{proof}
Because $p$ is strictly positive on a finite product space, every
single-site Gibbs label is uniformly positive. Define
\[
\epsilon
=
\min_{d,x^{-d},c}
h^d(c\mid x^{-d})
>0.
\]
Every fitted categorical head is a conditional average of Gibbs
labels. Hence, after the first population update,
$g_k^d(c\mid x^{<d})\geq\epsilon$
for every site, prefix, and symbol. The corresponding fixed-order
proposal therefore satisfies
\begin{equation}
q_k(x)
=
\prod_{d=1}^Dg_k^d(x^d\mid x^{<d})
\geq
\epsilon^D,
\qquad k\geq1.
\label{eq:fixed_order_proposal_lower_bound}
\end{equation}

The unnormalized importance weight is
\[
W_k(x)=\frac{e^{-U(x)}}{q_k(x)},
\]
so its log is
\[
\lambda_k(x)=-U(x)-\log q_k(x).
\]
Since $q_k(x)\leq1$ and \cref{eq:fixed_order_proposal_lower_bound}
gives $q_k(x)\geq\epsilon^D$,
\[
0\leq-\log q_k(x)\leq D\log(1/\epsilon).
\]
It follows that the full log-weight range is uniformly bounded:
\begin{align}
L_k :=
\operatorname{osc}_x\lambda_k(x) \leq
U_{\max}-U_{\min}
+
D\log(1/\epsilon)
=:L_{\max}<\infty.
\label{eq:uniform_log_weight_range}
\end{align}

For any finite batch drawn at iteration $k$, its empirical log-weight
range is at most $L_k$, and therefore at most $L_{\max}$. By
Lemma \ref{lem:ess_exponent_lower}, the exponent selected by the ESS rule satisfies
\[
\nu_k
\geq
\min\left\{
1,\frac{\log(1/\xi)}{2L_{\max}}
\right\}
=:\nu_{\min}>0,
\]
with the convention $\nu_{\min}=1$ if $L_{\max}=0$.

Thus the adaptively selected exponents satisfy
$\inf_k\nu_k\geq\nu_{\min}>0$. The fixed-order matrix bound
\cref{eq:fixed_order_matrix_bound} consequently gives
\[
\max_d e_k^d
=
O\!\left(
k^{D-1}(1-\nu_{\min})^k
\right)
\longrightarrow0.
\]
Hence all fixed-order heads, and therefore the joint proposal $q_k$,
converge to the target.
\end{proof}

The preceding result only requires the selected exponents to remain
positive. In fact, the ESS rule eventually selects full weighting in
the exact population setting. Since $q_k\to p$ on the finite state
space,
\[
\operatorname{osc}_x
\log\frac{e^{-U(x)}}{q_k(x)}
=
\operatorname{osc}_x
\log\frac{p(x)}{q_k(x)}
\longrightarrow0.
\]
By \cref{eq:ess_lower}, once this range is small enough that
\[
B\exp(-2L_k)\geq\xi B,
\]
every possible batch satisfies the ESS threshold at $\nu=1$. The rule
then selects $\nu_k=1$, and the next exact population update recovers
$g_*$ in one step.

This conclusion concerns an idealized hybrid procedure: a finite batch
is used only to select $\nu_k$, while $\Phi_{\nu_k}(g_k)$ is still
evaluated as an exact population fit. If the same finite batch is also
used to approximate the fitting expectation, additional sampling and
optimization errors arise and are not covered by this proposition.

\begin{lemma}[Lower bound for ESS-selected tempering]
\label{lem:ess_exponent_lower}
Let $W_i=e^{\lambda_i}>0$, $i=1,\ldots,B$, be a finite collection
of importance weights, and define their log-weight range by $L
=
\max_i\lambda_i-\min_i\lambda_i$.
For $\nu\in[0,1]$, let
$\operatorname{ESS}(\nu)
=\big(\sum_{i=1}^B e^{\nu\lambda_i}\big)^2/ \sum_{i=1}^B e^{2\nu\lambda_i}$.
Then
\begin{equation}
\operatorname{ESS}(\nu)
\geq
B e^{-2\nu L}.
\label{eq:ess_lower}
\end{equation}
Consequently, for any threshold $\xi\in(0,1)$, the largest feasible
tempering exponent
$\nu^\star
=
\max\left\{
    \nu\in[0,1]:
    \operatorname{ESS}(\nu)\geq\xi B
\right\}$
satisfies
\begin{equation}
\nu^\star
\geq
\min\left\{
    1,
    \frac{\log(1/\xi)}{2L}
\right\}
\label{eq:ess_exponent_lower}
\end{equation}
when $L>0$. If $L=0$, then
$\operatorname{ESS}(\nu)=B$ for every $\nu\in[0,1]$ and hence
$\nu^\star=1$.
\end{lemma}

\begin{proof}
Let $\lambda_{\min}=\min_i\lambda_i$,
$\lambda_{\max}=\max_i\lambda_i$.
The numerator of $\operatorname{ESS}(\nu)$ satisfies $\big(\sum_{i=1}^B e^{\nu\lambda_i}\big)^2
\geq
B^2 e^{2\nu\lambda_{\min}}$,
while its denominator satisfies
$\sum_{i=1}^B e^{2\nu\lambda_i}
\leq
B e^{2\nu\lambda_{\max}}$.
Therefore,
\[
\operatorname{ESS}(\nu)
\geq
\frac{
    B^2 e^{2\nu\lambda_{\min}}
}{
    B e^{2\nu\lambda_{\max}}
}
=
B e^{-2\nu(\lambda_{\max}-\lambda_{\min})}
=
B e^{-2\nu L},
\]
which proves \cref{eq:ess_lower}.

Suppose $L>0$ and define $\bar\nu
= \min\big\{ 1, \frac{\log(1/\xi)}{2L} \big\}$.
By construction, $e^{-2\bar\nu L}\geq\xi$,
and hence \cref{eq:ess_lower} gives
$\operatorname{ESS}(\bar\nu)\geq\xi B$.
Thus $\bar\nu$ is feasible. Since $\nu^\star$ is the largest feasible
exponent, $\nu^\star\geq\bar\nu$,
which proves \cref{eq:ess_exponent_lower}.

If $L=0$, all log-weights, and therefore all weights, are equal. It
follows directly that $\operatorname{ESS}(\nu)=B$ for every
$\nu\in[0,1]$, so the largest feasible exponent is $\nu^\star=1$.
\end{proof}

\subsection{Tempering scales the local Jacobian}
\label{app:local_proof}

The global contraction condition in
\cref{eq:anyorder_contraction} is sufficient but can be conservative.
We now study the linearized population update near the target. This
local result applies to any fixed, value-independent reveal-order law
$\rho$ and is independent of the fixed-order triangular argument.

\begin{proposition}[Local Jacobian scaling]
\label{prop:jacobian_scaling}
At the target fixed point,
\begin{equation}
D\Phi_\nu(g_*)
=
(1-\nu)D\Phi_0(g_*),
\qquad
\nu\in[0,1].
\label{eq:jacobian_scaling}
\end{equation}
Consequently, every eigenvalue of $D\Phi_\nu(g_*)$ is equal to
$1-\nu$ times the corresponding eigenvalue of $D\Phi_0(g_*)$.
\end{proposition}

\begin{proof}
Let $\mathcal R_\nu(g)=r_{\nu,g}$
denote the map from the current categorical heads to the tempered
endpoint distribution. The population update decomposes as $\Phi_\nu = \mathcal T\circ\mathcal R_\nu$,
where $\mathcal T$ is the exact Gibbs-label fitting operator from
\cref{prop:fitting_nonexpansion}.

Choose independent coordinates for the categorical simplices and let
$v$ be a tangent perturbation of $g_*$. Set
$g_\epsilon=g_*+\epsilon v$
for sufficiently small $\epsilon$. For each reveal order $\sigma$,
define
$\dot q_\sigma(x)
=
\left.
\frac{\dd}{\dd\epsilon}
q_{g_\epsilon}(x\mid\sigma)
\right|_{\epsilon=0}$.
Because $q_{g_\epsilon}(\cdot\mid\sigma)$ is normalized for every
$\epsilon$,
\begin{equation}
\sum_x\dot q_\sigma(x)=0.
\label{eq:q_tangent_normalization}
\end{equation}

Consider the unnormalized endpoint law
\[
\widetilde r_{\nu,g}(x)
=
p(x)^\nu
\sum_\sigma
\rho(\sigma)
q_g(x\mid\sigma)^{1-\nu}.
\]
At $g=g_*$, the target chain rule gives
$q_{g_*}(x\mid\sigma)=p(x)$
for every $\sigma$, and hence
$\widetilde r_{\nu,g_*}(x)=p(x)$.

Differentiating in the direction $v$ gives
\begin{equation}
\left.
\frac{\dd}{\dd\epsilon}
\widetilde r_{\nu,g_\epsilon}(x)
\right|_{\epsilon=0} =
p(x)^\nu
\sum_\sigma\rho(\sigma)
(1-\nu)
p(x)^{-\nu}
\dot q_\sigma(x) =
(1-\nu)
\sum_\sigma\rho(\sigma)\dot q_\sigma(x).
\label{eq:tempered_endpoint_derivative}
\end{equation}
The derivative of the normalizing constant vanishes, since
\begin{equation}
\sum_x
\left.
\frac{\dd}{\dd\epsilon}
\widetilde r_{\nu,g_\epsilon}(x)
\right|_{\epsilon=0}
=
(1-\nu)
\sum_\sigma\rho(\sigma)
\sum_x\dot q_\sigma(x) =0
\end{equation}
by equation \eqref{eq:q_tangent_normalization}. Therefore normalization does not
contribute an additional first-order term at $g_*$, and
\[
D\mathcal R_\nu(g_*)[v]
=
(1-\nu)
\sum_\sigma\rho(\sigma)\dot q_\sigma.
\]
For $\nu=0$, the same calculation gives
$D\mathcal R_0(g_*)[v]
=
\sum_\sigma\rho(\sigma)\dot q_\sigma$.
Thus
\begin{equation}
D\mathcal R_\nu(g_*)
=
(1-\nu)D\mathcal R_0(g_*).
\label{eq:endpoint_jacobian_scaling}
\end{equation}

On the interior of the finite probability simplex, the fitting map
$\mathcal T$ is differentiable. Moreover,
$\mathcal R_\nu(g_*)=p$
for every $\nu$. Applying the chain rule and
equation \eqref{eq:endpoint_jacobian_scaling} yields
\begin{equation}
D\Phi_\nu(g_*)
=
D\mathcal T(p)\,D\mathcal R_\nu(g_*) =
(1-\nu)
D\mathcal T(p)\,D\mathcal R_0(g_*) =
(1-\nu)D\Phi_0(g_*),
\end{equation}
which proves equation \eqref{eq:jacobian_scaling}.
\end{proof}

Let $s_0
=
\operatorname{spr}\!\left(D\Phi_0(g_*)\right)$
be the spectral radius of the unweighted population update. By
Proposition \ref{prop:jacobian_scaling},
$\operatorname{spr}\!\left(D\Phi_\nu(g_*)\right)
=
(1-\nu)s_0$.
Therefore,
\begin{equation}
(1-\nu)s_0<1
\label{eq:local_stability_condition}
\end{equation}
is sufficient for local asymptotic stability of $g_*$. In other
words, under this condition, every sufficiently small perturbation of
the target heads converges back to $g_*$ under repeated population
updates.

This local condition can be substantially weaker than the global
sufficient condition
$2D(1-\nu)<1$
from equation \eqref{eq:anyorder_contraction}. The global result controls every
positive initialization using a worst-case accumulation of error
across all $D$ coordinates, whereas
equation \eqref{eq:local_stability_condition} describes only the dynamics of
small perturbations around $g_*$.

At $\nu=1$, the Jacobian vanishes. This is consistent with the stronger
global fact that full importance weighting maps every positive
initialization to $g_*$ in a single exact population update. Thus the
fixed-point identity establishes that the target is correct, while
\cref{eq:jacobian_scaling} quantifies how tempering improves its local
stability.

\subsection{Robustness to approximate population fitting}
\label{sec:approximate_fitting}

The preceding convergence results assume unrestricted categorical
heads and exact population minimization. In the arbitrary-order
contraction regime, bounded fitting errors can be incorporated through
a standard perturbation argument.

Let $\widehat\Phi_k(g)$ denote the model produced by the implemented
fit at iteration \(k\), and suppose its error relative to the exact
population update satisfies
\begin{equation}
\delta_k
=
\max_{(R,d,x_R)\in\mathcal C_{\mathrm{train}}}
\operatorname{osc}_c
\log
\frac{
[\widehat\Phi_k(g_k)]^d(c\mid x_R)
}{
[\Phi_\nu(g_k)]^d(c\mid x_R)
}.
\label{eq:approximate_fitting_error}
\end{equation}
If $g_{k+1}=\widehat\Phi_k(g_k)$,
then subadditivity of log-ratio oscillation and
\cref{eq:anyorder_contraction} give
\begin{equation}
\Delta(g_{k+1})
\leq
c_\nu\Delta(g_k)+\delta_k,
\qquad
c_\nu=2D(1-\nu).
\label{eq:perturbed_contraction}
\end{equation}

Assume \(c_\nu<1\). Iterating \cref{eq:perturbed_contraction} yields
\begin{equation}
\Delta(g_k)
\leq
c_\nu^k\Delta(g_0)
+
\sum_{j=0}^{k-1}
c_\nu^{k-1-j}\delta_j.
\label{eq:perturbed_contraction_unrolled}
\end{equation}
Consequently:

\begin{itemize}[leftmargin=*]
\item If $\delta_k\to0$ for every $k$, then $\Delta(g_k)\to0$.
\item If $\delta_k\leq\overline\delta$ for every $k$, then
\[
\Delta(g_k)
\leq
c_\nu^k\Delta(g_0)
+
\frac{1-c_\nu^k}{1-c_\nu}\,\overline\delta,
\]
and hence
\[
\limsup_{k\to\infty}\Delta(g_k)
\leq
\frac{\overline\delta}{1-c_\nu}.
\]
\end{itemize}

Thus vanishing fitting errors preserve convergence, while persistent
bounded errors yield convergence only to a controlled neighborhood of
the target. This is a conditional robustness statement: it does not
establish that any particular neural architecture, finite replay
buffer, or optimizer achieves a prescribed sequence \(\delta_k\).

Both this argument and
\cref{prop:label_only_convergence} rely on strict positivity. Hard
constraints can produce zero Gibbs labels or inaccessible contexts, in
which case the uniform kernel contraction and log-ratio metrics used
above may fail and require a separate analysis.

\section{General discrete interpolants}
\label{app:general}

We recap and prove the construction of general forward and backward CTMC that we introduced in \Cref{sec:general_interpolants}.
We will prove the construction with  $  p_t^d(\cdot \mid x^{-d})$ and $\phi_t^d( \cdot \mid x^{-d})$ in  \Cref{eq:m-difference,eq:flux} minimizes the kinetic energy in \Cref{app:kinetic}, and derive the training objective \Cref{app:learn}.

\subsection{Minimum kinetic energy local  flux}\label{app:kinetic}

We aim to find a flux to minimize the following kinetic energy

\begin{equation}
\mathcal E_t(j)=
\frac12\sum_d\sum_{x^{-d}}p_t(x^{-d})\sum_a \sum_{c \neq a}
\frac{j_t^d(a,c\mid x^{-d})^2}
     {p_t(x^d=a|x^{-d}) p_t(x^d=c|x^{-d}) }
\end{equation}
 subject to the Kolmogorov constraint:
 \begin{align}
       \partial_t^{(d)}p_t(x^d=a\mid x^{-d})
    = \sum_{c\neq a} j^d_t(c, a |x^{-d}), \quad \forall a, d
 \end{align}
\begin{proof}

We solve the constrained minimization problem by introducing Lagrange multipliers:
\begin{align}
\mathcal{L}(j,\lambda)
=
\sum_d \sum_{x^{-d}} p_t(x^{-d}) \sum_a
\Bigg[
&\frac12 \sum_{c\neq a}
\frac{j_t^d(a,c\mid x^{-d})^2}
     {p_t(x^d=a\mid x^{-d})\,
      p_t(x^d=c\mid x^{-d})}
\\
& 
+\lambda_t^d(a\mid x^{-d})
\left(
\partial_t^{(d)}p_t(x^d=a\mid x^{-d})
-\sum_{c\neq a}j_t^d(c,a\mid x^{-d})
\right)
\Bigg].\nonumber
\end{align}
Therefore, taking the variation, we have
\begin{align}
&\frac{\delta \mathcal L}
{\delta j_t^d(a,c\mid x^{-d})}\nonumber
\\=
&p_t(x^{-d})
\left[
\frac{ j_t^d(a,c\mid x^{-d})}
{p_t(x^d=a\mid x^{-d})\,
 p_t(x^d=c\mid x^{-d})}
+\lambda_t^d(a\mid x^{-d})
-\lambda_t^d(c\mid x^{-d})
\right]
=0.
\end{align}
Hence
\begin{align}
j_t^d(a,c\mid x^{-d})
=
 \,
p_t(x^d=a\mid x^{-d})\,
p_t(x^d=c\mid x^{-d})
\left[
\lambda_t^d(c\mid x^{-d})
-\lambda_t^d(a\mid x^{-d})
\right].
\end{align} 

Now, plugging this into the Kolmogorov constraint, we obtain
\begin{align}
\partial_t^{(d)}p_t(x^d=a\mid x^{-d})
&=
p_t(x^d=a\mid x^{-d})
\sum_{c\neq a}
p_t(x^d=c\mid x^{-d})
\left[
\lambda_t^d(a\mid x^{-d})
-\lambda_t^d(c\mid x^{-d})
\right]
\nonumber\\
&=
p_t(x^d=a\mid x^{-d})
\left[
\lambda_t^d(a\mid x^{-d})
-
\sum_c p_t(x^d=c\mid x^{-d})
\lambda_t^d(c\mid x^{-d})
\right].
\end{align}
Dividing both sides by $p_t(x^d=a\mid x^{-d})$ gives
\begin{align}
\phi_t^d(a\mid x^{-d})
=
\lambda_t^d(a\mid x^{-d})
-
\sum_c p_t(x^d=c\mid x^{-d})
\lambda_t^d(c\mid x^{-d}).
\end{align}
Since the second term is a constant w.r.t $a$, taking the difference
between states $a$ and $c$ yields
\begin{align}
\lambda_t^d(a\mid x^{-d})
-
\lambda_t^d(c\mid x^{-d})
=
\phi_t^d(a\mid x^{-d})
-
\phi_t^d(c\mid x^{-d}).
\end{align}
Therefore, the minimizing flux is
\begin{align}
J_t^d(c,a\mid x^{-d})
&=
p_t(x^d=a\mid x^{-d})
p_t(x^d=c\mid x^{-d})
\left[
\phi_t^d(a\mid x^{-d})
-
\phi_t^d(c\mid x^{-d})
\right]
\nonumber\\
&=
p_t(x^d=c\mid x^{-d})
\partial_t^{(d)}p_t(x^d=a\mid x^{-d})
-
p_t(x^d=a\mid x^{-d})
\partial_t^{(d)}p_t(x^d=c\mid x^{-d}).
\end{align}
\end{proof}

\subsection{Learning the discrete interpolants}\label{app:learn}

We now derive the training objective.
We also detail the form after using tempered importance reweighting.

\subsubsection{ Learning $ p_t^d(\cdot \mid x^{-d})$}
The entire derivation follows in three steps:
\begin{enumerate}
    \item  we want to use Gibbs label as efficient energy supervision, and hence we write down  $ p_t^d(\cdot \mid x^{-d})$ with both ends explicitly  $ p_t^d(\cdot \mid x^{-d}) = \mathbb{E}_{x_0, x_1 | x^{-d}} [\cdots] = \mathbb{E}_{x_0, x_1^{-d}, x_1^d | x^{-d}} [ \cdots] $, the second equality allows us to bake in the endpoints' Gibbs label into the training;
    \item we apply CE between the optimal $ p_t^d(\cdot \mid x^{-d})$ and our network, writing down a CE objective;
    \item we take expectation over end point distributions, and swapping the conditions: $\mathbb{E}_{x^{-d}}\mathbb{E}_{x_0, x_1  | x^{-d}} [ \cdots]  =\mathbb{E}_{x_0, x_1 }  \mathbb{E}_{x^{-d}|{x_0, x_1 }} [ \cdots]$, allowing us to directly train the model without simulation or sampling from the endpoint posteriors.
\end{enumerate}

\textbf{Step 1} Under the independent coupling and the coordinate-wise interpolant  in \Cref{eq:coord_interpolant}, we have
 \begin{align}
     p_t (x^d=c|x^{-d}) &= \frac{ p_t (x^d=c, x^{-d})}{p_t ( x^{-d})}\\
     &=\frac{1}{p_t ( x^{-d})} \sum_{x_0, x_1} p_0(x_0)p_1(x_1)\kappa_t^d(x^d=c\mid x_0^d,x_1^d)  \prod_{j\neq d}\kappa_t^j(x^j\mid x_0^j,x_1^j)
 \end{align}
 We can also factorize both endpoints into Gibbs label:
 \begin{align}
     p_0(x_0) = p_0^d(x_0^d |x_0^{-d} ) p_0( x_0^{-d} ) ;     \quad p_1(x_1) = p_1^d(x_1^d |x_1^{-d} ) p_1( x_1^{-d} )
 \end{align}
 We can achieve three different variations from these factorization: (1) factorize both  endpoints; (2) factorize $p_0$; and (3)  factorize $p_1$.
 In our experiments, we consider having samples from one endpoints, and hence we only  factorize $p_1$. However, we can also derive other variations.

The factorization over $p_1$ gives us:
\begin{align}
      p_t (x^d=c|x^{-d})  =\frac{1}{p_t ( x^{-d})} \sum_{x_0, x_1} p_0(x_0) p_1^d(x_1^d |x_1^{-d} ) p_1( x_1^{-d} ) \kappa_t^d(x^d=c\mid x_0^d,x_1^d)  \prod_{j\neq d}\kappa_t^j(x^j\mid x_0^j,x_1^j)
\end{align}
Note that
\begin{align}
    \prod_{j\neq d}\kappa_t^j(x^j\mid x_0^j,x_1^j) = p_t(x^{-d}|x_0 , x_1^{-d})
\end{align}
and hence
\begin{align}
    p_0(x_0)   p_1( x_1^{-d} )    \prod_{j\neq d}\kappa_t^j(x^j\mid x_0^j,x_1^j)/ p_t ( x^{-d})  = p_t( x_0 , x_1^{-d} | x^{-d})
\end{align}
Therefore, 
\begin{align}
      p_t (x^d=c|x^{-d})  = \sum_{x_0, x_1^{-d}} p_t( x_0 , x_1^{-d} | x^{-d})   \sum_{x_1^{d}} p_1^d(x_1^d |x_1^{-d} )   \kappa_t^d(x^d=c\mid x_0^d,x_1^d)   
\end{align}
\textbf{Step 2}  To learn this objective, we can simply use cross entropy:
\begin{align}
&\mathcal L_q(\theta)\\
=
&-\sum_{d,c}\mathbb{E}_{x^{-d},t}
p_t(x^d=c\mid x^{-d})
\log p_{\theta,t}^d(c\mid x^{-d})
\\=
&-\sum_{d,c}\mathbb{E}_{t}\sum_{x^{-d}} p(x^{-d})
\sum_{x_0, x_1^{-d}} p_t( x_0 , x_1^{-d} | x^{-d})   \sum_{x_1^{d}} p_1^d(x_1^d |x_1^{-d} )   \kappa_t^d(x^d=c\mid x_0^d,x_1^d)   
\log p_{\theta,t}^d(c\mid x^{-d})
\\
=
&-\sum_{d,c}\mathbb{E}_{t}\sum_{x^{-d}} 
\sum_{x_0, x_1^{-d}} p_t( x_0 , x_1^{-d} , x^{-d})   \sum_{x_1^{d}} p_1^d(x_1^d |x_1^{-d} )   \kappa_t^d(x^d=c\mid x_0^d,x_1^d)   
\log p_{\theta,t}^d(c\mid x^{-d})
\\
={}&
-\sum_{d,c}\mathbb{E}_{t}
\sum_{x_0,x_1^{-d},x^{-d}}
p_0(x_0)\,p_1(x_1^{-d})
\prod_{j\neq d}
\kappa_t^j(x^j\mid x_0^j,x_1^j)
\nonumber\\
&\qquad\qquad\qquad\qquad\times
\left[
\sum_a
p_1^d(a\mid x_1^{-d})
\kappa_t^d(c\mid x_0^d,a)
\right]
\log p_{\theta,t}^d(c\mid x^{-d})
\\
={}&
-\sum_{d,c}\mathbb E_t
\sum_{x_0,x_1,x}
p_0(x_0)\,
p_1(x_1^{-d})p_1^d(x_1^d\mid x_1^{-d})
\nonumber\\
&\qquad\qquad\times
\left[
\prod_{j\neq d}
\kappa_t^j(x^j\mid x_0^j,x_1^j)
\right]
\kappa_t^d(x^d\mid x_0^d,x_1^d)
\nonumber\\
&\qquad\qquad\times
\left[
\sum_a p_1^d(a\mid x_1^{-d})
\kappa_t^d(c\mid x_0^d,a)
\right]
\log p_{\theta,t}^d(c\mid x^{-d})
\end{align}
The last line hold as the remaining part does not depend on the newly added $x^d$ variable. 
\begin{align}   
&\mathcal L_q(\theta)=
-\sum_{d,c}\mathbb E_t
\sum_{x_0,x_1,x}
p_0(x_0)p_1(x_1)
K_t(x\mid x_0,x_1)
\nonumber\\
&\qquad\qquad\times
\left[
\sum_a p_1^d(a\mid x_1^{-d})
\kappa_t^d(c\mid x_0^d,a)
\right]
\log p_{\theta,t}^d(c\mid x^{-d})
\end{align}
\textbf{Step 3} Finally, 
\begin{align}
-\mathbb E_{t,X_0\sim p_0,X_1\sim p_1,
X_t\sim K_t(\cdot\mid X_0,X_1)}
\sum_{d,c}
\left[
\sum_a
p_1^d(a\mid X_1^{-d})
\kappa_t^d(c\mid X_0^d,a)
\right]
\log p_{\theta,t}^d(c\mid X_t^{-d}).
\end{align}

\subsubsection{Learning the temporal potential $\phi_t^d(\cdot\mid x^{-d})$}
Now, we consider how to learn $\phi_t^d( \cdot \mid x^{-d})$.
We follow the same three steps as above.

\textbf{Step 1}
Recall that
\begin{align}
\phi_t^d(c\mid x^{-d})
=
\partial_t^{(d)}
\log p_t(x^d=c\mid x^{-d})
=
\frac{
\partial_t^{(d)}p_t(x^d=c\mid x^{-d})
}{
p_t(x^d=c\mid x^{-d})
}.
\end{align}
Since $\partial_t^{(d)}$ only differentiates the $d$-th interpolant
kernel, we have $\partial_t^{(d)}p_t(x^{-d})=0$, and therefore
\begin{align}
\partial_t^{(d)}p_t(x^d=c\mid x^{-d})
&=
\frac{1}{p_t(x^{-d})}
\partial_t^{(d)}p_t(x^d=c,x^{-d})
\\
&=
\frac{1}{p_t(x^{-d})}
\sum_{x_0,x_1}
p_0(x_0)p_1(x_1)
\partial_t\kappa_t^d(c\mid x_0^d,x_1^d)
\prod_{j\neq d}
\kappa_t^j(x^j\mid x_0^j,x_1^j).
\end{align}
Factorizing $p_1$ as before gives
\begin{align}
\partial_t^{(d)}p_t(x^d=c\mid x^{-d})
=
\sum_{x_0,x_1^{-d}}
p_t(x_0,x_1^{-d}\mid x^{-d})
\sum_{x_1^d}
p_1^d(x_1^d\mid x_1^{-d})
\partial_t\kappa_t^d(c\mid x_0^d,x_1^d).
\end{align}

\textbf{Step 2}
To learn $\phi_t^d$, we consider the quadratic objective
\begin{align}
\mathcal L_\phi(\psi)
=
\sum_{d,c}\mathbb E_{x^{-d},t}
\Bigg[
&\frac12
p_t(x^d=c\mid x^{-d})
\phi_{\psi,t}^d(c\mid x^{-d})^2
\nonumber\\
&-
\partial_t^{(d)}p_t(x^d=c\mid x^{-d})
\phi_{\psi,t}^d(c\mid x^{-d})
\Bigg].
\end{align}
Its first-order condition is
\begin{align}
p_t(x^d=c\mid x^{-d})
\phi_{\psi,t}^d(c\mid x^{-d})
-
\partial_t^{(d)}p_t(x^d=c\mid x^{-d})
=0,
\end{align}
whose solution is exactly
$\phi_{\psi,t}^d=\phi_t^d$.

Substituting the expressions derived above for
$p_t(x^d=c\mid x^{-d})$ and
$\partial_t^{(d)}p_t(x^d=c\mid x^{-d})$ gives
\begin{align}
\mathcal L_\phi(\psi)
=
\sum_{d,c}\mathbb E_t
\sum_{x^{-d},x_0,x_1^{-d}}
&p_t(x_0,x_1^{-d},x^{-d})
\Bigg[
\frac12
\left(
\sum_a
p_1^d(a\mid x_1^{-d})
\kappa_t^d(c\mid x_0^d,a)
\right)
\phi_{\psi,t}^d(c\mid x^{-d})^2
\nonumber\\
&\qquad -
\left(
\sum_a
p_1^d(a\mid x_1^{-d})
\partial_t\kappa_t^d(c\mid x_0^d,a)
\right)
\phi_{\psi,t}^d(c\mid x^{-d})
\Bigg].
\end{align}
Using
\begin{align}
p_t(x_0,x_1^{-d},x^{-d})
=
p_0(x_0)p_1(x_1^{-d})
\prod_{j\neq d}
\kappa_t^j(x^j\mid x_0^j,x_1^j),
\end{align}
we obtain
\begin{align}
\mathcal L_\phi(\psi)
=
\sum_{d,c}\mathbb E_t
\sum_{x_0,x_1^{-d},x^{-d}}
&p_0(x_0)p_1(x_1^{-d})
\prod_{j\neq d}
\kappa_t^j(x^j\mid x_0^j,x_1^j)
\nonumber\\
&\times
\Bigg[
\frac12
\left(
\sum_a
p_1^d(a\mid x_1^{-d})
\kappa_t^d(c\mid x_0^d,a)
\right)
\phi_{\psi,t}^d(c\mid x^{-d})^2
\nonumber\\
&\qquad -
\left(
\sum_a
p_1^d(a\mid x_1^{-d})
\partial_t\kappa_t^d(c\mid x_0^d,a)
\right)
\phi_{\psi,t}^d(c\mid x^{-d})
\Bigg].
\end{align}

As above, the remaining integrand does not depend on the newly
introduced $x_1^d$ or $x^d$. We can therefore multiply by the
normalized distributions
$p_1^d(x_1^d\mid x_1^{-d})$ and
$\kappa_t^d(x^d\mid x_0^d,x_1^d)$ and sum over these variables.
This gives
\begin{align}
\mathcal L_\phi(\psi)
=
\sum_{d,c}\mathbb E_t
\sum_{x_0,x_1,x}
&p_0(x_0)p_1(x_1)
K_t(x\mid x_0,x_1)
\nonumber\\
&\times
\Bigg[
\frac12
\left(
\sum_a
p_1^d(a\mid x_1^{-d})
\kappa_t^d(c\mid x_0^d,a)
\right)
\phi_{\psi,t}^d(c\mid x^{-d})^2
\nonumber\\
&\qquad -
\left(
\sum_a
p_1^d(a\mid x_1^{-d})
\partial_t\kappa_t^d(c\mid x_0^d,a)
\right)
\phi_{\psi,t}^d(c\mid x^{-d})
\Bigg].
\end{align}

\textbf{Step 3}
Recognizing the preceding joint distribution as the usual endpoint
and interpolant sampling procedure, we finally obtain
\begin{align}
\mathcal L_\phi(\psi)
=
\mathbb E_{\substack{
t,\,
X_0\sim p_0,\,
X_1\sim p_1,\\
X_t\sim K_t(\cdot\mid X_0,X_1)
}}
\sum_{d,c}
\Bigg[
&\frac12
\left(
\sum_a
p_1^d(a\mid X_1^{-d})
\kappa_t^d(c\mid X_0^d,a)
\right)
\phi_{\psi,t}^d(c\mid X_t^{-d})^2
\nonumber\\
&-
\left(
\sum_a
p_1^d(a\mid X_1^{-d})
\partial_t\kappa_t^d(c\mid X_0^d,a)
\right)
\phi_{\psi,t}^d(c\mid X_t^{-d})
\Bigg].
\end{align}
\subsubsection{Final Objectives}
The objectives above assume that endpoint samples are drawn from the
correct distributions.
In practice, as in the masked-diffusion case, we only have direct samples
from one endpoint, while samples from the other endpoint are generated by
the current sampler.
We therefore use the same fixed-point iteration with tempered importance
weighting and ESS control.

Let $\bar w(X_1)$ denote the tempered importance weight. The final training objectives are
\begin{align}
\mathcal L_q(\theta)
=
-\mathbb E_{\substack{
t,\,
X_0\sim p_0,\,
X_1\sim q_{\theta^-},\\
X_t\sim K_t(\cdot\mid X_0,X_1)
}}
\bar w(X_1)
\sum_{d,c}
\left[
\sum_a
p_1^d(a\mid X_1^{-d})
\kappa_t^d(c\mid X_0^d,a)
\right]
\log p_{\theta,t}^d(c\mid X_t^{-d}),
\end{align}
and
\begin{align}
\mathcal L_\phi(\psi)
=
\mathbb E_{\substack{
t,\,
X_0\sim p_0,\,
X_1\sim q_{\theta^-},\\
X_t\sim K_t(\cdot\mid X_0,X_1)
}}
\bar w(X_1)
\sum_{d,c}
\Bigg[
&\frac12
\left(
\sum_a
p_1^d(a\mid X_1^{-d})
\kappa_t^d(c\mid X_0^d,a)
\right)
\phi_{\psi,t}^d(c\mid X_t^{-d})^2
\nonumber\\
&-
\left(
\sum_a
p_1^d(a\mid X_1^{-d})
\partial_t\kappa_t^d(c\mid X_0^d,a)
\right)
\phi_{\psi,t}^d(c\mid X_t^{-d})
\Bigg].
\end{align}

\section{Further related work}
\label{app:related}

\paragraph{Neural samplers on discrete spaces.}
Early work on discrete neural samplers uses autoregressive models
\citep{wu2019solving,pan2021solving,hibatallah2021variational,nicoli2020asymptotically,mcnaughton2020boosting,bialas2022hierarchical,bialas2023analysis,sanokowski2023variational},
with subsequent work developing architectural refinements
\citep{biazzo2023autoregressive,delbono2025nearest,bialas2026sampling}.
These models provide explicit generation probabilities and can be
trained through variational objectives
\citep{wu2019solving,nicoli2020asymptotically}.
GFlowNets instead learn generation policies from unnormalized rewards
\citep{bengio2021flow,malkin2022trajectory,zhang2022generative,jain2022biological,zhang2023let},
while variational annealing and diffusion samplers address
combinatorial optimization and statistical physics
\citep{sanokowski2024diffusion,sanokowski2025scalable}.
Recent CTMC-based approaches include locally equivariant rate models
\citep{holderrieth2025leaps,ou2025discrete},
importance-weighted energy-based generator matching with bootstrapping
\citep{woo2025energy}, optimal-control approaches to masked diffusion
\citep{zhu2025mdns,guo2026proximal}, and discrete adjoint matching or
off-policy bridge learning
\citep{guo2026asbs,carter2026discrete}.
Applications include sampling disordered materials and alloys
\citep{damewood2022sampling,du2026scaling,blessing2026janus,min2026frogs},
with related work also investigating enhanced sampling
\citep{du2026metadns}.

Among masked energy-based samplers, MDNS trains with
importance-weighted denoising cross entropy \citep{zhu2025mdns}.
Our method retains the cross-entropy objective but replaces one-hot
reconstruction targets with Gibbs soft labels, combining this
supervision with tempered importance weighting.
This use of local energy information is motivated by continuous
neural samplers, including adjoint Schr\"odinger bridge sampling,
neural flow samplers, and bridge matching samplers, which incorporate
energy-gradient information into their learning objectives
\citep{liu2025adjoint,chen2025neuralflow,blessing2026bridge}.
Single-site energy differences play a related informational role
in discrete spaces, without requiring a differentiable extension
of the energy, and are also exploited by discrete ASBS
\citep{guo2026asbs}.
Our construction converts these local comparisons into categorical
supervision.
This local conditional-matching objective is also related
to distribution-based target concrete score matching
\citep{zhang2025target}.

\paragraph{Masked diffusion models.}
Discrete diffusion with uniform or masking noise
\citep{hoogeboom2021argmax,austin2021structured}
has been formulated in continuous time
\citep{campbell2022continuous}, through ratio estimation
\citep{meng2022concrete,sun2023score,lou2024discrete,zhang2025target},
and through flow or generator matching along general probability paths
\citep{campbell2024generative,gat2024discrete,shaul2025flow,holderrieth2025generator}.
For masking processes, training reduces to a weighted cross-entropy
objective for conditional categorical predictions
\citep{sahoo2024simple,shi2024simplified}.
Under value-independent masking, the exact clean-symbol conditional
does not depend on time once the revealed context is given
\citep{ou2025absorbing,zheng2025masked}.
Masked models also admit an any-order autoregressive interpretation
\citep{uria2014deep,hoogeboom2022autoregressive,shih2022training},
while masked parallel decoding predates these diffusion formulations
\citep{ghazvininejad2019mask,chang2022maskgit,bondtaylor2022unleashing}.

Trained from data, masked diffusion has been developed for language
and protein modeling
\citep{nie2025scaling,gong2025scaling,nie2025large,ye2025dream,arriola2025block,wang2024diffusion}.
Related advances include guidance
\citep{nisonoff2025unlocking,schiff2025simple},
correctors and remasking
\citep{zhao2024informed,wang2025remasking},
reveal-order planning
\citep{liu2025think,peng2025path,kim2025train},
and energy-based or uniform-state variants
\citep{xu2025energy,sahoo2025diffusion}.
Our setting differs from these data-trained constructions:
the target is specified through an unnormalized energy, without
access to target-distributed training samples.
We nevertheless exploit the same conditional-generation structure.
For a sampled single-site reveal order, the conditional generation
probability is a tractable product of categorical prediction
probabilities \citep{hoogeboom2022autoregressive}.
This permits exact importance weights on an augmented space that
includes the generation path and its reveal order.
It does not imply that the marginal probability of a random-order
model, obtained by summing over reveal orders, is tractable.

\paragraph{Discrete interpolant and flow models.}
Discrete flow matching
and generator matching learn transition dynamics along general
interpolant paths
\citep{gat2024discrete,holderrieth2025generator}.
However, they do not consider constructing time-reversal forward and backward between two arbitrary distributions.
\citet{he2026free} develop a general framework for learning such dynamics from endpoint interpolants.
However, all these models only consider data-based training, while our interpolant framework makes use of the Gibbs label for efficient energy-based training.

\paragraph{Autoregressive Boltzmann generators for molecules}
Autoregressive Boltzmann generators (ArBGs) bin molecular coordinates and model the resulting categorical variables with a causal transformer \citep{rehman2026autoregressive}.
This is closely related to our experiments on molecules.
However, unlike ArBGs, which are trained by maximum likelihood on molecular-dynamics data, our method learns directly from the energy, using local energy differences to construct Gibbs soft labels and a fixed-point training procedure. Moreover, our masked formulation supports any-order rather than a fixed causal generation order, while retaining tractable generation probabilities and exact importance reweighting to the underlying continuous Boltzmann distribution.
In our molecular experiments, we discretize internal coordinates and recover equilibrium distributions without using MD data for training. 

\section{Experimental setup}\label{app:setup}

This section collects the setup of every experiment of \cref{sec:experiments},
one subsection per task family.  Further results are in \cref{app:further}.

\subsection{Lattice models}\label{app:lattice}

\subsubsection{Tasks}\label{app:lattice-tasks}

Both models are defined on a periodic two-dimensional lattice of $D$ sites
carrying $q$-state spins $\sigma_i\in\{1,\dots,q\}$, with
\begin{equation}
  E(\sigma) = -\!\!\sum_{\langle ij\rangle}\!\delta_{\sigma_i\sigma_j}
  \ \ \text{(Potts)},
  \qquad
  E(\sigma) = -\!\!\sum_{\langle ij\rangle}\! s_is_j,\quad
  s_i = 2\sigma_i-3\in\{\pm1\}
  \ \ \text{(Ising)},
  \label{eq:lattice-energy}
\end{equation}
each bond counted once, and $p_\beta(\sigma)=e^{-\beta E(\sigma)}/Z(\beta)$
at inverse temperature $\beta$.  The Ising energy is the $q=2$ Potts energy in
its usual convention using $\delta_{\sigma_i\sigma_j}=(1+s_is_j)/2$. 
The three phases of the Ising model are $\beta=0.30$ (disordered), $\beta=\beta_c=\tfrac12\ln(1+\sqrt2)=0.4407$
(critical) and $\beta=0.60$ (ordered) and $\beta=0.70$, $\beta\approx\beta_c=\ln(1+\sqrt3)=1.0050$,
and $\beta=1.20$ for the Potts model. The ordered phase of the Potts model has $q=3$ symmetry-related modes rather than two.
The  reference structures used to compute reference values come from a long Monte Carlo chain with cluster updates
\citep{wolff1989collective,swendsen1987nonuniversal}. Example configurations are shown in \cref{fig:configs}.

\begin{figure}[t!]
\centering
\tikzsetnextfilename{configs}
\begin{tikzpicture}
\begin{groupplot}[
    group style={group size=3 by 2, horizontal sep=0.45cm, vertical sep=0.7cm},
    cfgpanel, colormap name=ising, point meta min=0, point meta max=1]
  \nextgroupplot[title={disordered}, ylabel={Ising $25\times25$},
                 ylabel style={font=\small}]
    \addplot[matrix plot*, point meta=explicit, mesh/cols=25]
      table[x=x, y=y, meta=v] {data/lattice/ising25_disordered_config.dat};
  \nextgroupplot[title={critical}]
    \addplot[matrix plot*, point meta=explicit, mesh/cols=25]
      table[x=x, y=y, meta=v] {data/lattice/ising25_critical_config.dat};
  \nextgroupplot[title={ordered}]
    \addplot[matrix plot*, point meta=explicit, mesh/cols=25]
      table[x=x, y=y, meta=v] {data/lattice/ising25_ordered_config.dat};
  \nextgroupplot[ylabel={Potts $16\times16$}, ylabel style={font=\small},
                 colormap name=potts, point meta max=2]
    \addplot[matrix plot*, point meta=explicit, mesh/cols=16]
      table[x=x, y=y, meta=v] {data/lattice/potts16_disordered_config.dat};
  \nextgroupplot[colormap name=potts, point meta max=2]
    \addplot[matrix plot*, point meta=explicit, mesh/cols=16]
      table[x=x, y=y, meta=v] {data/lattice/potts16_critical_config.dat};
  \nextgroupplot[colormap name=potts, point meta max=2]
    \addplot[matrix plot*, point meta=explicit, mesh/cols=16]
      table[x=x, y=y, meta=v] {data/lattice/potts16_ordered_config.dat};
\end{groupplot}
\end{tikzpicture}
\caption{Example configurations of the Ising and Potts model at different phases.}
\label{fig:configs}
\end{figure}

\subsubsection{Evaluation metrics}\label{app:lattice-metrics}

Each method draws $M=\NDraws$ independent configurations.  Writing $a^{(m)} = \ln\hat\nu(\sigma^{(m)}) - \ln
q(\sigma^{(m)})$ for the log importance weight of draw $m$ and
$\bar w^{(m)} = e^{a^{(m)}}/\sum_{m'} e^{a^{(m')}}$ for its self-normalized
form. The evaluation metrics are computed as follows:
\begin{align}
  \ESS &= \frac{1}{M\sum_m (\bar w^{(m)})^2} \in \big[\tfrac1M, 1\big],
  \label{eq:ess}\\
  |\Delta \ln Z| &= \Big|\ln\tfrac1M\textstyle\sum_m e^{a^{(m)}} - \ln Z\Big|,
  \label{eq:dlnz}\\
  \mTVD &= \tfrac12\textstyle\sum_{k=1}^{q} \big|\hat\pi_k - \pi^{\mathrm{ref}}_k\big|,
  \qquad \hat\pi_k = \tfrac1M\textstyle\sum_m \rho_k(\sigma^{(m)}),
  \label{eq:tvd}
\end{align}
where $\rho_k(\sigma)$ assigns each configuration to one of the $q$ ordered
modes: for Ising, $\rho_k$ is the indicator of the sign of the magnetization
$m(\sigma)=\frac1D\sum_i s_i$ (the two modes of $m>0$ and $m<0$); for Potts,
of the majority color $\arg\max_k\sum_i\delta_{\sigma_i k}$, with a tie shared
equally between the tied colors.  $\pi^{\mathrm{ref}}$ is the reference
distribution over the modes, exactly uniform at zero field
(\cref{app:lattice-tasks}). In the amortized runs, where a field breaks the
symmetry, it is estimated from reference samples.

\subsubsection{Hyperparameters}\label{app:lattice-hparams}

All methods train from the energy alone, with no reference data, and are given
the same step budget, the same batch size and the same network capacity where
the architecture admits it.  No method uses MCMC at evaluation.

\begin{table}[h]
\caption{Settings shared by every method and state point.}
\label{tab:hp-shared}
\centering
\small
\begin{tabular}{@{}>{\raggedright}p{4.6cm}p{8.2cm}@{}}
\toprule
Systems & Ising $25\times25$ ($q=2$, $D=625$); Potts $16\times16$ ($q=3$, $D=256$) \\
Phases per system & 3 (disordered, critical, ordered); see \cref{app:lattice-tasks} \\
Seeds & \NSeeds{} per (system, phase, method) \\
Training steps & 150k \\
Evaluation draws & \NDraws{} per seed, independent, no MCMC correction \\
Backbone & periodic residual CNN, 2-D circular convolutions, 96 channels, 7 residual blocks, kernel 3, dilations covering the torus (all methods but LEAPS)\\
Optimizer & AdamW \citep{loshchilov2019decoupled} without weight decay (Adam for LEAPS) \\
\bottomrule
\end{tabular}
\end{table}

\paragraph{\method{}.}
A masked interpolant.  The sampler reveals one site per network
evaluation in a uniformly random order, $D$ evaluations per draw, so the
proposal density of every draw is exact.  Training runs the fixed-point loop:
every 100 gradient steps the EMA model draws 512 configurations into a replay
buffer of 2048, each stored with its exact table of single-site energy
differences.  The denoiser is trained by soft cross-entropy onto the Gibbs labels of the buffered endpoints at freshly masked copies of them, with
the labels damped toward the EMA model's prediction ($\eta=0.1$) and the
per-endpoint losses weighted by importance weights tempered to an $\ESS$ of
0.8 within each training batch.  Batch 128, learning rate $3\times10^{-4}$,
EMA decay 0.995, gradient clipping at 1.0.

\paragraph{MDNS.}
The masked diffusion neural sampler \citep{zhu2025mdns}, run in the same
codebase with the same interpolant, backbone, sampler, buffer and optimizer as
\method{}.  It is trained by cross-entropy onto the endpoint's own token
rather than a heat-bath label, weighted by the full self-normalized importance
weights, without tempering and without damping.

\paragraph{PDNS.}
The proximal diffusion neural sampler \citep{guo2026proximal}: MDNS with its
importance weights tempered to an $\ESS$ of 0.8 within each training batch,
the same rule as \method{}.  Labels are hard and there is no damping, so PDNS
differs from \method{} in exactly the label and the damping, and from MDNS in
exactly the tempering.

\paragraph{DASBS.}
The discrete adjoint Schr\"odinger bridge sampler \citep{guo2026asbs}: a
time-conditioned controller and a time-free corrector, both on the shared
backbone, over a uniform reference CTMC with a log-linear noise schedule
($\gamma=1$, $\alpha=0.5$).  The two are trained in alternating stages of
Bregman matching, 400 controller steps and then 200 corrector steps; the
controller regresses onto the target's single-site score ratios at the
endpoint, which needs the energy-difference table.  Each batch of 128
endpoints is replicated 8 times, the buffer of 2048 is refreshed with 512
rollouts every 20 steps, and a draw takes 100 steps.  Learning rate
$10^{-3}$, EMA decay 0.9999.  With memory, the path weights give $\ln Z$
only at convergence and no importance weights, so its $|\Delta\ln Z|$ and
$\ESS$ are not reported; the memoryless variant drops the corrector and has
unbiased path weights.

\paragraph{LEAPS.}
LEAPS \citep{holderrieth2025leaps} learns a CTMC that transports the uniform
distribution to the target along the annealed path
$\rho_t\propto e^{-t\beta E}$, with a locally equivariant rate network (20
channels, kernels 3, 5, 7 and 9) and a free-energy MLP (hidden width 200),
trained on the squared residual of the Kolmogorov forward equation at
buffered trajectory states. We use training rollouts of 128 walkers every 20 steps, fill a buffer of 50 rollouts, and each loss batch
holds 128 states.  Rates are clamped at 10.  Learning rate $3\times10^{-4}$,
decayed by 0.995 every 1000 steps. Importance weights follow from the discrete
Jarzynski equality.

\subsubsection{Computational costs}\label{app:lattice-compute}
Each run trains on one 80\,GB H100 for 150k steps.  \Cref{tab:compute}
reports the training time and the number of energy calls during training,
counted per configuration.  A method calls the target for two things: the
energy $E(\sigma)$ of a configuration, for its importance weight or the
annealed path, and the table of single-site energy differences
$\{E(\sigma^{d\leftarrow n})-E(\sigma)\}_{d\le D,\,n\le q}$, which \method{}
needs for its heat-bath labels and DASBS and LEAPS for their training
targets.  For the lattice models the energy is local, and the whole table
follows from one neighbour-count pass at about the cost of a single energy
evaluation, so we count it as one call.  For a target that is only available
as a black box, the table needs the energy of every configuration that
differs from $\sigma$ in one site, that is $D(q-1)+1$ energy evaluations
instead of 1; the black-box column counts the same runs at this rate, $626$
calls per table for the Ising model and $513$ for the Potts model.

\begin{table}[h]
\caption{\textbf{Computational cost.}  Training time in hours (mean${\pm}$standard
deviation over \NSeeds{} seeds and the three phases) and the total number of
energy calls during training, with the single-site difference table counted
as one call (local) or as $D(q-1)+1$ calls (black box), the cost it would
have if the energy were not local.}
\label{tab:compute}
\centering
\footnotesize
\setlength{\tabcolsep}{4pt}
\begin{tabular}{@{}l ccc@{\hspace{10pt}} ccc@{}}
\toprule
& \multicolumn{3}{c}{Ising $25\times25$} & \multicolumn{3}{c}{Potts $16\times16$} \\
\cmidrule(lr){2-4}\cmidrule(lr){5-7}
& & \multicolumn{2}{c}{energy calls} & & \multicolumn{2}{c}{energy calls} \\
\cmidrule(lr){3-4}\cmidrule(lr){6-7}
Method & train (h) & local & black box & train (h) & local & black box \\
\midrule
\method{} (ours)  & \CmpIsingDSITrain   & \CmpIsingDSITotal   & \CmpIsingDSIBB   & \CmpPottsDSITrain   & \CmpPottsDSITotal   & \CmpPottsDSIBB   \\
MDNS              & \CmpIsingMDNSTrain  & \CmpIsingMDNSTotal  & \CmpIsingMDNSBB  & \CmpPottsMDNSTrain  & \CmpPottsMDNSTotal  & \CmpPottsMDNSBB  \\
PDNS              & \CmpIsingPDNSTrain  & \CmpIsingPDNSTotal  & \CmpIsingPDNSBB  & \CmpPottsPDNSTrain  & \CmpPottsPDNSTotal  & \CmpPottsPDNSBB  \\
DASBS             & \CmpIsingASBSTrain  & \CmpIsingASBSTotal  & \CmpIsingASBSBB  & \CmpPottsASBSTrain  & \CmpPottsASBSTotal  & \CmpPottsASBSBB  \\
DASBS, memoryless & \CmpIsingASBSMTrain & \CmpIsingASBSMTotal & \CmpIsingASBSMBB & \CmpPottsASBSMTrain & \CmpPottsASBSMTotal & \CmpPottsASBSMBB \\
LEAPS             & \CmpIsingLEAPSTrain & \CmpIsingLEAPSTotal & \CmpIsingLEAPSBB & \CmpPottsLEAPSTrain & \CmpPottsLEAPSTotal & \CmpPottsLEAPSBB \\
\bottomrule
\end{tabular}
\end{table}

\subsection{Amortized training for lattice models}\label{app:amortized}

\subsubsection{Setup}\label{app:fp-setup}
The distribution over the thermodynamic conditions $(\beta,h)$ is set as follows: Half of the samples are drawn uniformly over $[\beta_{\mathrm{min}},\beta_\mathrm{max}]\times [h_{\mathrm{min}},h_\mathrm{max}]$. The other half is drawn from an
acquisition map on an $8\times12$ grid of cells $C_k$ of the box.  Each cell
is scored by the trace of the Fisher information of the target family,
estimated from the model's own draws,
\begin{equation}
  a_k=\operatorname{Tr}\operatorname{Cov}_{\sigma\sim q(\cdot\mid\theta),\,\theta\in C_k}
  \big[\nabla_{\tilde\theta}\ln\tilde p(\sigma\mid\theta)\big],
  \label{eq:fp-acq}
\end{equation}
with $\theta=(\beta,h)$ rescaled to the unit box ($\tilde\theta$).  For Ising
the two components are $\partial_\beta\ln\tilde p=\sum_{\langle ij\rangle}s_is_j+h\sum_is_i$
and $\partial_h\ln\tilde p=\beta\sum_is_i$, so $a_k$ combines the specific heat
and the magnetic susceptibility, which peak at the phase transition.  Each
component is divided by its median over the cells, and cells are drawn with
probability proportional to $a_k^{1/2}$.  The
acquisition only decides which conditions receive more training, not the
labels or the loss, so the fixed point is unchanged. Importance weights are self-normalized per acquisition cell as 
$\ln Z(\beta,h)$ depends on $(\beta,h)$. \Cref{tab:hp-fp} lists the settings.

\begin{table}[h]
\caption{Settings of the amortized runs.}
\label{tab:hp-fp}
\centering
\small
\setlength{\tabcolsep}{5pt}
\begin{tabular}{@{}>{\raggedright}p{3.1cm}>{\raggedright}p{4.7cm}p{4.7cm}@{}}
\toprule
& Ising & Potts, $q=3$ \tabularnewline
\midrule
Lattice & $50\times50$ torus, $D=2500$ & $24\times24$ torus, $D=576$ \tabularnewline
Training box & $\beta\in[\FPIsingBetaLo,\FPIsingBetaHi]$, $|h|\le\FPIsingHmax$
             & $\beta\in[\FPPottsBetaLo,\FPPottsBetaHi]$, $|h|\le\FPPottsHmax$ \tabularnewline
Zero-field $\beta_c$ & $\tfrac12\ln(1+\sqrt2)=0.4407$ & $\ln(1+\sqrt3)=1.00505$ \tabularnewline
Network & \multicolumn{2}{p{9.6cm}}{residual CNN on the torus: 7 blocks, 96 channels,
  $3\times3$ circular convolutions, conditioned by FiLM \citep{perez2018film} on
  Fourier features of $(t,\beta,h)$; dilations $(1,1,2,4,8,1,1)$ for Ising so that
  7 blocks span the $50\times50$ torus, none for Potts} \tabularnewline
Optimizer & \multicolumn{2}{p{9.6cm}}{AdamW \citep{loshchilov2019decoupled}, learning
  rate $3\times10^{-4}$, no weight decay, gradient clipping at 1.0, EMA 0.995} \tabularnewline
Batch / buffer & 1536 / 49\,152 & 1536 / 196\,608 \tabularnewline
Damping, ESS target & \multicolumn{2}{p{9.6cm}}{$\eta=0.5$; weights tempered to
  ESS 0.8 per acquisition cell} \tabularnewline
Training & \FPIsingRounds{} rounds of 100 steps & \FPPottsRounds{} rounds of 100 steps \tabularnewline
Sampler & \multicolumn{2}{p{9.6cm}}{random-order unmasking with an exact proposal
  density, \FPIsingSteps{} (Ising) and \FPPottsSteps{} (Potts) network evaluations,
  the first 64 revealing a single site} \tabularnewline
Seeds & \FPIsingNSeeds & \FPPottsNSeeds \tabularnewline
Evaluation & \multicolumn{2}{p{9.6cm}}{$9\times9$ lattice over the box, \FPNDraws{}
  draws per cell, EMA weights of the last checkpoint, no MCMC, no resampling} \tabularnewline
\bottomrule
\end{tabular}
\end{table}

\subsection{Alloy phase diagrams}\label{app:alloy}

\subsubsection{Tasks}\label{app:alloy-tasks}

\paragraph{Semi-grand canonical ensemble.}
A binary alloy on a fixed lattice of $D$ sites is described by occupations
$\sigma=(\sigma_1,\dots,\sigma_D)\in\{-1,+1\}^D$, where $\sigma_i=+1$ marks
the first species (Ag in Ag--Pd, Au in Cu--Au).  Its configurational energy is
a cluster expansion \citep{sanchez1984generalized}, a multilinear polynomial
in the occupations,
\begin{equation}
  E(\sigma) = J_\emptyset + \sum_{\omega} J_\omega \prod_{i\in\omega}\sigma_i ,
\end{equation}
where the sum runs over the clusters $\omega$ (points, pairs and higher-order
clusters) of the orbits that the expansion retains.  We use the DFT-fitted
effective cluster interactions $J_\omega$ of \citet{damewood2022sampling},
which are also the energies used by \citet{min2026frogs} and
\citet{du2026scaling}, without subtracting a pure-element reference.  At
condition $\theta=(T,\DMu)$, with $\DMu$ the difference between the chemical
potentials of the first and second species and $\beta=1/\kB T$, the target is
\begin{equation}
  \nu_\theta(\sigma) = \frac{\tilde\nu_\theta(\sigma)}{Z(\theta)},\qquad
  \tilde\nu_\theta(\sigma) = e^{-\beta H_\theta(\sigma)},\qquad
  H_\theta(\sigma) = E(\sigma) - \frac{\DMu}{2}\sum_{i=1}^D\sigma_i .
  \label{eq:sgc}
\end{equation}
Let $n(\sigma)=\sum_i(1+\sigma_i)/2$ be the number of atoms of the first
species, so that $c=n/D$ is the composition.  Since $\sum_i\sigma_i=2n-D$,
\cref{eq:sgc} equals $E-\DMu\,n$ up to a constant that does not depend on
the configuration.  That form is the convention of \citet{min2026frogs}, and
both forms define the same distribution.

\paragraph{Systems.}
\Cref{tab:alloy-systems} lists the two cells.  Sites are indexed by their
integer coordinates in the fcc primitive basis.  This maps each cell onto a
dense periodic grid on which 3-D convolutions act, while the energy is always
evaluated on the physical fcc geometry through the cluster tables.  The
training box is the same for both systems.  The evaluation sweep is the grid
on which the reference data are tabulated.

\begin{table}[h]
\caption{The two alloy tasks.}
\label{tab:alloy-systems}
\centering
\small
\begin{tabular}{@{}lll@{}}
\toprule
& Ag--Pd & Cu--Au \\
\midrule
Lattice, cell & fcc, $5\times5\times5$ primitive & fcc, $8\times4\times4$ primitive \\
Sites $D$ & 125 & 128 \\
$\sigma_i=+1$ & Ag & Au \\
Phases & one miscibility gap & Cu$_3$Au (L1$_2$), CuAu (L1$_0$), CuAu$_3$ (L1$_2$) \\
Training box & $T\in[200,900]$\,K, $\DMu\in[-0.4,0.4]$\,eV & $T\in[200,900]$\,K, $\DMu\in[-0.4,0.4]$\,eV \\
Evaluation $T$ & 200--900\,K, step 25\,K (\AgPdNT) & 200--900\,K, step 25\,K (\CuAuNT) \\
Evaluation $\DMu$ & $\AgPdDmuMin$ to $\AgPdDmuMax$\,eV, step 0.01\,eV (\AgPdNdmu) & $\CuAuDmuMin$ to $\CuAuDmuMax$\,eV, step 0.01\,eV (\CuAuNdmu) \\
State points & \AgPdNpoints & \CuAuNpoints \\
Reference tie lines & \AgPdNties, highest at \AgPdTClose\,K & \CuAuNties, highest at \CuAuTClose\,K \\
\bottomrule
\end{tabular}
\end{table}

\paragraph{Training conditions.}
All three methods draw training conditions with $1/T$ uniform on
$[1/900,1/200]$\,K$^{-1}$ and $\DMu$ uniform on $[-0.4,0.4]$\,eV.  Because the
log-density is linear in $\beta$, sampling uniformly in $\beta$ spreads the
magnitude of the energy term evenly across the box.  \method{} replaces half of each
batch of conditions with draws from an adaptive acquisition map
(\cref{app:dsi-hparams}).

\paragraph{Exact reference.}
The reference data tabulate the fixed-composition partition functions
\begin{equation}
  Z_n(T) = \sum_{\sigma:\,n(\sigma)=n} e^{-\beta E(\sigma)}, \qquad n=0,\dots,D,
\end{equation}
at every temperature of the sweep.  Since
$\DMu$ couples only to $n$, the partition function and the composition
distribution follow in closed form on the whole plane,
\begin{equation}
  \ln Z(\theta) = \logsumexp_{n}\Big[\ln Z_n(T) + \beta\tfrac{\DMu}{2}(2n-D)\Big],
  \qquad
  p_\theta(n) = \frac{Z_n(T)\,e^{\beta\DMu(2n-D)/2}}{Z(\theta)},
  \label{eq:exact-lnz}
\end{equation}
and so does the mean composition $\langle c\rangle_\theta=\sum_n (n/D)\,p_\theta(n)$.
All errors are measured against these exact values.

\subsubsection{Constructing the phase diagram}\label{app:phase-diagram}

We follow the constructor of \citet[App.~A.4]{min2026frogs} and apply it
unchanged to every method. However, for completeness, we outline the procedures below.

\paragraph{Fixed-composition free energy from weighted draws.}
At condition $\theta$, a sampler returns draws $\sigma^{(m)}\sim q(\cdot\mid\theta)$,
$m=1,\dots,M$, with log-weights
$a^{(m)} = \ln\tilde\nu_\theta(\sigma^{(m)}) - \ln q(\sigma^{(m)}\mid\theta)$
(for FrOGS, the path weight $A_1$ of \citealp{holderrieth2025leaps}).  Since
$\E_q[e^{a}\,\mathbf{1}\{n(\sigma)=n\}] = Z_n(T)\,e^{\beta\DMu(2n-D)/2}$,
\begin{equation}
  \hat Z_n(T) = e^{-\beta\DMu(2n-D)/2}\,\frac1M\sum_{m:\,n(\sigma^{(m)})=n} e^{a^{(m)}},
  \qquad
  \hat f(c,T) = -\frac{\kB T}{D}\ln\hat Z_n(T),\quad c=n/D,
  \label{eq:zn-hat}
\end{equation}
is an unbiased estimate of $Z_n(T)$ at every composition that some draw visits.  It is a
histogram without a bin-width parameter, because composition is discrete.
Because $\hat Z$ itself is estimated from the same draws, the estimate is absolute.

\paragraph{Pooling over chemical potentials.}
$Z_n(T)$ does not depend on $\DMu$, so every condition $\theta_i=(T,\DMu_i)$
at a given temperature gives an independent estimate $\hat y_i(n)=\ln\hat Z_n$
of the same curve.  Its standard error is
$s_i(n)=N_{\mathrm{eff},i}(n)^{-1/2}$, with
$N_{\mathrm{eff},i}(n)=\big(\sum_{m\in n}e^{a^{(m)}}\big)^2/\sum_{m\in n}e^{2a^{(m)}}$
the effective sample count of the draws at $\theta_i$ that land at
composition $n$.  The estimates are merged by precision weighting,
\begin{equation}
  \hat y(n) = \frac{\sum_i w_i(n)\,\hat y_i(n)}{\sum_i w_i(n)},\qquad
  \operatorname{Var}\hat y(n) = \frac{1}{\sum_i w_i(n)},\qquad w_i=s_i^{-2}.
\end{equation}
Error bars at temperature $T$ are then widened by the Birge ratio
\citep{birge1932calculation},
$\lambda=\max\{1,\operatorname{median}_{n:\,\sum_i w_i(n)\ge 5}\chi^2_\nu(n)\}^{1/2}$
with $\chi^2_\nu(n)=\frac{1}{k_n-1}\sum_i w_i(\hat y_i-\hat y)^2$ over the
$k_n$ contributing conditions.  This correction only acts when the conditions
disagree by more than their errors predict.

\paragraph{Common tangent.}
Two phases coexist at temperature $T$ when a single line is tangent to $f(\cdot,T)$ at two
compositions $c_a<c_b$.  These are the edges of the lower convex hull of
$\{(c,\hat f(c,T))\}$.  A hull edge is accepted as a tie line only if
(i)~at least four compositions lie strictly between its ends, counted on
the composition axis $n=0,\dots,D$ rather than on the list of resolved
compositions;
(ii)~$\hat f$ lies at least three standard errors above the chord at some
interior composition, or no draw at any $\DMu$ at that temperature visited the
interior; and
(iii)~an accepted gap overlapping it exists at the next lower temperature,
chained from the lowest temperature.
Criterion (i) matters at low temperature, where ordered phases are sharp and
most compositions are never visited.  Criterion (iii) must chain on
\emph{accepted} edges, since otherwise a rejected isotherm could license the one
above it.  A two-phase region ends at the highest temperature at which its edge is still
accepted.  The reference curves in \cref{fig:alloy-pd} are the same
common-tangent construction on the exact $\ln Z_n(T)$.

\Cref{fig:alloy-fmix} shows the whole construction on Cu--Au, at the
coldest isotherm and at the last one before the two-phase region closes.  At
both, the curve estimated from \method{}'s draws follows the exact one and
yields the same tie lines, \FmixTiesModA{} of \FmixTiesRefA{} at \FmixTA\,K
and \FmixTiesModB{} of \FmixTiesRefB{} at \FmixTB\,K.

\begin{figure}[t]
\centering
\tikzsetnextfilename{alloy-fmix}
\begin{tikzpicture}
\begin{groupplot}[
    group style={group size=4 by 1, horizontal sep=0.35cm,
                 y descriptions at=edge left},
    scale only axis, width=2.6cm, height=3.0cm,
    xmin=0, xmax=1, ymin=\FmixLo, ymax=\FmixHi,
    xtick={0,0.5,1}, ytick={0,-0.01,-0.02,-0.03},
    scaled y ticks=false,
    yticklabel style={/pgf/number format/fixed, /pgf/number format/precision=2},
    tick label style={font=\scriptsize}, label style={font=\small},
    title style={font=\small, yshift=-0.8ex},
    unbounded coords=jump, tick align=outside, axis line style={black!60},
    xlabel={$c$}]
  \nextgroupplot[title={exact, \FmixTA\,K},
                 ylabel={$f-\text{chord}$ (eV/site)}]
    \addplot[fmixref] table {data/alloy/cuau128_fmix_ref_A.dat};
    \addplot[tangentref] table {data/alloy/cuau128_tangent_ref_A.dat};
  \nextgroupplot[title={\textcolor{cDSI}{\method{}}, \FmixTA\,K}]
    \addplot[fmixback] table {data/alloy/cuau128_fmix_ref_A.dat};
    \addplot[fmixmodel] table {data/alloy/cuau128_fmix_dsi_A.dat};
    \addplot[tangentmodel] table {data/alloy/cuau128_tangent_dsi_A.dat};
  \nextgroupplot[title={exact, \FmixTB\,K}]
    \addplot[fmixref] table {data/alloy/cuau128_fmix_ref_B.dat};
    \addplot[tangentref] table {data/alloy/cuau128_tangent_ref_B.dat};
  \nextgroupplot[title={\textcolor{cDSI}{\method{}}, \FmixTB\,K}]
    \addplot[fmixback] table {data/alloy/cuau128_fmix_ref_B.dat};
    \addplot[fmixmodel] table {data/alloy/cuau128_fmix_dsi_B.dat};
    \addplot[tangentmodel] table {data/alloy/cuau128_tangent_dsi_B.dat};
\end{groupplot}
\end{tikzpicture}
\caption{\textbf{Common-tangent construction on Cu--Au.}  Mixing free
energy against composition at \FmixTA\,K (left pair) and \FmixTB\,K (right
pair), measured from the chord through the pure compositions.  In each pair:
exact (left) and from the draws of one \method{} run (right, exact curve in
grey).  Dashed chords are the accepted tie lines, whose ends are the phase
boundaries at that temperature.}
\label{fig:alloy-fmix}
\end{figure}

\subsubsection{Evaluation metrics}\label{app:alloy-metrics}

At every state point $\theta$ of the sweep each method draws $M=\AlloyNDraws$
independent configurations, with no MCMC moves and no resampling.  With
self-normalized weights $\bar w^{(m)}=e^{a^{(m)}}/\sum_{m'}e^{a^{(m')}}$,
\begin{align}
  \ESS(\theta) &= \frac{1}{M\sum_m (\bar w^{(m)})^2}\in\big[\tfrac1M,1\big], \\
  |\Delta\ln Z|(\theta)/D &= \frac1D\Big|\ln\frac1M\sum_m e^{a^{(m)}} - \ln Z(\theta)\Big|, \\
  |\Delta c|(\theta) &= \Big|\sum_m \bar w^{(m)}\,c(\sigma^{(m)}) - \langle c\rangle_\theta\Big|,
\end{align}
with $\ln Z(\theta)$ and $\langle c\rangle_\theta$ exact from
\cref{eq:exact-lnz}.  The ESS is that of \citet{kong1994sequential}: it
equals 1 when all weights are equal and $1/M$ when one draw carries all the
weight.  \Cref{fig:alloy-pd} reports each quantity averaged over all state points
of the sweep, then the mean and standard deviation over seeds.

\subsubsection{Hyperparameters}\label{app:alloy-hparams}\label{app:dsi-hparams}\label{app:frogs-hparams}
\Cref{tab:hp-dsi,tab:hp-frogs} list the settings of \method{} and FrOGS,
which are identical for Ag--Pd and Cu--Au.  Both methods train for 60k
steps on the same energy, conditioning box and conditioning distribution, and
are evaluated with \AlloyNDraws{} draws per state point without MCMC moves.
Code for FrOGS was not public at the time of our experiments, so we
reimplemented it following Sections~4 and A.2--A.9 of \citet{min2026frogs}.

\begin{table}[h]
\caption{\method{} hyperparameters, identical for Ag--Pd and Cu--Au.}
\label{tab:hp-dsi}
\centering
\small
\begin{tabular}{@{}>{\raggedright}p{4.6cm}p{8.2cm}@{}}
\toprule
\multicolumn{2}{@{}l}{\emph{Interpolant and sampler}} \\
Noising process & masked (absorbing state) \\
Reveal schedule & one site per network evaluation ($D$ evaluations per draw) \\
Training labels & heat-bath conditionals from the exact single-site energy-difference table, cached with the buffered endpoint \\
Loss & soft cross-entropy on masked sites \\
\midrule
\multicolumn{2}{@{}l}{\emph{Network}} \\
Backbone & periodic residual CNN, 3-D circular convolutions over the cell \\
Width, depth, kernel & 96 channels, 7 pre-norm residual blocks (GroupNorm), kernel 3 \\
Dilations & chosen so that the receptive field covers the periodic cell \\
Conditioning & Fourier features of normalized $(T,\DMu)$ (6 frequencies), 2-layer MLP, FiLM in every block \\
Output head & zero-initialized (uniform predictor at step 0) \\
\midrule
\multicolumn{2}{@{}l}{\emph{Optimization}} \\
Optimizer & AdamW \citep{loshchilov2019decoupled}, learning rate $3\times10^{-4}$, weight decay 0 \\
Batch size, steps & 1024, 60\,000 \\
Gradient clipping & global norm 1.0 \\
EMA of weights & decay 0.995 (the EMA model is the sampler) \\
\midrule
\multicolumn{2}{@{}l}{\emph{Fixed-point loop}} \\
Replay buffer & 25\,000 endpoints, pre-filled with $12\times2048$ uniform configurations \\
Exploration & 2048 fresh draws from the EMA model every 200 optimizer steps \\
Proximal damping $\eta$ & 0.5 (labels blended with the EMA model's prediction; $\eta=1$ is undamped) \\
\midrule
\multicolumn{2}{@{}l}{\emph{Importance weighting}} \\
Scheme & tempered weights, exponent and normalization per conditioning cell \\
Weight grid & $4\times4$ cells over $(T,\DMu)$ \\
Target ESS & 0.9 of the cell's buffered endpoints \\
Minimum occupancy & 24 buffered endpoints; thinner cells use the pooled exponent \\
\midrule
\multicolumn{2}{@{}l}{\emph{Conditioning acquisition}} \\
Uniform fraction & 0.5 \\
Grid, EMA, floor, power & $16\times16$ cells, moment EMA 0.9, uniform floor 0.05, score exponent 0.25 \\
\midrule
Evaluation & \AlloyNDraws{} draws per state point \\
\bottomrule
\end{tabular}
\end{table}

\begin{table}[h]
\caption{FrOGS hyperparameters, identical for Ag--Pd and Cu--Au.}
\label{tab:hp-frogs}
\centering
\small
\begin{tabular}{@{}>{\raggedright}p{4.6cm}p{8.2cm}@{}}
\toprule
\multicolumn{2}{@{}l}{\emph{Transport}} \\
Rate network & locally equivariant CNN: centre-masked convolution (kernel 3), then a state-dependent convolution (kernel 5); 24 channels; FiLM conditioning \\
Rates & Metropolis-reparametrized, clamped at 10 \\
Free-energy head & MLP, hidden width 200, input $(t,T,\DMu)$ \\
Path discretization & $\Delta t=0.02$ (50 CTMC steps) \\
Horizon curriculum & from 1 step to 50, quadratic over the first 2000 optimizer steps \\
\midrule
\multicolumn{2}{@{}l}{\emph{Prior}} \\
Network & PixelCNN-style causal 3-D convolutions \citep{vandenoord2016pixel}, 64 channels, 4 blocks, kernel 3, FiLM conditioning \citep{perez2018film} \\
\midrule
\multicolumn{2}{@{}l}{\emph{Training}} \\
Rollouts & 64 walkers every 20 optimizer steps; the buffer keeps the last 20 rollouts \\
Loss batch & 64 (time, walker) states drawn from one buffered rollout \\
MCMC moves during training & none \\
Optimizer & Adam \citep{kingma2015adam}; learning rate $3\times10^{-4}$ (transport and free-energy head), $10^{-3}$ (prior) \\
Schedule & step decay by 0.995 every 1000 steps \\
Gradient clipping, steps & global norm 1.0, 60\,000 \\
\midrule
Evaluation & \AlloyNDraws{} draws per state point, 50 CTMC steps, no MCMC moves \\
\bottomrule
\end{tabular}
\end{table}

\paragraph{Differences from the published FrOGS configuration.}
To keep the step budget and evaluation equal across methods, our runs differ
from \citet[App.~A.11, A.14]{min2026frogs}: 60k training steps rather than
200k, 50 CTMC steps rather than 125, no MCMC moves during training rather
than 10 per time step, a flat locally equivariant CNN as rate network rather
than a three-encoder, two-decoder pyramid with 16 channels, a prior with 64
channels rather than 32 and a prior learning rate of $10^{-3}$ rather than
$10^{-4}$, and a Cu--Au box that stops at 900\,K rather than 1200\,K.

\subsection{Discretized Biomolecular ensemble generator}\label{app:bg}

\subsubsection{Systems and energy}\label{app:bg-systems}
The two peptides are capped alanines, Ace--Ala--Nme (ALDP, 22 atoms) and
Ace--(Ala)$_3$--Nme (AL4, 42 atoms), with the amber99SB-ILDN force field
\citep{lindorff2010improved} and the OBC generalized-Born implicit solvent
\citep{onufriev2004exploring} at $T=300$\,K, without cutoffs or constraints.  They are
built and minimized with OpenMM \citep{eastman2024openmm}, and the force field is
reimplemented in PyTorch, batched on the GPU, where it matches OpenMM per force group
to $10^{-6}$\,kJ/mol: a Gibbs label needs $DL$ energies per configuration.  A molecule
of $n$ atoms has $D=3n-6$ coordinates $z$ after removing the rigid-body motion, and the
target is $\pi(z)\propto e^{-U(z)}$ with $U(z)=E(x(z))/\kB T-\ln|J(z)|$, where $x(z)$ are
the Cartesian positions and $J$ the Jacobian of the map.

\subsubsection{Discretization}\label{app:bg-disc}
Coordinate $d$ is confined to an interval $[c_d-h_d,c_d+h_d]$, split into $L$ bins of
width $2h_d/L$; dihedrals that span the full circle are binned periodically.  The discrete
target is the mass $\pi(b)\propto\int_{\mathrm{cell}(b)}e^{-U(z)}\,\dd z$ of every cell
of the grid.  A draw $b$ of \method{} is placed uniformly inside its cell
\citep{theis2016note}, so that $z$ has the density $q_\theta(b)/\mathrm{vol}$ with
$\mathrm{vol}=\prod_d 2h_d/L$, and
\begin{equation}
  \ln w(z) = -U(z)+\ln\mathrm{vol}-\ln q_\theta(b),
  \qquad
  \E_{z\sim q_\theta}[w(z)] = \int_{\mathrm{box}}e^{-U(z)}\,\dd z = Z,
  \label{eq:bg-weight}
\end{equation}
with $\ln q_\theta(b)$ the sum of \cref{eq:iw_mask} along the order of the draw.  The
weights are thus exact for the continuous Boltzmann density on the box.  The Gibbs label of
coordinate $d$ evaluates $-U$ at the $L$ bin centres $c_{d,l}$ with all other coordinates held
at the sample, $h^d(l\mid z^{-d})\propto\exp[-U(z^{d\leftarrow c_{d,l}})]$: the heat-bath
conditional of the cell-centre approximation of $\pi(b)$.  This approximation enters only the
training signal, not the weights.

\paragraph{Internal coordinates.}  A Z-matrix built from a breadth-first spanning tree of the
bond graph, rooted at the central C$_\alpha$.  Bond lengths and angles span $\pm10\sigma$
around the minimized structure, with $\sigma=(\kB T/k)^{1/2}$ from the force constant $k$ of
the term.  The backbone dihedrals $\phi$, $\psi$ and the methyl rotors span the full circle,
the peptide bonds $\omega$ span trans $\pm60^\circ$, and the branch dihedrals (the
out-of-plane positions of O, H, H$_\alpha$, C$_\beta$ and the methyl hydrogens) span
$\pm45^\circ$, which fixes the L-chirality and trans peptide bonds.  $\ln|J|=\sum 2\ln
r+\sum\ln\sin\theta$ over the bonds $r$ and angles $\theta$ of the Z-matrix, and $L=64$.
The bond-vector coordinates of \cref{app:bg-bondvec} are described there.

\begin{table}[t]
\caption{\method{} hyperparameters for the discretized biomolecular ensemble generator, identical for both
peptides; internal coordinates and the bond vectors of \cref{app:bg-bondvec} differ only where
stated.}
\label{tab:hp-bg}
\centering
\small
\begin{tabular}{@{}>{\raggedright}p{4.6cm}p{8.2cm}@{}}
\toprule
\multicolumn{2}{@{}l}{\emph{Discretization}} \\
Bins $L$ & 64 (internal), 512 (bond vectors) per coordinate \\
Pivot moves per fresh draw & 0 (internal), 1000 about $\phi$/$\psi$ bonds (bond vectors) \\
\midrule
\multicolumn{2}{@{}l}{\emph{Model}} \\
Network & bidirectional transformer, $d_{\mathrm{model}}=256$, 6 layers, 8 heads; per-coordinate bin embeddings with Fourier features of the bin centre; zero-initialized output head (uniform first rollout) \\
Generation & one coordinate per network evaluation, uniformly random order \\
\midrule
\multicolumn{2}{@{}l}{\emph{Fixed-point loop}} \\
Buffer & FIFO of 4096; 512 draws of the EMA model every 100 steps; warm start of 4 rollouts \\
Gibbs labels & $DL$ energies per configuration, float32 \\
Masking & $\alpha\sim\mathcal U(0,1-1/D)$, each coordinate masked with probability $1-\alpha$, at least one \\
Damping $\eta$ & 0.5, towards the EMA model (\cref{app:damping}) \\
Importance weights & tempered to an $\ESS$ of 0.8 within each training batch \\
\midrule
\multicolumn{2}{@{}l}{\emph{Optimization}} \\
Optimizer & AdamW \citep{loshchilov2019decoupled}, learning rate $3\times10^{-4}$, no weight decay, gradient clipping at 1 \\
Batch, steps, EMA & 256, 150k, 0.995 \\
\midrule
Evaluation & \PepDiNSamples{} draws per run, independent of training \\
Seeds, compute & \PepNSeeds{} per setting; one H100, H200 or A100 GPU per run \\
\bottomrule
\end{tabular}
\end{table}

\subsubsection{Training}\label{app:bg-train}
\Cref{tab:hp-bg} lists the settings, which are the same for both peptides and both coordinate
systems.  Every 100 gradient steps, the EMA model draws 512 configurations, which
enter a buffer of 4096 together with their Gibbs labels and log-weights.  A gradient step draws a
batch of 256 buffered configurations, masks a random subset of each, damps the labels towards
the EMA model's prediction, and minimizes the soft cross-entropy at the masked coordinates, with
the importance weights tempered to an $\ESS$ of 0.8 within the batch.

\paragraph{Maximum-likelihood baseline.}\label{app:bg-mle}  The baseline of \cref{sec:bg} (MLE
on MD) has the same discretization and network as \method{} and is trained with the same
optimizer, batch size, number of steps and masking, but on MD frames instead of the energy.  Each
frame is binned, and the loss is the cross-entropy at the masked coordinates against the frame's
own bins, with uniform weights and no damping.  It needs no energy evaluations during training,
only the MD data.  The training set is every tenth frame of the reference of \cref{app:bg-eval},
\PepDiNMDTrain{} frames for ALDP and \PepTetraNMDTrain{} for AL4, of which more than
$99.9\,\%$ lie inside the box; the others are dropped.  The baseline's samples are evaluated
exactly like those of \method{}, including the importance weights of \cref{eq:bg-weight}.  Its
training frames are part of the reference, which can only favour its metrics.

\subsubsection{Reference and evaluation}\label{app:bg-eval}
The MD reference uses batched BAOAB Langevin dynamics \citep{leimkuhler2013rational} on the same
PyTorch force field, with parallel tempering between 300 and 600\,K
\citep{sugita1999replica}: 6 temperatures for ALDP and 8 for AL4, each with 64 independent
copies.  Steps are 2\,fs with hydrogen mass repartitioning to 4\,amu
\citep{hopkins2015long}, and a frame is saved every 1\,ps from the 300\,K replicas, giving
\PepDiNFrames{} frames for ALDP and \PepTetraNFrames{} for AL4.  Each trained run draws
\PepDiNSamples{} samples, about the size of the reference, so that histogram metrics compare
equal sample sizes.  With the self-normalized weights $\bar w$ of \cref{eq:bg-weight},
\begin{itemize}
  \item $\ESS$ is \cref{eq:ess};
  \item $\mathcal W_1$ is the Wasserstein-1 distance between the weighted empirical distributions
  of the potential energy of the samples and of the MD frames, with samples capped 200\,kJ/mol
  above the highest MD energy so that a few clashes cannot dominate;
  \item JS is the Jensen--Shannon divergence between the weighted histograms of $(\phi,\psi)$ of
  the samples and of MD on $36\times36$ cells of $[-\pi,\pi)^2$, averaged over the residues;
  \item the basin fractions of a residue are the weighted fractions of $\alpha_R$ ($\phi<0$,
  $-120^\circ<\psi<50^\circ$), $\beta$/PPII (the rest of $\phi<0$) and $\alpha_L$ ($\phi\ge0$);
  \item $\ln Z$ is \cref{eq:fe_absolute} with the weights of \cref{eq:bg-weight}.
\end{itemize}
The quantities ``as drawn'' use uniform weights instead of $\bar w$.  The Ramachandran surfaces
pool the three seeds, each seed's weights normalized to one third, and show $\ln p(\phi,\psi)$ on
$60\times60$ cells relative to its maximum, clipped at $-\PepFECap$.

\subsection{Free energy estimation}\label{app:fe}

\subsubsection{Tasks}\label{app:fe-tasks}
The two tasks are those of FEAT \citep{he2025feat}, on alanine dipeptide as built by
openmmtools: the amber ff96 force field with the OBC generalized-Born solvent, at 300\,K and
without cutoffs or constraints, implemented in PyTorch where it matches OpenMM
\citep{eastman2017openmm} per force group to $10^{-6}$\,kJ/mol.  A parameter $\lambda$
scales the charges of the solvent term, so that its polar energy scales as $\lambda^2$;
$\lambda=0$ is vacuum and $\lambda=1$ implicit solvent.  The MD data are FEAT's, 250k vacuum
and 210k solvent frames.
\begin{itemize}
  \item \textbf{ALDP-S}: endpoint 0 is vacuum, endpoint 1 is solvent, and
  $\ln Z_1/Z_0$ is minus the solvation free energy in units of $\kB T$.  All vacuum frames
  are the given samples.
  \item \textbf{ALDP-T}: both endpoints are vacuum ALDP, restricted to the conformational
  states $S_b=\{\phi\notin(0,2.15)\}$ (C7eq and C5) and $S_a=\{\phi\in(0,2.15)\}$ (C7ax),
  i.e.\ $U+\infty$ outside the state.  The \FETNPaths{} vacuum frames in $S_b$ are the given
  samples, and the \FETNRef{} frames in $S_a$ enter only the evaluation.
\end{itemize}
The coordinates are the internal coordinates of \cref{app:bg-disc}, with the boxes set by the
ff96 force constants and the solvent minimum, $D=60$ and $L=64$.  The boxes contain 99.98\,\%
of the frames of either MD set.

\subsubsection{The bridge}\label{app:fe-bridge}
\paragraph{Interpolant and labels.}  The interpolant of \cref{sec:general_interpolants} is the
coordinate-wise mixture $\kappa_t^d(c\mid x_0,x_1)=(1-t)\one_{\{c=x_0^d\}}+t\one_{\{c=x_1^d\}}$
under the independent coupling $p_0\otimes p_1$.  For an endpoint $X_1$ of the current sampler,
the labels of \cref{sec:general_interpolants} become
$\ell_{t,d}(c)=(1-t)\one_{\{c=X_0^d\}}+t\,h_1^d(c\mid X_1^{-d})$ and
$\dot\ell_{t,d}(c)=h_1^d(c\mid X_1^{-d})-\one_{\{c=X_0^d\}}$, where $h_1^d$ is the Gibbs label of
endpoint 1, a scan of $DL$ energies.  For ALDP-T the restriction to $S_a$ depends on $\phi$
alone, so it sets the label of $\phi$ to zero on the bins outside $S_a$ and leaves the others
unchanged.

\paragraph{Network.}  $p_{\theta,t}^d(\cdot\mid x^{-d})$ and $\phi_{\psi,t}^d(\cdot\mid x^{-d})$
are two heads of one hollow transformer \citep{sun2022score}: a forward and a backward causal
stream over the 60 coordinates, 4 layers each, and 2 readout blocks in which output $d$
attends only to stream outputs that do not depend on $x^d$.  The time enters through Fourier
features added to every input.  Both heads are zero-initialized and the Gibbs rate below is off
before the first gradient step, so the warm-start rollouts are the identity transport.

\paragraph{Sampler and weights.}  The minimum-action CTMC of \cref{sec:general_interpolants}
with an added Gibbs rate $g$: a jump of coordinate $d$ from $a$ to $c$ has the forward rate
$p_t^d(c)\,([\phi_t^d(c)-\phi_t^d(a)]_++g)$ and the backward rate
$p_t^d(c)\,([\phi_t^d(a)-\phi_t^d(c)]_++g)$.  The term in $g$ is a reversible heat-bath move
under $p_t^d$, which leaves the marginals unchanged.  In the $\tau$-leaped sampler it is needed
for unbiased weights: with $g=0$ a jump downhill in $\phi_t$ has zero forward probability, yet
the backward kernel proposes its reverse, so that this mass is missing from $\E[w]$.  We use
$K=256$ steps and $g=0.1$.  A path starts at an MD frame $z_0$, its endpoint bins are placed
uniformly in their cells, and its log-weight is \cref{eq:iw_general_discrete}, i.e.\
$-U_1(z_1)+U_0(z_0)$ plus the log-ratio of the backward and forward kernels, whose expectation
is $Z_1/Z_0$ for any network.

\paragraph{Training.}  Every 100 steps the EMA model transports 1024 training frames, and their
endpoints, labels and path log-weights enter a buffer of 16{,}384.  The tempering exponent is
chosen on the whole buffer so that $\ESS(w^\nu)\ge0.5$.  A gradient step draws 512 buffered
endpoints with independent MD frames $X_0$ and $t\sim\mathcal U(0,1)$, samples
$X_t\sim K_t(\cdot\mid X_0,X_1)$, and fits both heads with the losses of
\cref{sec:general_interpolants}, weighted by $w^\nu$ normalized to mean one, with the labels
damped towards the EMA model ($\eta=0.3$).  On ALDP-T an endpoint outside $S_a$ has
$\ln w=-\infty$, which is exact and enters the estimate.  For the choice of $\nu$ alone, it
counts as the lowest finite $\ln w$ minus 20, because a buffer with few endpoints in $S_a$ would
otherwise drive $\nu$ to zero and drop the labels that pull $\phi$ into $S_a$.  The settings
(\cref{tab:hp-fe}) are shared by both tasks.  They come from a $2\times2$ sweep over
$\eta\in\{0.3,1\}$ and target $\ESS\in\{0.3,0.5\}$ on one seed, in which all four settings
agreed within their spread; we kept the one with the smallest spread on ALDP-T.

\begin{table}[t]
\caption{\method{} hyperparameters of the bridge, identical for ALDP-S and ALDP-T.}
\label{tab:hp-fe}
\centering
\small
\begin{tabular}{@{}>{\raggedright}p{4.6cm}p{8.2cm}@{}}
\toprule
Interpolant & coordinate-wise mixture, independent coupling \\
Bins, coordinates & $L=64$, $D=60$ internal coordinates \\
Network & hollow transformer, $d_{\mathrm{model}}=256$, $2\times4$ causal layers and 2 readout blocks, 11.1M parameters; heads $p_\theta$ and $\phi_\psi$, zero-initialized \\
Sampler & minimum-action CTMC with Gibbs rate $g=0.1$, $\tau$-leaping with $K=256$ steps \\
Buffer & 16{,}384 endpoints; 1024 paths of the EMA model every 100 steps; warm start of 4 rollouts \\
Importance weights & tempered to an $\ESS$ of 0.5 on the whole buffer \\
Damping $\eta$ & 0.3 \\
Optimizer & AdamW, learning rate $3\times10^{-4}$, gradient clipping at 1 \\
Batch, steps, EMA & 512, 150k, 0.997 \\
Train/test split & 90/10 of the endpoint-0 frames, the same for every seed \\
Evaluation & 8192 held-out paths every 5000 steps; final: one path per endpoint-0 frame \\
Seeds, compute & \FENSeeds{} per task; one H100, H200 or A100 GPU per run \\
\bottomrule
\end{tabular}
\end{table}

\subsubsection{Evaluation and references}\label{app:fe-reference}
The endpoint-0 frames are split once into 90\,\% for training and 10\,\% held out, the same split
for every seed.  The final evaluation draws one path per endpoint-0 frame, \FESNPaths{} for
ALDP-S and \FETNPaths{} for ALDP-T, starting from the \FESNStart{} and \FETNStart{} held-out
frames in turn, so that each is carried about ten times with independent noise.  The paths of
one starting frame share $z_0$, so the standard error treats every starting frame as one unit, a
ratio estimator over frames; it does not account for the correlation between MD frames.  The
endpoint metrics compare with every endpoint-1 MD frame, \FESNRef{} solvent frames and
\FETNRef{} frames in $S_a$, as in \cref{app:bg-eval}.

For ALDP-S, the reference is BAR \citep{bennett1976efficient} on the two MD sets with our
energy, \FESRef, with the standard error of a bootstrap over frames, which like that of the
estimates ignores the correlation between MD frames.  For ALDP-T, the two states have disjoint
support, so neither BAR nor reweighting applies.  The frame counts of the vacuum MD give
$\ln(\text{\FETNRef}/\text{\FETNPaths})=-4.18$, but this MD rarely crosses between the states.
The reference is therefore umbrella sampling along $\phi$ \citep{torrie1977nonphysical}, 40
windows of 2\,ns each, combined by MBAR \citep{shirts2008statistically}, which gives \FETRef.
It is itself not fully converged: near the C7eq/C7ax barrier the windows sample $\psi$ slowly,
and the two time halves of one set of windows give $-4.25$ and $-5.00$.  The values of FEAT in
\cref{fig:fe} are those reported by \citet{he2025feat}, mean $\pm$ standard deviation over three
seeds, not rerun by us.

\subsection{Maximum entropy reinforcement learning}\label{app:rl}

\subsubsection{Tasks and evaluation}\label{app:rl-setup}
The \RLNTasks{} tasks are the DeepMind Control tasks of MuJoCo Playground
\citep{mujoco_playground_2025} of the online maximum entropy RL benchmark of
\citet{trudi2025}.  Their action dimensions range from $m=1$ to $m=6$
(\cref{tab:rl-tasks}) and their returns lie in $[0,1000]$.  Every method trains
from scratch for 55M environment steps with 1024 parallel environments and is
evaluated every 2.6M steps on 1024 episodes of 1000 steps.  \method{} is
evaluated greedily: it reveals the dimensions in a random order and takes the
most likely bin at every step.  We report the IQM of the return divided by 1000
over all $\RLNTasks\times\RLNSeeds$ runs, with $95\%$ intervals from a stratified
bootstrap over the seeds within each task \citep{agarwal2021deep}.  The curves,
intervals and final IQMs of the baselines, REPPO, PPO, SPO, TRUDI, AMDP, DPPO, FPO
and FFM, are those reported by \citet{trudi2025}, read off its figures.

\subsubsection{Policy, critic and training}\label{app:rl-method}

\paragraph{Binned actions.}
Each dimension of the action $a\in[-1,1]^m$ is split into $L$ bins of width
$\Delta=2/L$ with centres $c_l=-1+(l-\tfrac12)\Delta$, $l=1,\dots,L$.  A
configuration of bins $x\in\{1,\dots,L\}^m$ is executed as the action
$a^d=c_{x^d}+u^d$ with in-bin noise $u^d\sim\mathcal U(-\Delta/2,\Delta/2)$.  A
distribution over configurations thus defines a piecewise constant density over
actions, whose entropy is that of the configuration plus $m\log\Delta$.  We report
entropies in these action-density units, so that the entropy target of FFM
carries over unchanged.

\paragraph{Policy.}
The masked network receives the state $s$ and the revealed dimensions $R$ through
the input $[\,s,\ r,\ r\odot c(x),\ r\odot\mathrm{onehot}_L(x)\,]$, where
$r\in\{0,1\}^m$ indicates the revealed dimensions and $c(x)$ are their bin
centres.  Three blocks $\mathrm{Linear}(\cdot\to512)$--LayerNorm--GELU are followed
by a linear head with $mL$ outputs, the logits of $\mu_\theta^d(\cdot\mid s,x_R)$
for every dimension $d$.  The kernel of the head is scaled by $10^{-8}$ at
initialization, so the initial policy is uniform over the bins.  An action is
generated as in \cref{sec:background}: the dimensions are revealed one per network
evaluation in a uniformly random order, so an action costs $m$ network
evaluations, and its entropy along the sampled order, the sum of the entropies of
the $m$ categoricals, is computed exactly.

\begin{table}[t]
\caption{\method{} hyperparameters for maximum entropy RL, identical for all tasks.}
\label{tab:hp-rl}
\centering
\small
\begin{tabular}{@{}>{\raggedright}p{4.6cm}p{8.2cm}@{}}
\toprule
\multicolumn{2}{@{}l}{\emph{Environment loop}} \\
Parallel environments, rollout length & 1024, 128 steps ($131\,072$ transitions per iteration) \\
Iterations, environment steps & 420, 55M \\
Discount $\gamma$, TD $\lambda$ & 0.99, 0.95 \\
\midrule
\multicolumn{2}{@{}l}{\emph{Policy}} \\
Bins $L$ & 64 per action dimension \\
Reveal order & uniform over permutations, one dimension per network evaluation \\
Network & MLP, 3 blocks of width 512 with LayerNorm and GELU; linear head initialized at the uniform policy \\
\midrule
\multicolumn{2}{@{}l}{\emph{Critic}} \\
Critic & categorical over 151 atoms on $[0,150]$, HL-Gauss targets ($\sigma$ of 0.75 atoms) \\
Target critic & a copy of the critic after every update \\
\midrule
\multicolumn{2}{@{}l}{\emph{Fixed-point loop}} \\
Buffer & the $131\,072$ transitions of the current iteration \\
Epochs $\times$ minibatches & $4\times128$, minibatch of 1024 transitions (policy and critic) \\
Gibbs labels & one site per transition and epoch, $L$ critic evaluations at the bin centres \\
Importance weights & none ($\nu=0$) \\
Trust region & KL bound 0.1 nats per action, initial multiplier 1 \\
Temperature & target entropy $-0.1\,m$, initial value 0.01, range $[10^{-8},1]$ \\
\midrule
\multicolumn{2}{@{}l}{\emph{Optimization}} \\
Optimizer & Adam \citep{kingma2015adam}, learning rate $3\times10^{-4}$ (policy, critic and dual variables) \\
Gradient clipping & global norm 0.5 (critic) \\
\midrule
Evaluation & every 20 iterations (2.6M steps), 1024 episodes, greedy decoding \\
Seeds, compute & \RLNSeeds{} per task; one to two hours per run on a single A100 or V100 GPU \\
\bottomrule
\end{tabular}
\end{table}

\paragraph{Critic.}
The critic $Q_\phi(s,a)$ is that of FFM, a categorical critic on the continuous
action, trained by cross-entropy to HL-Gauss targets
\citep{pmlr-v235-farebrother24a} over 151 atoms on $[0,150]$.  Its targets are
soft TD($\lambda$) returns with the reward $r_t+\gamma\tau\hat H_{t+1}$, where
$\hat H_{t+1}$ is the entropy of the policy along the order sampled at $s_{t+1}$,
so that the critic estimates the soft action value of the current policy at the
current temperature.

\paragraph{Fixed-point loop.}
Each iteration collects $1024\times128$ transitions with the current policy, which
replace the buffer, fits the critic on them, and then fits the policy for 4 epochs
of 128 minibatches of 1024 transitions.  Every transition stores the order in which
its action was generated.  For every transition in every epoch we draw one step $k$
of this order uniformly and one dimension $d$ uniformly among those still masked at
step $k$, compute the Gibbs label $h^d(\cdot\mid s,a^{-d})$ from $L$ critic
evaluations at the bin centers, and minimize its cross-entropy with
$\mu_\theta^d(\cdot\mid s,x_R)$, where $R$ holds the dimensions revealed before step
$k$.  This is a one-site Monte Carlo estimate of \cref{eq:gibbs_ce} with a
different weighting of the reveal sets and the same minimizer.  Importance weights
are not used ($\nu=0$): every state is visited once, so weights that are
self-normalized per state are identically one, and weighting across states would
require the normalizing constant $\sum_x\exp(Q_\phi(s,x)/\tau)$ of every state.
Since the critic and the temperature change between iterations, so does the target
$\pi(\cdot\mid s)\propto\exp(Q_\phi(s,\cdot)/\tau)$, and the convergence analysis of
\cref{sec:convergence} does not apply to this setting.

\paragraph{Entropy control and damping.}
In the lattice and alloy experiments the temperature is part of the target.  In RL
it is set by the algorithm: as in soft actor--critic \citep{haarnoja2018sac} and
FFM, $\tau$ is the Lagrange multiplier of the constraint that the entropy of the
policy stays above a target,
\begin{equation}
  \bar H\ge H^\star,\qquad H^\star=-0.1\,m,
  \label{eq:rl-entropy}
\end{equation}
where $\bar H$ is the entropy along the stored orders, averaged over a minibatch, in
action-density units.  We parameterize $\tau=10^{-8}+\exp(\log\tau)$ and take an
Adam step on $\log\tau$ for the loss $\tau\,(\bar H-H^\star)$, with $\bar H$ held
fixed, after every minibatch.  The temperature thus rises while the policy is
narrower than the target and falls while it is wider.  At its cap $\tau=1$ it may
only decrease, so that it does not wind up while the target entropy is out of
reach.  Since $\tau$ enters both the labels and the soft returns, the target
$\pi(\cdot\mid s)\propto\exp(Q_\phi(s,\cdot)/\tau)$ is the Boltzmann distribution
of the critic at the temperature at which the policy has the target entropy.

Damping (\cref{app:damping}) stabilizes the fixed-point iteration by blending the
Gibbs labels with the predictions of the sampler that generated the data, which
moves the predictor a fraction $\eta$ of the way towards the undamped update.  In
RL we do not damp the labels but, as in FFM, bound the update of the policy
directly by a trust region around the policy $\mu_{\theta_i}$ that collected the
transitions,
\begin{equation}
  \min_\theta\ \mathcal L_{\mathrm{Gibbs}}(\theta)
  \quad\text{s.t.}\quad
  \mathbb E_s\big[\mathrm{KL}\big(\mu_{\theta_i}(\cdot\mid s)\,\big\|\,\mu_\theta(\cdot\mid s)\big)\big]\le\epsilon,
  \qquad\epsilon=0.1\text{ nats per action}.
  \label{eq:rl-trust-region}
\end{equation}
For a stored action, the KL along the order in which it was generated is the sum
over the $m$ steps of the KLs between the categoricals of $\mu_{\theta_i}$ and
$\mu_\theta$ at the visited contexts, each computed in closed form over the $L$
bins.  The constraint is enforced with a Lagrange multiplier
$\kappa=\exp(\log\kappa)$: the policy minimizes
$\mathcal L_{\mathrm{Gibbs}}+\kappa\,\mathrm{KL}$, and $\log\kappa$ takes an Adam
step on $-\kappa\,(\mathrm{KL}-\epsilon)$ after every minibatch, so that $\kappa$
grows while the update exceeds the budget and shrinks otherwise.  Like damping, the
trust region leaves the fixed points of the iteration unchanged: when the policy
already matches the labels, the unconstrained minimizer is $\mu_{\theta_i}$ itself
and the constraint is inactive.  Unlike damping with a fixed $\eta$, it sets the
size of every update by a KL budget rather than by a fixed fraction of the step.

\subsubsection{Hyperparameters}\label{app:rl-hparams}
\Cref{tab:hp-rl} lists the settings of \method{}, which are identical for all tasks.
All settings that are not listed, including the architecture of the critic, are
those of FFM \citep{ffm2026}.

\FloatBarrier
\section{Further experiments}\label{app:further}

\subsection*{Component ablation}\label{app:ablation}
\method{} adds three ingredients to the masked interpolant: heat-bath (Gibbs)
labels, importance weights tempered to a target $\ESS$, and proximal damping.
We remove one at a time on Ising $25\times25$ at the settings of
\cref{tab:lattice}: `no damping' sets $\eta=1$, `no Gibbs score' trains on the
endpoint's own token, once with tempered weights and once at full strength
(`no smoothing'), and `no importance weights' weights every endpoint equally.
Away from the ordered phase all variants converge (\cref{fig:abl,tab:abl}).
In the ordered phase only the full method is both mode-complete and accurate:
without damping \AblNoDampSingleMode{} of \AblNSeeds{} seeds collapse to one
mode, without the Gibbs score all seeds collapse at full weights and the run
never orders at tempered ones, and without importance weights both modes
survive but $|\Delta\ln Z|$ grows from \AblBaseOrdZ{} to \AblNoIWOrdZ{}
$\times10^{-3}$ nats.

\begin{figure}[h]
\centering
\tikzsetnextfilename{abl}
\begin{tikzpicture}
\begin{groupplot}[
    group style={group size=3 by 3, horizontal sep=1.0cm, vertical sep=0.6cm,
                 x descriptions at=edge bottom},
    ablpanel, height=2.1cm, xlabel={gradient steps}]
  \nextgroupplot[title={$\ESS\ \uparrow$}, ylabel={ordered}, ymin=0, ymax=1, ytick={0,0.5,1},
                 legend to name=abllegend, legend columns=3,
                 legend style={font=\scriptsize, draw=none, column sep=0.7em}]
    \addplot[ablline, cDSI] table[x=step, y=ess] {data/lattice/abl_Base_ordered.dat};
    \addlegendentry{\method{} (baseline)}
    \addplot[ablline, cNoDamp] table[x=step, y=ess] {data/lattice/abl_NoDamp_ordered.dat};
    \addlegendentry{no damping}
    \addplot[ablline, cNoGibbsT] table[x=step, y=ess] {data/lattice/abl_NoGibbsT_ordered.dat};
    \addlegendentry{no Gibbs score}
    \addplot[ablline, cNoGibbsP] table[x=step, y=ess] {data/lattice/abl_NoGibbsP_ordered.dat};
    \addlegendentry{no Gibbs score, no smoothing}
    \addplot[ablline, cNoIW] table[x=step, y=ess] {data/lattice/abl_NoIW_ordered.dat};
    \addlegendentry{no importance weights}
  \nextgroupplot[title={$|\Delta\ln Z|\ \downarrow$}, ymode=log, ymin=1e-4, ymax=1e3, ytick={1e-4,1e-2,1,100}]
    \addplot[ablline, cDSI] table[x=step, y=z] {data/lattice/abl_Base_ordered.dat};
    \addplot[ablline, cNoDamp] table[x=step, y=z] {data/lattice/abl_NoDamp_ordered.dat};
    \addplot[ablline, cNoGibbsT] table[x=step, y=z] {data/lattice/abl_NoGibbsT_ordered.dat};
    \addplot[ablline, cNoGibbsP] table[x=step, y=z] {data/lattice/abl_NoGibbsP_ordered.dat};
    \addplot[ablline, cNoIW] table[x=step, y=z] {data/lattice/abl_NoIW_ordered.dat};
    \draw[black!55, densely dashed, line width=0.6pt]
      ({rel axis cs:0,0} |- {axis cs:0,0.6931}) --
      ({rel axis cs:1,0} |- {axis cs:0,0.6931});
    \node[anchor=south east, font=\tiny, text=black!55, inner sep=1.5pt]
      at (rel axis cs:1,0.62) {$\ln 2$};
  \nextgroupplot[title={$\mTVD\ \downarrow$}, ymode=log, ymin=5e-4, ymax=1, ytick={1e-3,1e-2,1e-1,1}]
    \addplot[ablline, cDSI] table[x=step, y=tv] {data/lattice/abl_Base_ordered.dat};
    \addplot[ablline, cNoDamp] table[x=step, y=tv] {data/lattice/abl_NoDamp_ordered.dat};
    \addplot[ablline, cNoGibbsT] table[x=step, y=tv] {data/lattice/abl_NoGibbsT_ordered.dat};
    \addplot[ablline, cNoGibbsP] table[x=step, y=tv] {data/lattice/abl_NoGibbsP_ordered.dat};
    \addplot[ablline, cNoIW] table[x=step, y=tv] {data/lattice/abl_NoIW_ordered.dat};
    \draw[black!55, densely dashed, line width=0.6pt]
      ({rel axis cs:0,0} |- {axis cs:0,0.5}) --
      ({rel axis cs:1,0} |- {axis cs:0,0.5});
    \node[anchor=north east, font=\tiny, text=black!55, inner sep=1.5pt]
      at (rel axis cs:1,0.87) {one mode};
  \nextgroupplot[ylabel={critical}, ymin=0, ymax=1, ytick={0,0.5,1}]
    \addplot[ablline, cDSI] table[x=step, y=ess] {data/lattice/abl_Base_critical.dat};
    \addplot[ablline, cNoDamp] table[x=step, y=ess] {data/lattice/abl_NoDamp_critical.dat};
    \addplot[ablline, cNoGibbsT] table[x=step, y=ess] {data/lattice/abl_NoGibbsT_critical.dat};
    \addplot[ablline, cNoGibbsP] table[x=step, y=ess] {data/lattice/abl_NoGibbsP_critical.dat};
    \addplot[ablline, cNoIW] table[x=step, y=ess] {data/lattice/abl_NoIW_critical.dat};
  \nextgroupplot[ymode=log, ymin=1e-4, ymax=1e3, ytick={1e-4,1e-2,1,100}]
    \addplot[ablline, cDSI] table[x=step, y=z] {data/lattice/abl_Base_critical.dat};
    \addplot[ablline, cNoDamp] table[x=step, y=z] {data/lattice/abl_NoDamp_critical.dat};
    \addplot[ablline, cNoGibbsT] table[x=step, y=z] {data/lattice/abl_NoGibbsT_critical.dat};
    \addplot[ablline, cNoGibbsP] table[x=step, y=z] {data/lattice/abl_NoGibbsP_critical.dat};
    \addplot[ablline, cNoIW] table[x=step, y=z] {data/lattice/abl_NoIW_critical.dat};
  \nextgroupplot[ymode=log, ymin=5e-4, ymax=1, ytick={1e-3,1e-2,1e-1,1}]
    \addplot[ablline, cDSI] table[x=step, y=tv] {data/lattice/abl_Base_critical.dat};
    \addplot[ablline, cNoDamp] table[x=step, y=tv] {data/lattice/abl_NoDamp_critical.dat};
    \addplot[ablline, cNoGibbsT] table[x=step, y=tv] {data/lattice/abl_NoGibbsT_critical.dat};
    \addplot[ablline, cNoGibbsP] table[x=step, y=tv] {data/lattice/abl_NoGibbsP_critical.dat};
    \addplot[ablline, cNoIW] table[x=step, y=tv] {data/lattice/abl_NoIW_critical.dat};
  \nextgroupplot[ylabel={disordered}, ymin=0, ymax=1, ytick={0,0.5,1}]
    \addplot[ablline, cDSI] table[x=step, y=ess] {data/lattice/abl_Base_disordered.dat};
    \addplot[ablline, cNoDamp] table[x=step, y=ess] {data/lattice/abl_NoDamp_disordered.dat};
    \addplot[ablline, cNoGibbsT] table[x=step, y=ess] {data/lattice/abl_NoGibbsT_disordered.dat};
    \addplot[ablline, cNoGibbsP] table[x=step, y=ess] {data/lattice/abl_NoGibbsP_disordered.dat};
    \addplot[ablline, cNoIW] table[x=step, y=ess] {data/lattice/abl_NoIW_disordered.dat};
  \nextgroupplot[ymode=log, ymin=1e-4, ymax=1e3, ytick={1e-4,1e-2,1,100}]
    \addplot[ablline, cDSI] table[x=step, y=z] {data/lattice/abl_Base_disordered.dat};
    \addplot[ablline, cNoDamp] table[x=step, y=z] {data/lattice/abl_NoDamp_disordered.dat};
    \addplot[ablline, cNoGibbsT] table[x=step, y=z] {data/lattice/abl_NoGibbsT_disordered.dat};
    \addplot[ablline, cNoGibbsP] table[x=step, y=z] {data/lattice/abl_NoGibbsP_disordered.dat};
    \addplot[ablline, cNoIW] table[x=step, y=z] {data/lattice/abl_NoIW_disordered.dat};
  \nextgroupplot[ymode=log, ymin=5e-4, ymax=1, ytick={1e-3,1e-2,1e-1,1}]
    \addplot[ablline, cDSI] table[x=step, y=tv] {data/lattice/abl_Base_disordered.dat};
    \addplot[ablline, cNoDamp] table[x=step, y=tv] {data/lattice/abl_NoDamp_disordered.dat};
    \addplot[ablline, cNoGibbsT] table[x=step, y=tv] {data/lattice/abl_NoGibbsT_disordered.dat};
    \addplot[ablline, cNoGibbsP] table[x=step, y=tv] {data/lattice/abl_NoGibbsP_disordered.dat};
    \addplot[ablline, cNoIW] table[x=step, y=tv] {data/lattice/abl_NoIW_disordered.dat};
\end{groupplot}
\end{tikzpicture}\\[0.3em]
\tikzexternaldisable\pgfplotslegendfromname{abllegend}\tikzexternalenable
\caption{\textbf{Component ablation on Ising $25\times25$.}  $\ESS$,
$|\Delta\ln Z|$ and $\mTVD$ against training steps for the ordered (top),
critical (middle) and disordered phase (bottom), mean over \AblNSeeds{}
seeds.  Dashed: $\ln 2$ and $\mTVD=0.5$, the values of a sampler that holds
only one of the two ordered modes.}
\label{fig:abl}
\end{figure}

\begin{table}[h]
\caption{\textbf{Component ablation, final evaluation.}  Mean${\pm}$standard
deviation over \AblNSeeds{} seeds, $|\Delta\ln Z|$ in $10^{-3}$ nats.  PDNS
and MDNS are repeated from \cref{tab:lattice}.}
\label{tab:abl}
\centering
\scriptsize
\setlength{\tabcolsep}{1.5pt}
\begin{tabular}{@{}l ccc@{\hspace{2pt}} ccc@{\hspace{2pt}} ccc@{}}
\toprule
& \multicolumn{3}{c}{disordered, $\beta=0.30$} & \multicolumn{3}{c}{critical, $\beta=\beta_c=0.4407$} & \multicolumn{3}{c}{ordered, $\beta=0.60$} \\
\cmidrule(lr){2-4}\cmidrule(lr){5-7}\cmidrule(lr){8-10}
& $|\Delta\ln Z|$ & $\ESS$ & $\mTVD$ & $|\Delta\ln Z|$ & $\ESS$ & $\mTVD$ & $|\Delta\ln Z|$ & $\ESS$ & $\mTVD$ \\
variant & $\times10^{3}\,{\downarrow}$ & ${\uparrow}$ & ${\downarrow}$ & $\times10^{3}\,{\downarrow}$ & ${\uparrow}$ & ${\downarrow}$ & $\times10^{3}\,{\downarrow}$ & ${\uparrow}$ & ${\downarrow}$ \\
\midrule
\method{} (baseline)              & \AblBaseDisZ         & \AblBaseDisESS       & \AblBaseDisMTV       & \AblBaseCritZ        & \AblBaseCritESS      & \AblBaseCritMTV      & \AblBaseOrdZ         & \AblBaseOrdESS       & \AblBaseOrdMTV \\
no damping                  & \AblNoDampDisZ       & \AblNoDampDisESS     & \AblNoDampDisMTV     & \AblNoDampCritZ      & \AblNoDampCritESS    & \AblNoDampCritMTV    & \AblNoDampOrdZ       & \AblNoDampOrdESS     & \AblNoDampOrdMTV \\
no Gibbs score              & \AblNoGibbsTDisZ     & \AblNoGibbsTDisESS   & \AblNoGibbsTDisMTV   & \AblNoGibbsTCritZ    & \AblNoGibbsTCritESS  & \AblNoGibbsTCritMTV  & \AblNoGibbsTOrdZ     & \AblNoGibbsTOrdESS   & \AblNoGibbsTOrdMTV \\
no Gibbs, no smoothing  & \AblNoGibbsPDisZ     & \AblNoGibbsPDisESS   & \AblNoGibbsPDisMTV   & \AblNoGibbsPCritZ    & \AblNoGibbsPCritESS  & \AblNoGibbsPCritMTV  & \AblNoGibbsPOrdZ     & \AblNoGibbsPOrdESS   & \AblNoGibbsPOrdMTV \\
no importance weights       & \AblNoIWDisZ         & \AblNoIWDisESS       & \AblNoIWDisMTV       & \AblNoIWCritZ        & \AblNoIWCritESS      & \AblNoIWCritMTV      & \AblNoIWOrdZ         & \AblNoIWOrdESS       & \AblNoIWOrdMTV \\
\addlinespace
\multicolumn{10}{@{}l}{\emph{reference, from \cref{tab:lattice}}} \\
PDNS                        & \AblPDNSDisZ         & \AblPDNSDisESS       & \AblPDNSDisMTV       & \AblPDNSCritZ        & \AblPDNSCritESS      & \AblPDNSCritMTV      & \AblPDNSOrdZ         & \AblPDNSOrdESS       & \AblPDNSOrdMTV \\
MDNS                        & \AblMDNSDisZ         & \AblMDNSDisESS       & \AblMDNSDisMTV       & \AblMDNSCritZ        & \AblMDNSCritESS      & \AblMDNSCritMTV      & \AblMDNSOrdZ         & \AblMDNSOrdESS       & \AblMDNSOrdMTV \\
\bottomrule
\end{tabular}
\end{table}

\subsection{Bridging two Boltzmann distributions}\label{app:bridge}
The interpolant does not need a tractable reference: here both ends are Ising
distributions $\pi_\beta(s)\propto e^{\beta\sum_{\langle ij\rangle}s_is_j}$ on
a $\BrL\times\BrL$ torus at different $\beta$, and neither can be sampled
exactly.  The sampler is given draws from $\pi_{\beta_0}$, produced by MCMC,
and the energy of $\pi_{\beta_1}$, and learns a CTMC that carries one into the
other; path weights give an unbiased estimate of
$Z(\beta_1)/Z(\beta_0)$, which is known exactly \citep{kaufman1949crystal}.
We consider two transports: from $\beta_0=\BrCrossBetaZero$ to
$\beta_1=\BrCrossBetaOne$, across the phase transition, which splits one mode
into two, and from $\beta_0=\BrOrdBetaZero$ to $\beta_1=\BrOrdBetaOne$, within
the ordered phase, which moves two modes apart.  Both recover the free-energy
difference to within \BrCrossErrRel{} and \BrOrdErrRel{} $\times10^{-4}$ of
its value, and after reweighting both split their draws between the two
modes close to evenly (\cref{fig:bridge}).  The ordered-to-ordered transport has the
lower $\ESS$: the interpolant pairs opposite modes of the two ends half of the
time, and those paths lose their order midway.

\begin{figure}[t]
\centering
\begin{minipage}[c]{0.61\textwidth}
\centering
\tikzsetnextfilename{bridge}
\begin{tikzpicture}
\begin{groupplot}[
    group style={group size=3 by 2, horizontal sep=0pt, vertical sep=1.05cm}]
  \nextgroupplot[margL, title={$\pi_0$}, ylabel={magnetization $m$}]
    \addplot[refstair] table[x=p, y=m] {data/bridge/stair_Cross_ref0.dat} -- cycle;
    \addplot[modstair] table[x=p, y=m] {data/bridge/stair_Cross_t0.dat};
  \nextgroupplot[flowT, title={$\BrCrossBetaZero\to\BrCrossBetaOne$, across $\beta_c$}]
    \addplot3[flowdata] table[x=t, y=m, z expr=0, meta=p] {data/bridge/flow_Cross.dat};
  \nextgroupplot[margR, title={$\pi_1$}]
    \addplot[refstair] table[x=p, y=m] {data/bridge/stair_Cross_ref1.dat} -- cycle;
    \addplot[modstair] table[x=p, y=m] {data/bridge/stair_Cross_t1.dat};
    \addplot[rwstair]  table[x=p, y=m] {data/bridge/stair_Cross_rw.dat};
  \nextgroupplot[margL, title={$\pi_0$}, ylabel={magnetization $m$}]
    \addplot[refstair] table[x=p, y=m] {data/bridge/stair_Ord_ref0.dat} -- cycle;
    \addplot[modstair] table[x=p, y=m] {data/bridge/stair_Ord_t0.dat};
  \nextgroupplot[flowT, title={$\BrOrdBetaZero\to\BrOrdBetaOne$, ordered}, xlabel={time $t$}]
    \addplot3[flowdata] table[x=t, y=m, z expr=0, meta=p] {data/bridge/flow_Ord.dat};
  \nextgroupplot[margR, title={$\pi_1$}]
    \addplot[refstair] table[x=p, y=m] {data/bridge/stair_Ord_ref1.dat} -- cycle;
    \addplot[modstair] table[x=p, y=m] {data/bridge/stair_Ord_t1.dat};
    \addplot[rwstair]  table[x=p, y=m] {data/bridge/stair_Ord_rw.dat};
\end{groupplot}
\begin{axis}[zoomPos, at={(group c3r1.north east)}, anchor=north west,
             xshift=0.45cm, title={zoom}]
  \addplot[refstair] table[x=p, y=m] {data/bridge/stair_Cross_ref1.dat} -- cycle;
  \addplot[modstair] table[x=p, y=m] {data/bridge/stair_Cross_t1.dat};
  \addplot[rwstair]  table[x=p, y=m] {data/bridge/stair_Cross_rw.dat};
\end{axis}
\begin{axis}[zoomNeg, at={(group c3r1.south east)}, anchor=south west,
             xshift=0.45cm]
  \addplot[refstair] table[x=p, y=m] {data/bridge/stair_Cross_ref1.dat} -- cycle;
  \addplot[modstair] table[x=p, y=m] {data/bridge/stair_Cross_t1.dat};
  \addplot[rwstair]  table[x=p, y=m] {data/bridge/stair_Cross_rw.dat};
\end{axis}
\begin{axis}[zoomPos, at={(group c3r2.north east)}, anchor=north west,
             xshift=0.45cm]
  \addplot[refstair] table[x=p, y=m] {data/bridge/stair_Ord_ref1.dat} -- cycle;
  \addplot[modstair] table[x=p, y=m] {data/bridge/stair_Ord_t1.dat};
  \addplot[rwstair]  table[x=p, y=m] {data/bridge/stair_Ord_rw.dat};
\end{axis}
\begin{axis}[zoomNeg, at={(group c3r2.south east)}, anchor=south west,
             xshift=0.45cm]
  \addplot[refstair] table[x=p, y=m] {data/bridge/stair_Ord_ref1.dat} -- cycle;
  \addplot[modstair] table[x=p, y=m] {data/bridge/stair_Ord_t1.dat};
  \addplot[rwstair]  table[x=p, y=m] {data/bridge/stair_Ord_rw.dat};
\end{axis}
\end{tikzpicture}
\end{minipage}\hfill
\begin{minipage}[c]{0.37\textwidth}
\centering
\tikzsetnextfilename{bridge-cfg-small}
\begin{tikzpicture}
\begin{groupplot}[
    group style={group size=3 by 2, horizontal sep=0.15cm, vertical sep=0.15cm},
    brcfgpanel, width=1.4cm, height=1.4cm,
    title style={font=\scriptsize, yshift=-0.9ex}]
  \nextgroupplot[title={$t=0$}]
    \addplot[brcfgdata] table[x=i, y=j, meta=v] {data/bridge/cfg_Cross_t0.dat};
  \nextgroupplot[title={$t=\tfrac12$}]
    \addplot[brcfgdata] table[x=i, y=j, meta=v] {data/bridge/cfg_Cross_t4.dat};
  \nextgroupplot[title={$t=1$}]
    \addplot[brcfgdata] table[x=i, y=j, meta=v] {data/bridge/cfg_Cross_t8.dat};
  \nextgroupplot
    \addplot[brcfgdata] table[x=i, y=j, meta=v] {data/bridge/cfg_Ord_t0.dat};
  \nextgroupplot
    \addplot[brcfgdata] table[x=i, y=j, meta=v] {data/bridge/cfg_Ord_t4.dat};
  \nextgroupplot
    \addplot[brcfgdata] table[x=i, y=j, meta=v] {data/bridge/cfg_Ord_t8.dat};
\end{groupplot}
\end{tikzpicture}\\[0.4em]
\footnotesize
\setlength{\tabcolsep}{2pt}
\begin{tabular}{@{}lcc@{}}
\toprule
& $\BrCrossBetaZero\to\BrCrossBetaOne$ & $\BrOrdBetaZero\to\BrOrdBetaOne$ \\
\midrule
exact $\Delta\ln Z$           & \BrCrossDlnZExact & \BrOrdDlnZExact \\
$|\Delta\ln Z|\downarrow$     & \BrCrossErr       & \BrOrdErr \\
$\ESS\uparrow$                & \BrCrossESS       & \BrOrdESS \\
$\mTVD\downarrow$             & \BrCrossMTVRW     & \BrOrdMTVRW \\
\bottomrule
\end{tabular}
\end{minipage}
\caption{\textbf{Bridging two Ising temperatures.}  \emph{Left:} one row per
transport, time running left to right.  The map shows $P(m)$ along the path;
beside it, the magnetization at $\beta_0$ and $\beta_1$ (grey: Swendsen--Wang
reference, blue: the model's draws, dashed: reweighted), and the two target
modes magnified.  \emph{Right, top:} one draw carried along each transport, at
$t=0$, $\tfrac12$ and $1$ (upper row: across $\beta_c$, lower row: ordered).  \emph{Right, bottom:} $\ln Z(\beta_1)-\ln Z(\beta_0)$ error, $\ESS$ and reweighted $\mTVD$, mean $\pm$ standard
deviation over \BrNSeeds{} seeds.}
\label{fig:bridge}
\end{figure}

\subsection{Biomolecular ensemble generator on bond vectors}\label{app:bg-bondvec}
The internal coordinates of \cref{sec:bg} build the chemistry of the molecule into the
discretization: bond lengths and angles stay near their minima, the boxes of the branch dihedrals
fix the L-chirality and the trans peptide bonds, and each backbone dihedral is a single
coordinate.  Here we remove this prior and train \method{} on bond vectors, with everything else
as in \cref{sec:bg} (\cref{tab:hp-bg}).

\paragraph{Coordinates.}  The rigid-body motion is removed by a frame on the central residue:
C$_\alpha$ at the origin, N on the $+x$ axis and C in the $xy$-plane.  Every other atom is given by
the three components, in this frame, of its bond vector to its parent in the spanning tree of
\cref{app:bg-disc}, with $\ln|J|=2\ln N_x+\ln C_y$; the molecule is thus still described
relative to itself, not by absolute atom positions.  The boxes cover a sweep of the minimized
structure over uniform $\phi$, $\psi$ and methyl torsions, plus a margin of 0.12\,nm, and
$L=512$.  Chirality is not built in, so a configuration with a D-centre receives a large energy
penalty.  A Gibbs label changes one coordinate and so translates one subtree, which cannot rotate
half the molecule across a barrier in $\phi$ or $\psi$.  Before it is labelled, every fresh draw
therefore receives 1000 Metropolis pivot moves, each a uniform rotation of the subtree below a
randomly chosen $\phi$ or $\psi$ bond.  The pivot kernel leaves $\pi$ invariant, so a moved
configuration keeps the weight of its draw.  The moves only enter the training buffer: every
evaluated sample is an exact draw of the model.

\paragraph{Results.}  \Cref{tab:bg-cmp,tab:bg-cmp-raw} compare the two coordinate systems.  On ALDP,
\method{} still samples bond vectors with an $\ESS$ of \PepDiCartESS{} and a reweighted JS
divergence below $10^{-3}$.  Its draws sit at higher energies than MD, which the weights correct
(\cref{fig:bg-energy}).  On AL4, every residue has too little mass in the $\beta$/PPII region and
too much in $\alpha_R$ (\cref{tab:bg-basins}).  Its weights are heavy-tailed: a few draws carry
most of the weight, so that the reweighted metrics and $\ln Z$ change from one sample set to the
next, and samples of the same size from the same model give an $\ESS$ that varies by more than an
order of magnitude.  With eight times as many bins, a Gibbs label also costs eight times as many
energy evaluations: a training run takes \PepDiCartECalls{} billion of them on ALDP against
\PepDiIntECalls{} billion in internal coordinates, and \PepTetraCartECalls{} against
\PepTetraIntECalls{} billion on AL4.

\paragraph{The MLE baseline on bond vectors.}  MLE on MD (\cref{app:bg-mle}) fares worse on bond
vectors than in internal coordinates, and worse than \method{}.  Its draws sit far higher in
energy than MD, with an energy $\mathcal W_1$ as drawn of \PepDiCartMLEWOneRaw{}\,kJ/mol on ALDP
against \PepDiCartWOneRaw{} for \method{}, so that few of them carry weight: its $\ESS$ falls
below $10^{-3}$ on both peptides, and its reweighted metrics rest on a few draws
(\cref{tab:bg-cmp,tab:bg-cmp-raw}).

\paragraph{Absolute free energies.}  Both samplers estimate $\ln Z$ of the continuous Boltzmann
density on their box (\cref{eq:bg-weight,eq:fe_absolute}).  They are trained independently on
different discretizations, yet must agree up to an exact offset: a bond-vector box contains both
cyclic orders of the three hydrogens of each of the $m$ methyl groups, of which an internal box
keeps one, so that $\ln Z_{\mathrm{bond}}-\ln Z_{\mathrm{int}}=m\ln 2$.  For ALDP
($m=\PepDiMethyls$) the two estimates differ by \PepDiDLogZ{} against
$m\ln 2=\PepDiDLogZExpected$, and each is reproduced across seeds to within $10^{-3}$
(\cref{tab:bg-seeds}).  For AL4 ($m=\PepTetraMethyls$) the difference is \PepTetraDLogZ{}
against \PepTetraDLogZExpected{}, within the spread of the heavy-tailed bond-vector estimates.

\begin{table}[h]
\caption{\textbf{Internal coordinates and bond vectors, reweighted.}  \method{} (trained on the
energy) and the MLE baseline (trained on MD frames, \cref{app:bg-mle}), \PepDiNSamples{} draws
per run, mean $\pm$ standard deviation over \PepNSeeds{} seeds, with the importance weights;
\cref{tab:bg-cmp-raw} has the samples as drawn.  The better of the two methods in each row is
highlighted using blue shading.}
\label{tab:bg-cmp}
\centering
\footnotesize
\setlength{\tabcolsep}{4pt}
\renewcommand{\arraystretch}{1.12}
\begin{tabular}{@{}lcccccc}
\toprule
& \multicolumn{2}{c}{$\ESS\uparrow$} & \multicolumn{2}{c}{energy $\mathcal W_1\downarrow$ (kJ/mol)}
& \multicolumn{2}{c}{Ramachandran JS$\downarrow$ ($10^{-3}$)} \\
\cmidrule(lr){2-3}\cmidrule(lr){4-5}\cmidrule(l){6-7}
& MLE on MD & \method{} & MLE on MD & \method{} & MLE on MD & \method{} \\
\midrule
\multicolumn{7}{@{}l}{\emph{ALDP, $D=60$}} \\
internal & \PepDiIntMLECmpESS & \PepDiIntCmpESS & \PepDiIntMLECmpWOne & \PepDiIntCmpWOne
  & \PepDiIntMLECmpJS & \PepDiIntCmpJS \\
bond vectors & \PepDiCartMLECmpESS & \PepDiCartCmpESS & \PepDiCartMLECmpWOne & \PepDiCartCmpWOne
  & \PepDiCartMLECmpJS & \PepDiCartCmpJS \\
\addlinespace[0.5ex]
\multicolumn{7}{@{}l}{\emph{AL4, $D=120$}} \\
internal & \PepTetraIntMLECmpESS & \PepTetraIntCmpESS & \PepTetraIntMLECmpWOne & \PepTetraIntCmpWOne
  & \PepTetraIntMLECmpJS & \PepTetraIntCmpJS \\
bond vectors & \PepTetraCartMLECmpESS & \PepTetraCartCmpESS & \PepTetraCartMLECmpWOne & \PepTetraCartCmpWOne
  & \PepTetraCartMLECmpJS & \PepTetraCartCmpJS \\
\bottomrule
\end{tabular}
\end{table}

\begin{table}[h]
\caption{\textbf{Internal coordinates and bond vectors, as drawn}, and $\ln Z$ of the Boltzmann
density on the box (\cref{eq:bg-weight}), which differs between the two coordinate systems by
$m\ln2$.  Settings as in \cref{tab:bg-cmp}; the better of the two methods in each row is
highlighted using blue shading, $\ln Z$ has no reference and is not shaded.}
\label{tab:bg-cmp-raw}
\centering
\footnotesize
\setlength{\tabcolsep}{4pt}
\renewcommand{\arraystretch}{1.12}
\begin{tabular}{@{}lcccccc}
\toprule
& \multicolumn{2}{c}{energy $\mathcal W_1\downarrow$ (kJ/mol)}
& \multicolumn{2}{c}{Ramachandran JS$\downarrow$ ($10^{-3}$)} & \multicolumn{2}{c}{$\ln Z$} \\
\cmidrule(lr){2-3}\cmidrule(lr){4-5}\cmidrule(l){6-7}
& MLE on MD & \method{} & MLE on MD & \method{} & MLE on MD & \method{} \\
\midrule
\multicolumn{7}{@{}l}{\emph{ALDP, $D=60$}} \\
internal & \PepDiIntMLECmpWOneRaw & \PepDiIntCmpWOneRaw & \PepDiIntMLECmpJSRaw & \PepDiIntCmpJSRaw
  & \PepDiIntMLELogZ & \PepDiIntLogZ \\
bond vectors & \PepDiCartMLECmpWOneRaw & \PepDiCartCmpWOneRaw & \PepDiCartMLECmpJSRaw & \PepDiCartCmpJSRaw
  & \PepDiCartMLELogZ & \PepDiCartLogZ \\
\addlinespace[0.5ex]
\multicolumn{7}{@{}l}{\emph{AL4, $D=120$}} \\
internal & \PepTetraIntMLECmpWOneRaw & \PepTetraIntCmpWOneRaw & \PepTetraIntMLECmpJSRaw & \PepTetraIntCmpJSRaw
  & \PepTetraIntMLELogZ & \PepTetraIntLogZ \\
bond vectors & \PepTetraCartMLECmpWOneRaw & \PepTetraCartCmpWOneRaw & \PepTetraCartMLECmpJSRaw & \PepTetraCartCmpJSRaw
  & \PepTetraCartMLELogZ & \PepTetraCartLogZ \\
\bottomrule
\end{tabular}
\end{table}

\subsection{Additional results}\label{app:further-results}

\paragraph{Amortized lattice models.}\label{app:fp-maps}
The $\ESS$ of the amortized samplers stays high across the $(\beta,h)$ grid
and is lowest at and just below the zero-field critical point, where the
fluctuations are largest: at worst \FPIsingESSmin{} for Ising and
\FPPottsESSmin{} for Potts (\cref{fig:fp-ess}).

\begin{figure}[h]
\centering
\tikzsetnextfilename{fp-ess}
\begin{tikzpicture}
\begin{groupplot}[
    group style={group size=2 by 1, horizontal sep=1.35cm},
    fpmap, width=3.0cm, height=3.0cm, point meta min=0, point meta max=1]
  \nextgroupplot[fpising, title={Ising, $50\times50$}, xlabel={field $h$},
                 ylabel={$\beta$}]
    \addplot[fpdata] table[x=h, y=beta, meta=v] {data/lattice/fp_ising_ess.dat};
  \nextgroupplot[fppotts, title={Potts, $24\times24$}, xlabel={field $h$},
                 ylabel={$\beta$}]
    \addplot[fpdata] table[x=h, y=beta, meta=v] {data/lattice/fp_potts_ess.dat};
\end{groupplot}
\begin{axis}[at={(group c2r1.north east)}, anchor=north west, xshift=0.25cm,
    hide axis, scale only axis, width=1pt, height=3.0cm,
    xmin=0, xmax=1, ymin=0, ymax=1,
    colormap/viridis, point meta min=0, point meta max=1,
    colorbar, colorbar style={width=0.2cm, ytick={0,0.25,0.5,0.75,1},
      tick label style={font=\scriptsize}, ylabel={$\ESS$},
      ylabel style={font=\small}}]
  \addplot[draw=none] coordinates {(0,0)};
\end{axis}
\end{tikzpicture}
\caption{\textbf{Effective sample size of the amortized samplers across $(\beta,h)$.}
Fraction of the \FPNDraws{} draws per cell, mean over seeds.}
\label{fig:fp-ess}
\end{figure}

\paragraph{Alloys.}\label{app:further-alloy-ess}
\method{} keeps a high $\ESS$ over nearly the whole $(T,\DMu)$ sweep of both
alloys and loses it only at a small fraction of state points, mostly on the
coldest isotherms, whereas the $\ESS$ of FrOGS is low almost everywhere (\cref{fig:alloy-ess}).

\begin{figure}[t]
\centering
\tikzsetnextfilename{alloy-ess}
\begin{tikzpicture}
\begin{groupplot}[
    group style={group size=2 by 2, horizontal sep=0.5cm, vertical sep=1.05cm,
                 y descriptions at=edge left},
    essmap, width=3.3cm]
  \nextgroupplot[title={\textcolor{cDSI}{\textbf{\method{} (ours)}}},
                 ylabel={Ag--Pd\\[-0.2ex]$T$ (K)}, ylabel style={align=center},
                 xmin=\AgPdDmuLo, xmax=\AgPdDmuHi, ymin=\AgPdTLo, ymax=\AgPdTHi,
                 xtick={-0.4,-0.2,0,0.2,0.4}, xticklabels={,$-0.2$,0,$0.2$,},
                 xlabel={$\DMu$ (eV)}]
    \addplot[essdata=\AgPdNdmu] table[x=dmu, y=T, meta=ess] {data/alloy/agpd125_dsi_ess.dat};
  \nextgroupplot[title={\textcolor{cFrOGS}{\textbf{FrOGS}}},
                 xmin=\AgPdDmuLo, xmax=\AgPdDmuHi, ymin=\AgPdTLo, ymax=\AgPdTHi,
                 xtick={-0.4,-0.2,0,0.2,0.4}, xticklabels={,$-0.2$,0,$0.2$,},
                 xlabel={$\DMu$ (eV)}]
    \addplot[essdata=\AgPdNdmu] table[x=dmu, y=T, meta=ess] {data/alloy/agpd125_frogs_ess.dat};
    \panelmark{AgPd}{FrOGS}
  \nextgroupplot[ylabel={Cu--Au\\[-0.2ex]$T$ (K)}, ylabel style={align=center},
                 xmin=\CuAuDmuLo, xmax=\CuAuDmuHi, ymin=\CuAuTLo, ymax=\CuAuTHi,
                 xtick={-0.2,-0.1,0,0.1}, xlabel={$\DMu$ (eV)}]
    \addplot[essdata=\CuAuNdmu] table[x=dmu, y=T, meta=ess] {data/alloy/cuau128_dsi_ess.dat};
  \nextgroupplot[xmin=\CuAuDmuLo, xmax=\CuAuDmuHi, ymin=\CuAuTLo, ymax=\CuAuTHi,
                 xtick={-0.2,-0.1,0,0.1}, xlabel={$\DMu$ (eV)}]
    \addplot[essdata=\CuAuNdmu] table[x=dmu, y=T, meta=ess] {data/alloy/cuau128_frogs_ess.dat};
    \panelmark{CuAu}{FrOGS}
\end{groupplot}
\begin{axis}[at={(group c2r1.north east)}, anchor=north west, xshift=0.25cm,
    hide axis, scale only axis, width=1pt, height=6.55cm,
    xmin=0, xmax=1, ymin=0, ymax=1,
    colormap/viridis, point meta min=0, point meta max=1,
    colorbar, colorbar style={width=0.22cm, ytick={0,0.25,0.5,0.75,1},
      tick label style={font=\scriptsize}, ylabel={$\ESS$},
      ylabel style={font=\small}}]
  \addplot[draw=none] coordinates {(0,0)};
\end{axis}
\end{tikzpicture}
\caption{\textbf{Effective sample size across the alloy sweep.}  Fraction of the
\AlloyNDraws{} draws per state point, mean over \AlloyNSeeds{} seeds, for
Ag--Pd (top) and Cu--Au (bottom).\placeholdernote}
\label{fig:alloy-ess}
\end{figure}

\paragraph{Discretized biomolecular ensemble sampling.}\label{app:further-bg}
\Cref{tab:bg-seeds} lists every run of \cref{sec:bg} and \cref{app:bg-bondvec}, including those of
the MLE baseline, \cref{fig:bg-rama-all} the Ramachandran surfaces of the samples as drawn and
reweighted, \cref{fig:bg-rama-mle} those of the baseline, and \cref{fig:bg-energy} the energy
distributions of both.  In internal coordinates, the samples as drawn
are already close to MD; reweighting removes nearly all of the remaining difference in energy and
lowers the Ramachandran JS divergence further.  On bond vectors, the draws sit at higher energies
than MD, which the weights correct on ALDP but not on AL4 (\cref{app:bg-bondvec}).

\begin{table}[h]
\caption{\textbf{Discretized biomolecular ensemble sampler, every run}, of \method{} and of the MLE
baseline (\cref{app:bg-mle}).  \PepDiNSamples{} draws per run; $\mathcal W_1$ in kJ/mol and JS in
$10^{-3}$, for the samples as drawn and reweighted.  $\ln Z$ is that of the Boltzmann density on
the box (\cref{eq:bg-weight}); bond-vector and internal values differ by $m\ln2$, with $m$ the
number of methyl groups (\cref{app:bg-bondvec}).}
\label{tab:bg-seeds}
\centering
\footnotesize
\setlength{\tabcolsep}{5pt}
\begin{tabular}{@{}lllrrrrrr@{}}
\toprule
& & & & & \multicolumn{2}{c}{$\mathcal W_1\downarrow$} & \multicolumn{2}{c}{JS$\downarrow$} \\
\cmidrule(lr){6-7}\cmidrule(l){8-9}
peptide & coordinates & seed & $\ESS\uparrow$ & $\ln Z$ & as drawn & reweighted & as drawn & reweighted \\
\midrule
\multicolumn{9}{@{}l}{\emph{\method{}, trained on the energy}} \\
ALDP & internal & 0 & $.502$ & $-195.324$ & $0.95$ & $0.02$ & $0.27$ & $0.19$ \\
 &  & 1 & $.502$ & $-195.323$ & $0.98$ & $0.03$ & $0.23$ & $0.17$ \\
 &  & 2 & $.506$ & $-195.324$ & $0.98$ & $0.01$ & $0.23$ & $0.17$ \\
\addlinespace[0.4ex]
 & bond vectors & 0 & $.112$ & $-193.249$ & $3.94$ & $0.04$ & $1.33$ & $0.71$ \\
 &  & 1 & $.077$ & $-193.249$ & $3.71$ & $0.07$ & $1.35$ & $0.75$ \\
 &  & 2 & $.120$ & $-193.249$ & $4.01$ & $0.07$ & $1.27$ & $0.62$ \\
\addlinespace[0.4ex]
AL4 & internal & 0 & $.064$ & $-446.140$ & $4.68$ & $0.08$ & $1.62$ & $0.88$ \\
 &  & 1 & $.064$ & $-446.141$ & $5.06$ & $0.07$ & $2.29$ & $0.84$ \\
 &  & 2 & $.058$ & $-446.143$ & $4.79$ & $0.04$ & $2.13$ & $0.78$ \\
\addlinespace[0.4ex]
 & bond vectors & 0 & $8.4{\cdot}10^{-6}$ & $-442.582$ & $29.92$ & $4.26$ & $26.22$ & $114.93$ \\
 &  & 1 & $2.1{\cdot}10^{-4}$ & $-442.864$ & $28.27$ & $1.56$ & $22.90$ & $36.35$ \\
 &  & 2 & $1.1{\cdot}10^{-5}$ & $-442.634$ & $28.71$ & $3.46$ & $25.44$ & $121.29$ \\
 
\midrule
\multicolumn{9}{@{}l}{\emph{MLE on MD, trained on MD frames}} \\
ALDP & internal & 0 & $.271$ & $-195.324$ & $2.47$ & $0.02$ & $0.59$ & $0.37$ \\
 &  & 1 & $.240$ & $-195.322$ & $2.72$ & $0.02$ & $0.62$ & $0.34$ \\
 &  & 2 & $.278$ & $-195.324$ & $2.62$ & $0.02$ & $0.52$ & $0.31$ \\
\addlinespace[0.4ex]
 & bond vectors & 0 & $9.7{\cdot}10^{-5}$ & $-194.896$ & $21.38$ & $1.18$ & $2.08$ & $55.86$ \\
 &  & 1 & $3.7{\cdot}10^{-5}$ & $-194.444$ & $19.71$ & $1.26$ & $1.67$ & $97.45$ \\
 &  & 2 & $1.8{\cdot}10^{-5}$ & $-194.480$ & $19.70$ & $2.31$ & $1.82$ & $114.87$ \\
\addlinespace[0.4ex]
AL4 & internal & 0 & $.045$ & $-446.146$ & $9.01$ & $0.11$ & $0.79$ & $1.10$ \\
 &  & 1 & $.031$ & $-446.140$ & $10.10$ & $0.08$ & $0.74$ & $1.51$ \\
 &  & 2 & $.047$ & $-446.143$ & $9.03$ & $0.09$ & $0.76$ & $1.10$ \\
\addlinespace[0.4ex]
 & bond vectors & 0 & $7.4{\cdot}10^{-7}$ & $-443.165$ & $100.09$ & $14.69$ & $7.81$ & $500.37$ \\
 &  & 1 & $2.0{\cdot}10^{-5}$ & $-445.950$ & $96.11$ & $5.22$ & $7.09$ & $160.98$ \\
 &  & 2 & $2.5{\cdot}10^{-6}$ & $-444.724$ & $91.08$ & $7.90$ & $7.55$ & $251.61$ \\
 
\bottomrule
\end{tabular}
\end{table}

\begin{figure}[h]
\centering
\tikzsetnextfilename{bg-rama-all}
\begin{tikzpicture}
\begin{groupplot}[
    group style={group size=4 by 5, horizontal sep=0.55cm, vertical sep=0.45cm,
                 xticklabels at=edge bottom, yticklabels at=edge left},
    ramapanel, width=2.15cm, height=2.15cm]
  \nextgroupplot[title={ALDP}, ylabel={MD\\$\psi$}, ylabel style={align=center}]
    \rama{data/peptides/rama_aldp_md_r0}
  \nextgroupplot[title={AL4, Ala\,2}] \rama{data/peptides/rama_al4_md_r0}
  \nextgroupplot[title={AL4, Ala\,3}] \rama{data/peptides/rama_al4_md_r1}
  \nextgroupplot[title={AL4, Ala\,4}] \rama{data/peptides/rama_al4_md_r2}
  \nextgroupplot[ylabel={internal,\\as drawn\\$\psi$}, ylabel style={align=center}]
    \rama{data/peptides/rama_aldp_internal_raw_r0}
  \nextgroupplot \rama{data/peptides/rama_al4_internal_raw_r0}
  \nextgroupplot \rama{data/peptides/rama_al4_internal_raw_r1}
  \nextgroupplot \rama{data/peptides/rama_al4_internal_raw_r2}
  \nextgroupplot[ylabel={internal,\\reweighted\\$\psi$}, ylabel style={align=center}]
    \rama{data/peptides/rama_aldp_internal_rw_r0}
  \nextgroupplot \rama{data/peptides/rama_al4_internal_rw_r0}
  \nextgroupplot \rama{data/peptides/rama_al4_internal_rw_r1}
  \nextgroupplot \rama{data/peptides/rama_al4_internal_rw_r2}
  \nextgroupplot[ylabel={bond vectors,\\as drawn\\$\psi$}, ylabel style={align=center}]
    \rama{data/peptides/rama_aldp_cartesian_raw_r0}
  \nextgroupplot \rama{data/peptides/rama_al4_cartesian_raw_r0}
  \nextgroupplot \rama{data/peptides/rama_al4_cartesian_raw_r1}
  \nextgroupplot \rama{data/peptides/rama_al4_cartesian_raw_r2}
  \nextgroupplot[ylabel={bond vectors,\\reweighted\\$\psi$}, ylabel style={align=center}, xlabel={$\phi$}]
    \rama{data/peptides/rama_aldp_cartesian_rw_r0}
  \nextgroupplot[xlabel={$\phi$}] \rama{data/peptides/rama_al4_cartesian_rw_r0}
  \nextgroupplot[xlabel={$\phi$}] \rama{data/peptides/rama_al4_cartesian_rw_r1}
  \nextgroupplot[xlabel={$\phi$}] \rama{data/peptides/rama_al4_cartesian_rw_r2}
\end{groupplot}
\ramacolorbar{group c4r1.north east}{12.55cm}
\end{tikzpicture}
\caption{\textbf{Ramachandran surfaces of every setting}, $\ln p(\phi,\psi)$ relative to its
maximum.  Rows: MD reference, then \method{} in internal coordinates and on bond vectors
(\cref{app:bg-bondvec}), as drawn and reweighted; the three seeds pooled.  For AL4 on bond
vectors the reweighted surfaces rest on a few effective samples ($\ESS<10^{-3}$).}
\label{fig:bg-rama-all}
\end{figure}

\begin{figure}[h]
\centering
\tikzsetnextfilename{bg-rama-mle}
\begin{tikzpicture}
\begin{groupplot}[
    group style={group size=4 by 4, horizontal sep=0.55cm, vertical sep=0.45cm,
                 xticklabels at=edge bottom, yticklabels at=edge left},
    ramapanel, width=2.15cm, height=2.15cm]
  \nextgroupplot[title={ALDP}, ylabel={MD\\$\psi$}, ylabel style={align=center}]
    \rama{data/peptides/rama_aldp_md_r0}
  \nextgroupplot[title={AL4, Ala\,2}] \rama{data/peptides/rama_al4_md_r0}
  \nextgroupplot[title={AL4, Ala\,3}] \rama{data/peptides/rama_al4_md_r1}
  \nextgroupplot[title={AL4, Ala\,4}] \rama{data/peptides/rama_al4_md_r2}
  \nextgroupplot[ylabel={internal,\\as drawn\\$\psi$}, ylabel style={align=center}]
    \rama{data/peptides/rama_aldp_internal_mle_raw_r0}
  \nextgroupplot \rama{data/peptides/rama_al4_internal_mle_raw_r0}
  \nextgroupplot \rama{data/peptides/rama_al4_internal_mle_raw_r1}
  \nextgroupplot \rama{data/peptides/rama_al4_internal_mle_raw_r2}
  \nextgroupplot[ylabel={internal,\\reweighted\\$\psi$}, ylabel style={align=center}]
    \rama{data/peptides/rama_aldp_internal_mle_rw_r0}
  \nextgroupplot \rama{data/peptides/rama_al4_internal_mle_rw_r0}
  \nextgroupplot \rama{data/peptides/rama_al4_internal_mle_rw_r1}
  \nextgroupplot \rama{data/peptides/rama_al4_internal_mle_rw_r2}
  \nextgroupplot[ylabel={bond vectors,\\as drawn\\$\psi$}, ylabel style={align=center}, xlabel={$\phi$}]
    \rama{data/peptides/rama_aldp_cartesian_mle_raw_r0}
  \nextgroupplot[xlabel={$\phi$}] \rama{data/peptides/rama_al4_cartesian_mle_raw_r0}
  \nextgroupplot[xlabel={$\phi$}] \rama{data/peptides/rama_al4_cartesian_mle_raw_r1}
  \nextgroupplot[xlabel={$\phi$}] \rama{data/peptides/rama_al4_cartesian_mle_raw_r2}
\end{groupplot}
\ramacolorbar{group c4r1.north east}{9.95cm}
\end{tikzpicture}
\caption{\textbf{Ramachandran surfaces of the MLE baseline} (\cref{app:bg-mle}), $\ln p(\phi,\psi)$
relative to its maximum, the three seeds pooled.  Rows: MD reference, then MLE on MD in internal
coordinates, as drawn and reweighted, and on bond vectors as drawn.  On bond vectors its
$\ESS$ is below $10^{-3}$ on both peptides, so its reweighted surfaces are not shown.}
\label{fig:bg-rama-mle}
\end{figure}

\begin{figure}[h]
\centering
\tikzsetnextfilename{bg-energy}
\begin{tikzpicture}
\begin{groupplot}[
    group style={group size=4 by 1, horizontal sep=0.35cm},
    enpanel, xlabel={$E$ (kJ/mol)}]
  \nextgroupplot[title={ALDP, internal}, xmin=\PepDiELo, xmax=\PepDiEHi]
    \addplot[refhist] table[x=E, y=md] {data/peptides/energy_aldp_internal.dat} \closedcycle;
    \addplot[rawhist] table[x=E, y=raw] {data/peptides/energy_aldp_internal.dat};
    \addplot[rwhist]  table[x=E, y=rw]  {data/peptides/energy_aldp_internal.dat};
    \addplot[mlehist] table[x=E, y=mle_raw] {data/peptides/energy_aldp_internal.dat};
  \nextgroupplot[title={ALDP, bond vectors}, xmin=\PepDiELo, xmax=\PepDiEHi]
    \addplot[refhist] table[x=E, y=md] {data/peptides/energy_aldp_cartesian.dat} \closedcycle;
    \addplot[rawhist] table[x=E, y=raw] {data/peptides/energy_aldp_cartesian.dat};
    \addplot[rwhist]  table[x=E, y=rw]  {data/peptides/energy_aldp_cartesian.dat};
    \addplot[mlehist] table[x=E, y=mle_raw] {data/peptides/energy_aldp_cartesian.dat};
  \nextgroupplot[title={AL4, internal}, xmin=\PepTetraELo, xmax=\PepTetraEHi]
    \addplot[refhist] table[x=E, y=md] {data/peptides/energy_al4_internal.dat} \closedcycle;
    \addplot[rawhist] table[x=E, y=raw] {data/peptides/energy_al4_internal.dat};
    \addplot[rwhist]  table[x=E, y=rw]  {data/peptides/energy_al4_internal.dat};
    \addplot[mlehist] table[x=E, y=mle_raw] {data/peptides/energy_al4_internal.dat};
  \nextgroupplot[title={AL4, bond vectors}, xmin=\PepTetraELo, xmax=\PepTetraEHi]
    \addplot[refhist] table[x=E, y=md] {data/peptides/energy_al4_cartesian.dat} \closedcycle;
    \addplot[rawhist] table[x=E, y=raw] {data/peptides/energy_al4_cartesian.dat};
    \addplot[rwhist]  table[x=E, y=rw]  {data/peptides/energy_al4_cartesian.dat};
    \addplot[mlehist] table[x=E, y=mle_raw] {data/peptides/energy_al4_cartesian.dat};
\end{groupplot}
\end{tikzpicture}
\caption{\textbf{Potential energy distributions.}  Grey: MD; blue: \method{} as drawn; dashed:
\method{} reweighted; orange: MLE on MD as drawn (\cref{app:bg-mle}).  The three seeds pooled.}
\label{fig:bg-energy}
\end{figure}

\begin{table}[h]
\caption{\textbf{Basin fractions of AL4} per residue: $\alpha_R$ ($\phi<0$,
$-120^\circ<\psi<50^\circ$), $\beta$/PPII (the rest of $\phi<0$) and $\alpha_L$ ($\phi\ge0$), for
MD and for \method{} as drawn and reweighted, the three seeds pooled.}
\label{tab:bg-basins}
\centering
\footnotesize
\setlength{\tabcolsep}{6pt}
\begin{tabular}{@{}llccccc@{}}
\toprule
& & & \multicolumn{2}{c}{internal} & \multicolumn{2}{c}{bond vectors} \\
\cmidrule(lr){4-5}\cmidrule(l){6-7}
residue & basin & MD & as drawn & reweighted & as drawn & reweighted \\
\midrule
Ala\,2 & $\alpha_R$ & $0.692$ & $0.674$ & $0.691$ & $0.849$ & $0.656$ \\
 & $\beta$/PPII & $0.282$ & $0.310$ & $0.283$ & $0.139$ & $0.337$ \\
 & $\alpha_L$ & $0.026$ & $0.016$ & $0.026$ & $0.013$ & $0.007$ \\
\addlinespace[0.4ex]
Ala\,3 & $\alpha_R$ & $0.701$ & $0.673$ & $0.701$ & $0.864$ & $0.653$ \\
 & $\beta$/PPII & $0.273$ & $0.312$ & $0.275$ & $0.130$ & $0.335$ \\
 & $\alpha_L$ & $0.025$ & $0.015$ & $0.024$ & $0.006$ & $0.012$ \\
\addlinespace[0.4ex]
Ala\,4 & $\alpha_R$ & $0.610$ & $0.578$ & $0.610$ & $0.768$ & $0.569$ \\
 & $\beta$/PPII & $0.373$ & $0.409$ & $0.374$ & $0.219$ & $0.426$ \\
 & $\alpha_L$ & $0.017$ & $0.013$ & $0.016$ & $0.013$ & $0.005$ \\
 
\bottomrule
\end{tabular}
\end{table}

\paragraph{Free energy estimation.}\label{app:further-fe}
\Cref{tab:fe-seeds} lists every run of \cref{sec:fe}.  \Cref{fig:fe-rama} shows where the
transport carries the MD frames: on ALDP-S the endpoints as drawn already match the solvent MD,
and on ALDP-T the transport moves the frames from $S_b$ into the narrow state $S_a$ at
$\phi>0$, where all reweighted mass lies.  \Cref{fig:fe-extra} shows the energies of the
endpoints and the estimate during training, from 8192 held-out paths every 5000 steps.

\begin{table}[h]
\caption{\textbf{Free-energy differences, every run.}  One path per endpoint-0 MD frame;
$\ln Z_1/Z_0$ with its standard error over starting frames, the $\ESS$, the endpoint metrics
against the endpoint-1 MD as drawn and reweighted (JS in $10^{-3}$, $\mathcal W_1$ in
kJ/mol), and the fraction of endpoints in the support of endpoint 1.}
\label{tab:fe-seeds}
\centering
\footnotesize
\setlength{\tabcolsep}{5pt}
\begin{tabular}{@{}llrrrrrrr@{}}
\toprule
& & & & \multicolumn{2}{c}{JS$\downarrow$} & \multicolumn{2}{c}{$\mathcal W_1\downarrow$} & \\
\cmidrule(lr){5-6}\cmidrule(lr){7-8}
task & seed & $\ln Z_1/Z_0$ & $\ESS\uparrow$ & as drawn & reweighted & as drawn & reweighted & in state 1 \\
\midrule
ALDP-S & 0 & $29.432\pm0.010$ & $0.118$ & $1.44$ & $3.71$ & $1.30$ & $0.18$ & $1.0000$ \\
 & 1 & $29.441\pm0.010$ & $0.101$ & $1.24$ & $3.86$ & $1.17$ & $0.28$ & $1.0000$ \\
 & 2 & $29.436\pm0.010$ & $0.091$ & $1.44$ & $3.26$ & $1.35$ & $0.29$ & $1.0000$ \\
\addlinespace[0.4ex]
ALDP-T & 0 & $-4.469\pm0.011$ & $0.063$ & $6.13$ & $4.85$ & $2.02$ & $0.28$ & $0.9977$ \\
 & 1 & $-4.438\pm0.029$ & $0.005$ & $6.19$ & $6.86$ & $2.09$ & $0.95$ & $0.9977$ \\
 & 2 & $-4.467\pm0.011$ & $0.071$ & $5.95$ & $4.62$ & $2.01$ & $0.38$ & $0.9976$ \\
 
\bottomrule
\end{tabular}
\end{table}

\begin{figure}[h]
\centering
\tikzsetnextfilename{fe-rama}
\begin{tikzpicture}
\begin{groupplot}[
    group style={group size=4 by 2, horizontal sep=0.55cm, vertical sep=0.45cm,
                 xticklabels at=edge bottom, yticklabels at=edge left},
    ramapanel, width=2.15cm, height=2.15cm]
  \nextgroupplot[title={endpoint 0, MD}, ylabel={ALDP-S\\$\psi$}, ylabel style={align=center}]
    \rama{data/freeenergy/rama_aldps_md0}
  \nextgroupplot[title={\method{}, as drawn}] \rama{data/freeenergy/rama_aldps_raw}
  \nextgroupplot[title={\method{}, reweighted}] \rama{data/freeenergy/rama_aldps_rw}
  \nextgroupplot[title={endpoint 1, MD}] \rama{data/freeenergy/rama_aldps_md1}
  \nextgroupplot[ylabel={ALDP-T\\$\psi$}, ylabel style={align=center}, xlabel={$\phi$}]
    \rama{data/freeenergy/rama_aldpt_md0}
  \nextgroupplot[xlabel={$\phi$}] \rama{data/freeenergy/rama_aldpt_raw}
  \nextgroupplot[xlabel={$\phi$}] \rama{data/freeenergy/rama_aldpt_rw}
  \nextgroupplot[xlabel={$\phi$}] \rama{data/freeenergy/rama_aldpt_md1}
\end{groupplot}
\ramacolorbar{group c4r1.north east}{4.75cm}
\end{tikzpicture}
\caption{\textbf{Endpoints of the bridge.}  Ramachandran surfaces of the given endpoint-0 MD
frames, of the endpoints the trained CTMC carries them to, as drawn and reweighted (the three
seeds pooled), and of the endpoint-1 MD, which enters only the evaluation.}
\label{fig:fe-rama}
\end{figure}

\begin{figure}[h]
\centering
\tikzsetnextfilename{fe-extra}
\begin{tikzpicture}
\begin{groupplot}[
    group style={group size=2 by 2, horizontal sep=1.5cm, vertical sep=1.9cm}]
  \nextgroupplot[enpanel, title={ALDP-S, endpoint 1}, xlabel={$E$ (kJ/mol)}]
    \addplot[refhist] table[x=E, y=md] {data/freeenergy/energy_aldps.dat} \closedcycle;
    \addplot[rawhist] table[x=E, y=raw] {data/freeenergy/energy_aldps.dat};
    \addplot[rwhist]  table[x=E, y=rw]  {data/freeenergy/energy_aldps.dat};
  \nextgroupplot[enpanel, title={ALDP-T, endpoint 1}, xlabel={$E$ (kJ/mol)}]
    \addplot[refhist] table[x=E, y=md] {data/freeenergy/energy_aldpt.dat} \closedcycle;
    \addplot[rawhist] table[x=E, y=raw] {data/freeenergy/energy_aldpt.dat};
    \addplot[rwhist]  table[x=E, y=rw]  {data/freeenergy/energy_aldpt.dat};
  \nextgroupplot[trainpanel, title={ALDP-S, training},
                 xlabel={step}, ylabel={$\ln Z_1/Z_0$}, ymin=\FESTrainYmin, ymax=\FESTrainYmax]
    \refband{\FESRefPlain}{\FESRefErr}{black}{solid}
    \addplot[convseed] table[x=step, y=s0] {data/freeenergy/train_aldps.dat};
    \addplot[convseed] table[x=step, y=s1] {data/freeenergy/train_aldps.dat};
    \addplot[convseed] table[x=step, y=s2] {data/freeenergy/train_aldps.dat};
  \nextgroupplot[trainpanel, title={ALDP-T, training},
                 xlabel={step}, ymin=\FETTrainYmin, ymax=\FETTrainYmax]
    \refband{\FETRefPlain}{\FETRefErr}{black}{solid}
    \addplot[convseed] table[x=step, y=s0] {data/freeenergy/train_aldpt.dat};
    \addplot[convseed] table[x=step, y=s1] {data/freeenergy/train_aldpt.dat};
    \addplot[convseed] table[x=step, y=s2] {data/freeenergy/train_aldpt.dat};
\end{groupplot}
\end{tikzpicture}
\caption{\textbf{Bridge diagnostics.}  \emph{Top:} endpoint-1 energies of the MD (grey), of
the transported endpoints as drawn (blue) and reweighted (dashed).  \emph{Bottom:} the estimate
from 8192 held-out paths during training, one line per seed, with the reference (black): BAR on
the MD of both states for ALDP-S and umbrella sampling for ALDP-T.}
\label{fig:fe-extra}
\end{figure}

\paragraph{Reinforcement learning.}\label{app:rl-tasks}
\Cref{tab:rl-tasks} lists the final IQM of every method on each of the
\RLNTasks{} DMC tasks of \cref{sec:rl}.  \method{} is the best method, or within 5 of
the best, on \RLDIGINNBest{} of the \RLNTasks{} tasks.  Its largest gains are on
HopperHop and HopperStand and on the two Acrobot swing-up tasks.  It falls behind
TRUDI on CheetahRun, FishSwim, WalkerRun and CartpoleSwingupSparse, and behind FFM
and SPO on FingerSpin.

\begin{table}[h]
\caption{\textbf{Final IQM of the return per DMC task}, \RLNSeeds{} seeds per task.
\emph{Top:} the ten domains, each with the names of its tasks.  \emph{Bottom:} the
final IQM per task; $m$ is the action dimension.  The last row is the IQM over all
tasks of \cref{fig:rl}, times 1000.  The values of the baselines are the final
points of the curves reported by \citet{trudi2025}.  Entries within 5 of the best
entry of their row are highlighted using blue shading.}
\label{tab:rl-tasks}
\centering
\scriptsize
\setlength{\tabcolsep}{3.4pt}
\sbox{\rltasksbox}{%
\begin{tabular}{@{}lc*{9}{c}}
\toprule
& & \multicolumn{3}{c}{Gaussian} & \multicolumn{3}{c}{diffusion} & \multicolumn{2}{c}{flow} & binned \\
\cmidrule(lr){3-5}\cmidrule(lr){6-8}\cmidrule(lr){9-10}\cmidrule(l){11-11}
task & $m$ & REPPO & PPO & SPO & TRUDI & AMDP & DPPO & FPO & FFM & \method{} \\
\midrule
AcrobotSwingup & 1 & $358$ & $4$ & $111$ & $317$ & $360$ & $5$ & $20$ & $381$ & \best{$418$} \\
AcrobotSwingupSparse & 1 & $58$ & $0$ & $0$ & $77$ & $22$ & $0$ & $0$ & $2$ & \best{$113$} \\
BallInCup & 2 & \best{$975$} & $191$ & $966$ & \best{$975$} & \best{$974$} & $879$ & $964$ & \best{$975$} & \best{$975$} \\
CartpoleBalance & 1 & \best{$1000$} & $664$ & \best{$998$} & \best{$1000$} & $988$ & \best{$1000$} & $861$ & \best{$999$} & \best{$999$} \\
CartpoleBalanceSparse & 1 & \best{$999$} & $215$ & \best{$1000$} & \best{$1000$} & \best{$999$} & $470$ & \best{$1000$} & $847$ & \best{$1000$} \\
CartpoleSwingup & 1 & \best{$867$} & $745$ & $832$ & \best{$867$} & $858$ & $816$ & $469$ & \best{$865$} & \best{$867$} \\
CartpoleSwingupSparse & 1 & $652$ & $1$ & $433$ & \best{$844$} & $833$ & $141$ & $699$ & $761$ & $789$ \\
CheetahRun & 6 & \best{$933$} & $812$ & $794$ & \best{$934$} & $918$ & $409$ & $775$ & $929$ & $893$ \\
FingerSpin & 2 & $646$ & $18$ & $949$ & $718$ & $757$ & $426$ & $616$ & \best{$964$} & $938$ \\
FingerTurnEasy & 2 & $949$ & $340$ & $771$ & \best{$959$} & $954$ & $210$ & $918$ & $941$ & \best{$963$} \\
FingerTurnHard & 2 & $944$ & $168$ & $738$ & $945$ & \best{$951$} & $86$ & $848$ & $936$ & \best{$956$} \\
FishSwim & 5 & $642$ & $507$ & $454$ & \best{$696$} & $677$ & $104$ & $622$ & $613$ & $542$ \\
HopperHop & 4 & $139$ & $1$ & $1$ & $80$ & $2$ & $0$ & $51$ & $116$ & \best{$277$} \\
HopperStand & 4 & $820$ & $6$ & $19$ & $615$ & $429$ & $6$ & $310$ & $232$ & \best{$943$} \\
PendulumSwingup & 1 & $719$ & $16$ & $66$ & \best{$836$} & $811$ & $11$ & $63$ & $430$ & \best{$838$} \\
ReacherEasy & 2 & \best{$979$} & $974$ & \best{$978$} & \best{$981$} & \best{$977$} & $901$ & $971$ & \best{$979$} & \best{$981$} \\
ReacherHard & 2 & \best{$975$} & $953$ & \best{$973$} & \best{$977$} & \best{$973$} & $820$ & $964$ & \best{$976$} & \best{$976$} \\
WalkerRun & 6 & $893$ & $267$ & $674$ & \best{$926$} & $903$ & $66$ & $639$ & $856$ & $857$ \\
WalkerStand & 6 & \best{$986$} & $858$ & $975$ & \best{$989$} & \best{$986$} & $310$ & $978$ & \best{$988$} & \best{$985$} \\
WalkerWalk & 6 & \best{$978$} & $726$ & $954$ & \best{$979$} & \best{$976$} & $238$ & $964$ & \best{$976$} & \best{$974$} \\
\midrule
IQM &  & $888$ & $310$ & $739$ & $906$ & $895$ & $242$ & $744$ & $879$ & \best{$917$} \\
 
\bottomrule
\end{tabular}}%
\begin{minipage}{\wd\rltasksbox}
\dmcdomain{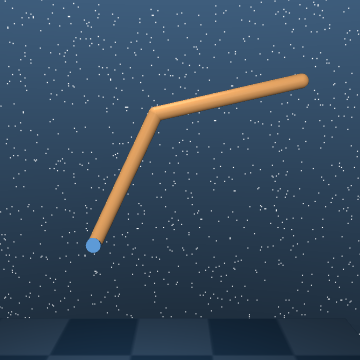}{AcrobotSwingup\\AcrobotSwingupSparse}\hfill
\dmcdomain{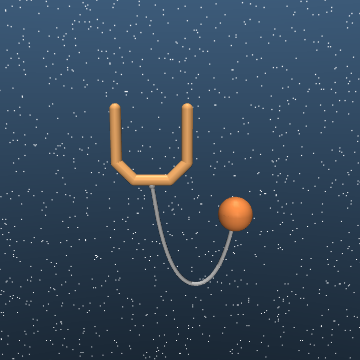}{BallInCup}\hfill
\dmcdomain{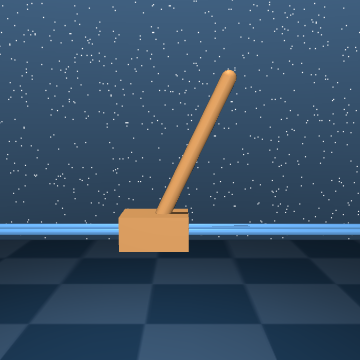}{CartpoleBalance\\CartpoleBalanceSparse\\CartpoleSwingup\\CartpoleSwingupSparse}\hfill
\dmcdomain{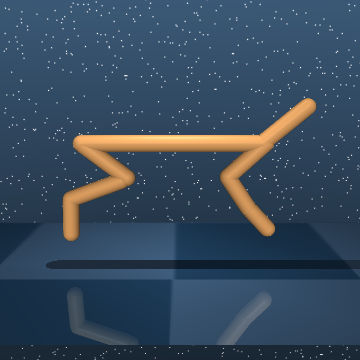}{CheetahRun}\hfill
\dmcdomain{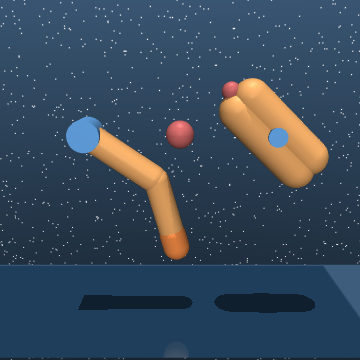}{FingerSpin\\FingerTurnEasy\\FingerTurnHard}\par
\vspace{1.5ex}
\dmcdomain{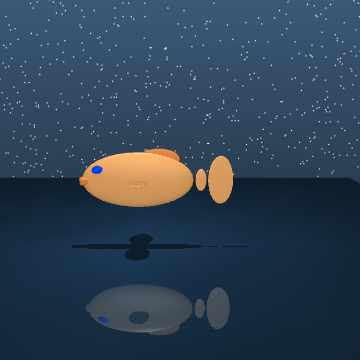}{FishSwim}\hfill
\dmcdomain{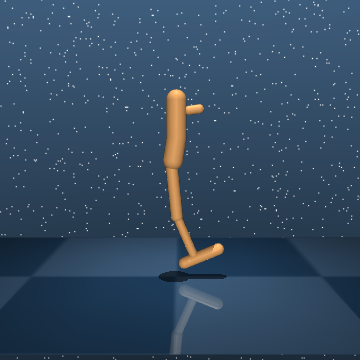}{HopperHop\\HopperStand}\hfill
\dmcdomain{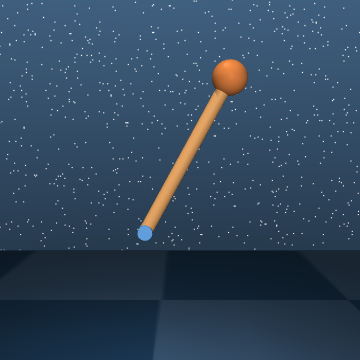}{PendulumSwingup}\hfill
\dmcdomain{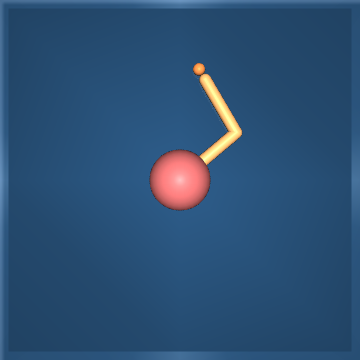}{ReacherEasy\\ReacherHard}\hfill
\dmcdomain{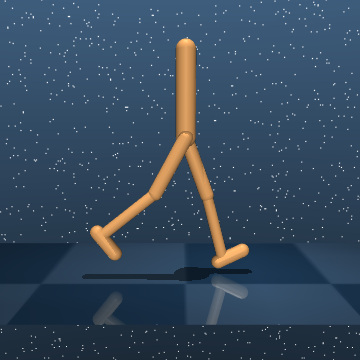}{WalkerRun\\WalkerStand\\WalkerWalk}
\end{minipage}\par
\vspace{3ex}
\usebox{\rltasksbox}
\end{table}

\end{document}